\documentclass{article}
\usepackage[nonatbib,preprint]{neurips_2020}
\usepackage[utf8]{inputenc} 
\usepackage[T1]{fontenc}    
\usepackage{hyperref}       
\usepackage{url}            
\usepackage{booktabs}       
\usepackage{amsfonts}       
\usepackage{nicefrac}       
\usepackage{microtype}      

\usepackage{multicol,lipsum}
\usepackage{microtype}
\usepackage{bbm}
\usepackage{comment}
\usepackage{color}
\usepackage{dsfont}
\usepackage{mathtools,multirow,array}
\usepackage{thm-restate}
\usepackage{algorithm, algorithmic}
\usepackage{graphicx}
\usepackage{subfigure}
\usepackage{amsmath}
\usepackage{dcolumn}
\usepackage{amssymb}
\usepackage{amsthm}
\usepackage{booktabs}
\usepackage{dsfont}
\usepackage{wrapfig}
\usepackage[font=small,labelfont=bf]{caption}

\newcommand{\MC}{\mathrm{MC}}
\newcommand{\dmix}{\mathrm{d}_{\mathrm{mix}}}
\newcommand{\dmixbar}{\bar{\mathrm{d}}_{\mathrm{mix}}}

\newcommand{\tmix}{\tau_{\mathrm{mix}}}
\newcommand{\loss}{\mathcal{L}}
\newcommand{\bias}{\mathrm{bias}}
\newcommand{\ber}{\mathrm{Ber}}
\newcommand{\tr}{\mathrm{Tr}}
\newcommand{\id}{\mathrm{I}}

\newcommand{\var}{\mathrm{var}}
\newcommand{\tv}{\mathrm{TV}}
\newcommand{\kl}{\mathrm{KL}}
\newcommand{\alg}{\mathrm{ALG}}
\newcommand{\law}{\mathcal{D}}
\newtheorem{proposition}{Proposition}
\newtheorem{theorem}{Theorem}

\newtheorem{lemma}{Lemma}

\newcommand{\defeq}{:=}
\newcommand{\norm}[1]{\left\|#1\right\|}
\newcommand{\trans}[1]{{#1}^{\top}}
\newcommand{\noise}{n}
\newcommand{\iprod}[2]{\langle #1, #2 \rangle}
\newcommand{\unif}{\mathsf{Unif}}
\newcommand{\bI}{\mathbf{I}}
\newcommand{\bJ}{\mathbf{J}}
\newcommand{\sgddd}{SGD-DD\xspace}

\def\R{\mathbb{R}}

\def\eps{\epsilon}

\def\N{\mathcal{N}}

\def \E{\mathbb{E}}

\usepackage{mathtools}
\usepackage{enumitem}

\usepackage[xspace]{ellipsis}

\usepackage[colorinlistoftodos,prependcaption,textsize=tiny,textwidth=20mm]{todonotes}

\title{Least Squares Regression with Markovian Data: Fundamental Limits and Algorithms}

\author{
  Guy Bresler \\
  Massachusetts Institute of Technology\\
  Cambridge, USA 02139 \\
  \texttt{guy@mit.edu} \\
  \And
   Prateek Jain \\
  Microsoft Research\\
  Bengaluru, India 560001 \\
  \texttt{prajain@microsoft.com} \\
  \And
  Dheeraj Nagaraj \\
  Massachusetts Institute of Technology\\
  Cambridge, USA 02139 \\
  \texttt{dheeraj@mit.edu} \\
   \And
  Praneeth Netrapalli \\
  Microsoft Research\\
  Bengaluru, India 560001\\
  \texttt{praneeth@microsoft.com} \\
  \And
  Xian Wu \\
  Stanford University\\
  Stanford, USA 94305 \\
  \texttt{xwu20@stanford.edu} \\
}
\begin{document}
\maketitle

\begin{abstract}
We study the problem of least squares linear regression where the datapoints are {\em dependent} and are sampled from a Markov chain. We establish sharp information theoretic minimax lower bounds for this problem in terms of $\tmix$, the mixing time of the underlying Markov chain, under different noise settings. Our results establish that in general, optimization with Markovian data is {\em strictly} harder than optimization with independent data and a trivial algorithm (\sgddd) that works with only one in every $\tilde{\Theta}(\tmix)$ samples, which are approximately independent, is minimax optimal. In fact, it is strictly better than the popular Stochastic Gradient Descent (SGD) method with constant step-size which is otherwise {\em minimax optimal} in the regression with {\em independent} data setting. 

Beyond a worst case analysis, we investigate whether structured datasets seen in practice such as Gaussian auto-regressive dynamics can admit more efficient optimization schemes. Surprisingly, even in this specific and natural setting, Stochastic Gradient Descent (SGD) with constant step-size is still no better than \sgddd. 
Instead, we propose an algorithm based on experience replay--a popular reinforcement learning technique--that achieves a significantly better error rate. Our improved rate serves as one of the first results where an algorithm outperforms \sgddd on an interesting Markov chain and also provides one of the first theoretical analyses to support the use of experience replay in practice. 

\end{abstract}

\section{Introduction}
Typical machine learning algorithms and their analyses crucially require the training data to be sampled independently and identically (i.i.d.).  However, real-world datapoints collected can be highly dependent on each other. One model for capturing data dependencies which is popular in many applications such as Reinforcement Learning (RL) is to assume that the data is generated by a Markov process. While it is intuitive that the state of the art optimization algorithms with provable guarantees for iid data will not converge as quickly or as efficiently for dependent data, a solid theoretical foundation that rigorously quantifies tight fundamental limits and upper bounds on popular algorithms in the non-asymptotic regime is sorely lacking. Moreover, popular schemes to break temporal correlations in the input datapoints that have been shown to work well in practice, such as \emph{experience replay}, are also wholly lacking in theoretical analysis. Through the classical problem of linear least squares regression, we first present fundamental limits for Markovian data, followed by an in-depth study of the performance of Stochastic Gradient Descent and variants (ie with experience replay). Our work is comprehensive in its treatment of this important problem and in particular, we offer the first theoretical analysis for experience replay in a structured Markovian setting, an idea that is widely adopted in practice for modern deep RL. 

There exists a rich literature in statistics, optimization and control that studies learning/modeling/optimization with Markovian data \cite{tsitsiklis1997analysis,kushner2003stochastic,Mokkadem88}. However, most of the existing analyses work only in the infinite data regime. \cite{DuchiAJJ12} provides non-asymptotic analysis of the  Mirror Descent method for Markovian data and \cite{bhandari2018finite}  provides a similar analysis for $\textrm{TD}(\lambda)$ (Temporal Difference) algorithms which are widely used in RL. However, the provided guarantees are in general suboptimal and seem to at best match the simple data drop technique where most of the data points are dropped. \cite{srikant2019finite} considers constant step size $\textrm{TD}(\lambda)$ algorithm but their guarantees suffer from a constant bias. 

Stochastic Gradient Descent (SGD) is the modern workhorse of large scale optimization and often used with dependent data, but our understanding of its performance is also weak. Most of the results, ie ~\cite{kushner2003stochastic} are asymptotic and do not hold for finite number of samples. Works like \cite{daskalakis2019regression,dagan2019learning} do provide stronger results but can handle only weak dependence among observations rather than the general Markovian structure. On the other hand, works such as~\cite{DuchiAJJ12} present non-asymptotic analyses, but the rates obtained are at least $\tmix$ factor worse than the rates for independent case. In fact, in general, the existing rates are no better than those obtained by a trivial SGD-Data Drop (\sgddd) algorithm which reduces the problem approximately to the i.i.d. setting by processing {\em only one} sample from each batch of $\tmix$ training points. 
 These results suggest that optimization with Markov chain data is a strictly harder problem than the i.i.d. data setting, and also SGD might not be the "correct" algorithm for this problem.  We refer to \cite{tsitsiklis1997analysis, bhandari2018finite} for similar analyses of the related TD learning algorithm widely used in reinforcement learning. 

To gain a more complete understanding of the fundamental problem of optimization with Markovian data, our work addresses the following two key questions: 1) what are the fundamental limits for learning with Markovian data and how does the performance of SGD compare, 2) can we design algorithms with better error rates than the trivial \sgddd method that throws out most of the data. 

%
%

We investigate these questions for the classical problem of linear least squares regression. We establish algorithm independent {\em information theoretic lower bounds} which show that the minimax error rates are necessarily worse by a factor of $\tmix$ compared to the i.i.d. case and surprisingly, these lower bounds are achieved by the \sgddd method. We also show that SGD is not minimax optimal when observations come with independent noise, and that SGD may suffer from constant bias when the noise correlates with the data. 

To study {\bf{Question (2)}}, we restrict ourselves to a simple Gaussian Autoregressive (AR) Markov Chain which is popularly used for modeling time series data \cite{adhikari2013introductory}. 
Surprisingly, even for this restricted Markov chain, SGD does not perform better than the \sgddd method in terms of dependence on the mixing time.
However, we show that a method similar to experience replay \cite{mnih2015human,schaul2015prioritized,andrychowicz2017hindsight}, that is popular in reinforcement learning, achieves significant improvement over SGD for this problem. To the best of our knowledge, this represents the first rigorous analysis of the experience replay technique, supporting it's practical usage. Furthermore, for a non-trivial Markov chain, this  represents first improvement over performance of \sgddd.

We elaborate more on our problem setup and contributions in the next section.

\newcommand{\splitab}[2]{\begin{tabular}{@{}c@{}}
		#1\\
		#2
\end{tabular}}


\begin{table}[t]
	\centering
	{\scriptsize 
	\begin{tabular}{|>{\centering\arraybackslash}m{2.1cm}|>{\centering\arraybackslash}m{2.2cm}|c|>{\centering\arraybackslash}m{4cm}|>{\centering\arraybackslash}m{2cm}|}
		\hline
		Setting & Algorithm & Lower/upper & Bias & Variance \\\hline
		\multirow{4}{*}{Agnostic} & Information theoretic & Lower & $\exp\left(\frac{-T}{\kappa \tmix}\right) \norm{w_0 - w^*}^2$ Theorem~\ref{thm:bias-lb} & $\frac{\tmix \sigma^2 d}{T}$ Theorem~\ref{thm:agn_minimax} \\ \cline{2-5}
		 & SGD& Lower & $\quad$ Constant $\quad$ Theorem~\ref{thm:agn_sgd}& --- \\ \cline{2-5}
		 & \sgddd & Upper & $\exp\left(\frac{-T}{\kappa \tmix }\right) \norm{w_0 - w^*}^2$ Theorem~\ref{thm:agn_sgd_dd} & $\frac{\tmix \sigma^2 d}{T}$ Theorem~\ref{thm:agn_sgd_dd} \\ \hline
		 \multirow{5}{*}{Independent} &  Information theoretic & Lower & $\exp\left(\frac{-T}{\kappa \tmix}\right) \norm{w_0 - w^*}^2$ Theorem~\ref{thm:bias-lb} & $\;\; \frac{\sigma^2 d}{T}$ \cite{van2000asymptotic} \\ \cline{2-5}
		  & SGD & Lower & --- & $\frac{ \tmix \sigma^2 d}{T}$ Theorem~\ref{thm:sgd-lb} \\ \cline{2-5}
		  & Parallel SGD & Upper & $\exp\left(\frac{-T}{\kappa \tmix }\right) \norm{w_0 - w^*}^2$ Theorem~\ref{thm:par-sgd} & $\frac{\sigma^2 d}{T}$ Theorem~\ref{thm:par-sgd} \\ \hline
	    Gaussian Autoregressive  &  SGD& Lower & $\exp\left(\frac{-T\log({d})}{\kappa\tmix}\right) \norm{w_0 - w^*}^2$ Theorem~\ref{thm:sequential_lower_bound} & $-$ \\ \cline{2-5}
		Dynamics with Independent Noise& SGD-ER (Algorithm~\ref{alg_main}) & Upper & $\exp\left(\frac{-T\log({d})}{\boldsymbol{\kappa\color{red}\sqrt{\tmix}}}\right) \norm{w_0 - w^*}^2$ Theorem~\ref{thm:exp-replay} & $\frac{\sqrt{\tmix}\sigma^2 d}{T}$ Theorem~\ref{thm:exp-replay} \\ \hline
		\end{tabular}}
		\vspace{2mm}
		\caption{{See Section~\ref{sec:main_problem} for a description of the three settings considered in this paper. We suppress universal constants and $\log$ factors in the expressions above. For linear regression with i.i.d. data, tail-averaged SGD with constant stepsize achieves minimax optimal bias and variance rates of $\exp\left(\frac{-T}{\kappa \tmix}\right) \norm{w_0 - w^*}^2$ and $\frac{\sigma^2 d}{T}$ respectively. In contrast, even the minimax rates in the general agnostic Markov chain setting are $\tmix$-factor worse, and tail-averaged SGD with constant-step size is not able to achieve these rates as well. We modify and analyze variants of SGD (i.e., \sgddd and Parallel SGD) that achieve close to minimax error rates. Finally, for the Gaussian Autoregressive Markov chain, SGD still achieves a trivial bias error rate while our proposed experience replay based SGD-ER method can decay the bias significantly faster.}
		}\label{tab:results}
		\end{table}

\subsection{Notation and Markov Chain Preliminaries}
\label{sec:notation}
In this work, $\|\cdot\|$ denotes the standard $\ell^2$ norm over $\mathbb{R}^d$. Given any random variable $X$, we use $\law(X)$ to denote the distribution of $X$. $\tv(\mu,\nu)$ denotes the total variation distance between the measures $\mu$ and $\nu$. Sometimes, we abuse notation and use $\tv(X,Y)$ as shorthand for $\tv(\law(X),\law(Y))$. We let $\kl(\mu\|\nu)$ denote the KL divergence between measures $\mu$ and $\nu$. Consider a time invariant Markov chain $\MC$ with state  space $\Omega \subset \mathbb{R}^d$ and transition matrix/kernel $P$. We assume throughout that $\MC$ is ergodic with stationary distribution  $\pi$. For $x \in \Omega$, by $P^{t}(x,\cdot)$ we mean $\law(X_{t+1}|X_1 = x)$, where $X_1, X_2, \dots, X_{t+1} \sim \MC$.

For a given Markov chain $\MC$ with transition kernel $P$ we consider the following standard measure of distance from stationarity at time $t$,
$$	\dmix(t) := \sup_{x \in \Omega} \tv(P^{t}(x,\cdot), \pi)\,.$$
We note that all irreducible aperiodic finite state Markov chains are ergodic and exponentially mixing i.e, $\dmix(t) \leq Ce^{-ct}$ for some $C,c > 0$. For a finite state ergodic Markov chain $\MC$, the \emph{mixing time} is defined as $$\tmix = \inf\{t: \dmix(t) \leq 1/4\}\,.$$ We note the standard result that $\dmix(t)$ is a decreasing function of $t$ and whenever $t = l \tmix$ for some $l \in \mathbb{N}$, we have
\begin{equation}\label{eq:binary_mixing}
\dmix(l\tmix) \leq 2^{-l}.
\end{equation} 
See Chapter 4 in \cite{levin2017markov} for further details.
\section{Problem Formulation and Main Results}
\label{sec:main_problem}
Let $X_1 \to X_2 \to \cdots \to X_T$ be samples from an irreducible Markov chain $\MC$ with each $X_t \in \Omega \subset \R^d$. Let $Y_t(X_t) \in \mathbb{R}$ be observations whose distribution depends on $X_t$ and exogenous contextual parameters (such as noise). That is, $Y_t$ is conditionally independent of $(X_s)_{s\neq t}$ given $X_t$. 
Given samples $(X_1,Y_1), \cdots, (X_T,Y_T)$, our goal is to estimate a parameter $w^*$ that minimizes the out-of-sample loss, which is the expected loss on a new sample $(X,T)$ where $X$ is drawn independently from the stationary distribution $\pi$ of $\MC$:  
\begin{align}
	w^{*} = \arg\min_{\mathbb{R}^d}\loss_{MC}(w)\,, \quad \text{where}\quad \loss_{MC}(w) \defeq \E_{X \sim \pi}\left[\left(\trans{X}w-Y\right)^2\right].\label{eqn:prob}
\end{align}

 Define $A := \mathbb{E}_{X\sim \pi}XX^{\intercal}$. Let $\|X_t\| \leq 1 $ almost surely and $A = \mathbb{E}_{X \sim \pi}XX^{\intercal}\succeq \frac{1}{\kappa}I$ for some finite `condition number' $\kappa \geq 1$, implying unique minimizer $w^*$. Also, let $\upsilon < \infty$ be such that $\mathbb{E}\left[|Y_t|^2|X_t = x\right] \leq \upsilon$ for every $X_t \in \Omega$. We define the `noise' or `error' to be $\noise_t(X_t,Y_t) := Y_t - \langle X_t , w^{*}\rangle$ and by abusing notation, we denote $\noise_t := Y_t - \langle X_t, w^{*}\rangle$. We also let $\sigma^2 \defeq \E_{X_t \sim \pi}\left[\noise_t^2\right]$. 

\subsection{Problem Settings}
\label{subsec:problem_settings}
Our main results are in the context of the following problem settings:
\begin{itemize}[leftmargin=*]
	\itemsep 0pt
	\topsep 0pt
	\parskip 0pt
	\item \textbf{Agnostic setting}:
In this setting, the vectors $X_i$ are stationary (distributed according to $\pi$) and come from a finite state space $\Omega \subseteq \mathbb{R}^d$.
	\item \textbf{Independent noise setting}: In addition to our assumptions in the agnostic setting, in this setting, we assume that $\noise_t(X)$ is an independent and identically distributed zero mean random variable with variance $\sigma^2$ for all $X \in \Omega$.
	\item \textbf{Experience Replay for the Gaussian Autoregressive Chain}: In this setting, we fix a parameter $\epsilon$ and consider the non-stationary Markov chain $X_t$ that evolves as $X_t = \sqrt{1-\epsilon^2} X_{t-1} + \epsilon g_t$, where $g_t \sim \frac{1}{\sqrt{d}}\N\left(0,I\right)$ is sampled independently for different $t$. The observations $Y_t$ are given by $\iprod{X_t}{w^*} + \xi_t$ for some fixed $w^*$, and $\xi_t$ is an independent mean 0 variance $\sigma^2$ random variable.
\end{itemize}

\subsection{Main Results}
We are particularly interested in understanding the limits (both upper and lower bounds) of SGD type algorithms, with constant step sizes, for solving~\eqref{eqn:prob}. These algorithms are, by far, the most widely used methods in practice for two reasons: 1) these methods are memory efficient, and 2) constant step size allows decreasing the error rapidly in the beginning stages and is crucial for good convergence.
In general, the error achieved by any SGD type procedure can be decomposed as a sum of two terms: \emph{bias} and \emph{variance} where the bias part depends on step size $\alpha$ and $\norm{w_1 - w^*}^2$ and the variance depends on $\sigma^2$, where $w_1$ is the starting iterate of the SGD procedure. Thus,
\begin{align}\label{eqn:bias-var}
	\loss_{MC}(w_T^{\textrm{SGD}}) - \loss_{MC}(w^*)= \loss_{MC}^{\textrm{bias}}\left(\norm{w_1 - w^*}^2\right) + \loss_{MC}^{\textrm{variance}}\left(\sigma^2\right).
\end{align}
The bias term arises because the algorithm starts at $w_0$ and needs to travel a distance of $\norm{w_1-w^*}$ to the optimum. The variance term arises because the gradients are stochastic and even if we initialize the algorithm at $w^*$, the stochastic gradients are nonzero.

\noindent We provide a brief summary of our contributions below; See Table~\ref{tab:results} for a comprehensive overview: 
\begin{itemize}[leftmargin=*]
	\itemsep 0pt
	\topsep 0pt
	\parskip 0pt
	\item For general least squares regression with Markovian data, we give information theoretic minimax lower bounds under different noise settings that show that any algorithm will suffer from slower convergence rates (by a factor of $\tmix$) compared to the i.i.d. setting (Section \ref{sec:minimax_lb}). We then show via algorithms like \sgddd and parallel SGD that the lower bounds are tight.

	\item We study lower bounds for SGD specifically and show that SGD converges at a suboptimal rate in the independent noise setting and that SGD with with constant step size and averaging might not even converge to the optimal solution in the agnostic noise setting. (Section \ref{subsec:agnostic_SGD_bias}).

	\item For Gaussian Autoregressive (AR) dynamics, we show that SGD with experience replay can achieve significantly faster convergence rate (by a factor of $\sqrt{\tmix}$) compared to vanilla SGD. This is one of the first analyses of experience replay that validates its effectiveness in practice. (Section \ref{sec:brownian_motion_section}). Simulations confirm our analysis and indicates that our derived rates are tight. 
	\end{itemize}

%
%

\section{Information Theoretic Minimax Lower Bounds for Bias and Variance}
\label{sec:minimax_lb}
We consider the class $\mathcal{Q}$ of all Markov chain linear regression problems $Q$, as described in Section~\ref{sec:main_problem}, where the following conditions hold:
\begin{enumerate}[leftmargin=*]
	\itemsep 0pt
	\topsep 0pt
	\parskip 0pt
	\item The optimal parameter has norm $\|w^{*}\| \leq 1$.
	\item Markov chain $\MC$ is such that $\tmix \leq \tau_0 \in \mathbb{N}$.
	\item The condition number $\kappa \leq \kappa_0$.
	\item Noise sequence from a noise model $\mathcal{N}$ (ex: independent noise, noiseless, agnostic etc.) 
\end{enumerate} 
We want to lower bound the minimax excess risk:
\begin{align}
\mathcal{L}(\mathcal{Q}) := \inf_{\alg \in \mathcal{A}} \sup_{Q \in \mathcal{Q}} \mathbb{E}\left[\loss_Q\bigr(\alg\left(D_Q(T)\right)\bigr)\right] - \loss_Q(w_Q^{*}) ,\label{eq:minimax_def}
\end{align}

where for a given $Q\in\mathcal{Q}$, $\loss_Q$ is the loss function with optimizer $w_Q^{*}$, and the class of algorithms $\mathcal{A} := \{\alg : (\R^d\times \R)^T \to \R^d\}$ which take as input the data $D_Q(T):= \{(X_t,Y_t): 1\leq t\leq T\}$ and output an estimate $\alg(D_Q(T))$ for $w_{Q}^{*}$.

\subsection{General Minimax Lower Bound for Bias Decay}
\label{subsec:minimax_bias}
Theorem \ref{thm:bias-lb} gives the most general minimax lower bound which holds for any algorithm, in any kind of noise setting. In particular, this gives a bound on the `bias' term in the bias-variance decomposition of SGD algorithm's excess loss \eqref{eqn:bias-var} by letting noise variance $\sigma^2 \to 0$.

\begin{theorem}\label{thm:bias-lb}
	In the definition of $\mathcal{Q}$, we let $\mathcal{N}$ be any noise model. Then, for any mixing time $\tau_0$ and condition number $\kappa_0 \geq 2$, we have: $\mathcal{L}(\mathcal{Q}) \geq \frac{\kappa_0-1}{\kappa_0^2}\left(1-\frac{C}{\tau_0 \kappa_0}\right)^T,$	where $C$ is a universal constant.
\end{theorem}

See Appendix~\ref{app:minimax_bias} for a complete proof. Note that the bias decay rate is a $\tmix$ factor worse than that in the i.i.d. data setting \cite{jain2018parallelizing}. Furthermore, our result holds for {\em any noise} model, and for all settings of key parameters $\kappa_0$ and $\tau_0$. This implies that unless the Markov chain itself has specific structure, one cannot hope to design an algorithm with better {\em bias decay rate} than the trivial \sgddd method. Section~\ref{sec:brownian_motion_section} describes a class of Markov chains for which improved rates are indeed possible. 

\subsection{A Tight Minimax Lower Bound for Agnostic Noise Setting}
\label{subsec:minimax_variance}
We now present a minimax lower bound in the agnostic setting (Section \ref{subsec:problem_settings}). The bound analyzes the variance term $\loss_\mathcal{Q}^{\textrm{variance}}\left(\sigma^2\right)$. Again, we incur an additional, unavoidable $\tmix$ factor compared to the setting with i.i.d. samples (Table~\ref{tab:results}). 


\begin{theorem}\label{thm:agn_minimax}
	For the class of problems $\mathcal{Q}$ with noise model $\mathcal{N}$ of agnostic noise and for the class of algorithms $\mathcal{A}$ defined above, we have $\mathcal{L}(\mathcal{Q}) \geq c_1\frac{\tau_0\sigma^2 d}{T}$,
	where $T$ is the number of observed data points such that $T \geq c_2d^{2}\tau_0\sigma^2$ and $c_1,c_2$ are universal constants. 
	
	Furthermore, \sgddd achieves above mentioned rates up to logarithmic factors (Theorem \ref{thm:agn_sgd_dd}). 
\end{theorem}
This result combined with Theorem~\ref{thm:bias-lb}, implies that for general MC in agnostic noise setting, both the bias and the variance terms suffer from an additional $\tau_0$ factor. Our proof shows existence of two different MCs whose evolution till time $T$ can be coupled with high probability and hence they give the same sequence of data. But, since the chains are different and the noise is agnostic, the corresponding optimum parameters $w^{*}$ are different. See Appendix~\ref{app:minimax_var} for a detailed proof. 

\subsection{A Tight Minimax Lower Bound for Independent Noise Setting}
\label{subsec:minimax_independent_noise}
We now discuss the variance lower and upper bound for the general Markov Chain based linear regression when the noise is {\em independent} (Section~\ref{subsec:problem_settings}). 
\begin{theorem}\label{thm:ind_minimax}
For the class of problems $\mathcal{Q}$ with noise model $\mathcal{N}$ of independent noise and for the class of algorithms $\mathcal{A}$ defined above, we have $\mathcal{L}(\mathcal{Q}) \geq \frac{d\sigma^2}{T}$. This bound is tight up to logarithmic factors since `Parallel SGD' achieves the rates established above. (Theorem \ref{thm:par-sgd}, Section \ref{subsec:parsgd})
 \end{theorem}
 Note that the lower bound follows directly from the classical iid samples case (Theorem 1 in \cite{mourtada2019exact}) apply. For upper bound, we propose and study a Parallel SGD method discussed in detail in Appendix~\ref{app:ind_psgd_var}. Interestingly, SGD with constant step size, which is minimax optimal for i.i.d samples with independent noise, is not minimax optimal when the samples are Markovian. We establish this fact and others in the next section in our study of SGD.
\section{Sub-Optimality of SGD}
\label{sec:sgd_optimal}
In previous section, we presented information theoretic limits on the error rates of {\em any} method when applied to the general Markovian data, and presented algorithms that achieve these rates. However, in practice, SGD is the most commonly used method for learning problems. So, in this section, we specifically analyze the performance of constant step size SGD on Markovian data. Somewhat surprisingly, SGD shows a sub-optimal rates for both independent and agnostic noise settings. 

\begin{algorithm}
\caption{SGD with tail-averaging}
\label{alg_SGD}
\begin{algorithmic}
\REQUIRE  $T \in \mathbb{N}$ , samples $(X_1,Y_1),\dots, (X_T,Y_T) \in  \R^d\times \R$, step size $\alpha > 0$, initial point $w_1 \in \R^d$.

\FOR{t in range [1, $T$]}
\STATE  Set $w_{t+1}   \leftarrow w_t - \alpha (X_tX_t^{\intercal}w_t - X_tY_t)\,.$
\ENDFOR

\RETURN $\hat{w} \leftarrow \frac{1}{T - \lfloor T/2\rfloor}\sum_{t = \lfloor T/2\rfloor+1}^{T}w_t \,.$
\end{algorithmic}
\end{algorithm}

\paragraph{SGD with Constant Step Size is Asymptotically Biased in the Agnostic Noise Setting}
\label{subsec:agnostic_SGD_bias}
It is well known that when data is iid, the expected iterate of Algorithm \ref{alg_SGD}, $\E\left[w_t\right]$ converges to $w^*$ as $t \rightarrow \infty$ in any noise setting. However, this does not necessarily hold when the data is Markovian. When the noise in each observation $\noise_t(X)$ depends on $X$, SGD with constant step size may yield iterates that are biased estimators of the parameter $w^{*}$ even as $t \to \infty$. In this case, even tail-averaging such as in Algorithm \ref{alg_SGD}, cannot resolve this issue. See Appendix~\ref{app:sgd_bias} for the detailed proof.


\begin{theorem}\label{thm:agn_sgd}
	There exists a finite Markov chain $\MC$ with  $\tmix, \kappa < C$
 	$X_0 \sim \pi(\MC)$ and $X_1 \to X_2 \to \cdots \to X_T \sim \MC$, SGD (Algorithm~\ref{alg_SGD}) run with any constant step size $\alpha>0$ leads to a constant bias, i.e., for every $t$ large enough, $$\|\mathbb{E}[w_t]-w^*\|\geq c \alpha,$$ where $w_t$ is the $t$-th step iterate of the SGD algorithm. (Where $c,C >0$ are universal constants.)
\end{theorem}

\paragraph{SGD in the Independent Noise Setting is not Minimax Optimal (Appendix~\ref{app:sgd_lb})}
\label{subsec:minimax_sgd_constant_step_independent_noise}

Let $\textrm{SGD}_{\alpha}$ be the SGD algorithm with step size $\alpha$ and tail averaging (Algorithm~\ref{alg_SGD}). For $X_1 \to \dots \to X_T \sim \MC_0$, we denote the output of $\textrm{SGD}_{\alpha}$ corresponding to the data $D_{0}(T) := (X_t,Y_t)_{t=1}^{T}$ by $\textrm{SGD}_{\alpha}(D_{0}(T))$.  We let $w_{0}^{*}$ to be the optimal parameter corresponding to the regression problem. We have the following lower bound.
\begin{theorem}\label{thm:sgd-lb}
For every $\tau_0, d \in \mathbb{N}$, there exists a finite state Markov Chain, $\MC_0$ and associated independent noise observation model (see Section \ref{subsec:problem_settings}) with points in $\R^d$, mixing time at most $\tau_0$ and $\|w_{\MC_0}^{*}\| \leq 1$, such that: 
$ \mathbb{E}\loss_{\MC_0}(\textrm{SGD}_{\alpha}(D_{0}(T))) - \loss_{\MC_0}(w_0^{*}) \geq \big(1-o_T(1)\big)\frac{c^{\prime}\alpha\tau_0 \sigma^2 d}{T},
$	where $c^{\prime}$ is a universal constant and $o_T(1) \to 0$ exponentially in $T$ and $\sigma^2 = \mathbb{E}n_t^2$ is the noise variance. 
\end{theorem}
The above result shows that while SGD with constant step size and tail averaging is minimax optimal in the independent data setting, it's variance rate is $\tau_0$ factor sub-optimal in the setting of Markovian data and independent noise. It is also $\tau_0$ factor worse compared to the rate established in Theorem~\ref{thm:ind_minimax}.
%
%
%


\section{Experience Replay for Gaussian Autoregressive (AR) Dynamics}
\label{sec:brownian_motion_section}
Previous two sections indicate that for worst case Markov chains, \sgddd, despite wasting most of the samples, might be the best algorithm for Markovian data. This naturally seems quite pessimistic, as in practice, approaches like experience replay are popular \cite{sutton2018reinforcement}. In this section, we attempt to reconcile this gap by considering a restricted but practical  Markov Chain (Gaussian AR chain) that is used routinely for time-series modeling \cite{adhikari2013introductory} and intuitively seems quite related to the type of samples we can expect in reinforcement learning (RL) problems. Interestingly, even for this specific chain, we show that SGD's rates are no better than the \sgddd method. On the other hand, an experience replay based SGD method (Algorithm~\ref{alg_main}) is able to give significantly faster rates, thus supporting it's usage in practice. More details and proofs are found in Section \ref{app:gaussian_ar}. 

Suppose our sample vectors $X \in \R^d$ are generated from a Markov chain (MC) with the following dynamics: \vspace*{-5pt}
\begin{equation}\label{eq:gaussar}X_1 = G_1, \cdots, X_{t+1} = \sqrt{1-\eps^2}X_t + \eps G_{t + 1},\cdots,\end{equation}
where $\eps$ is fixed and known, and each $G_j$ is independently sampled from $ \frac{1}{\sqrt{d}} \N(0,\id_d)$. Each observation $Y_i = X^T_i w^* + \xi_i$, where the noise $\xi_i$ is independently drawn with mean 0 and variance $\sigma^2$. That is, every new sample in this MC is a random perturbation from a fixed distribution of the previous sample, which is intuitively similar to the sample generation process in RL. 

The mixing time of this Markov chain is $\tmix = \Theta \left(\frac{1}{\epsilon^2}\log({d})\right)$ (Lemma \ref{lem:brownian_mixing}, Section~\ref{sec:upper_bound_formulation}). Also, the covariance matrix of the stationary distribution is $\frac{1}{d}\id_d$, so the condition number of this chain is $\kappa = d$.

%
%
\subsection{Lower Bound for SGD with Constant Step Size}
We first establish a lower bound on the rate of bias decay for SGD with constant step size for this problem, which will help demonstrate that experience replay is effective in making SGD iterations more efficient. 
\begin{theorem}[Lower Bound for SGD with constant step size for Gaussian AR Chain]
	\label{thm:sequential_lower_bound}
	$\loss(w^{\bias}) \geq \Omega\left(\exp\left(\frac{-T\log({d})}{\kappa\tmix}\right) \norm{w_0 - w^*}^2\right)$. Recall that $\kappa=d$ for MC in \eqref{eq:gaussar}
\end{theorem}

The proof for this lemma involves carefully tracking the norm of the error at each iteration $\|w_t - w^*\|$. We show that the expected norm of the error contracts by a factor of at most $\frac{\eps^2}{d}$ in each iteration, therefore, we require $T = \Omega(\frac{d}{\eps^2})$ samples and iterations to get a $\delta$-approximate $w_T$. Note that the number of samples required here is $\Omega(\frac{d}{\eps^2}) = \Omega(\frac{\kappa \tmix}{\log({d})})$. See Section \ref{app:er_lb} for a detailed proof.
\subsection{SGD with Experience Replay}
We propose that the following interpretation of experience replay applied to SGD, which improves the dependence on $\tmix$ on the rate of error decay. 

Suppose we have a continuous stream of samples $X_1, X_2, \ldots X_T$ from the Markov Chain. We split the $T$ samples into $\frac{T}{S}$ separate buffers of size $S$ in a sequential manner, ie $X_1, \ldots X_S$ belong to the first buffer. Let $S = B + u$, where $B$ is orders of magnitude larger than $u$. From within each buffer, we drop the first $u$ samples. Then starting from the first buffer, we perform $B$ steps of SGD, where for each iteration, we sample uniformly at random from within the $[u, B+u]$samples in the first buffer. Then perform the next $B$ steps of SGD by uniformly drawing samples from within the $[u, B+u]$ samples in the second buffer. We will choose $u$ so that the buffers are are approximately i.i.d.. 

We run SGD this way for the first $\frac{T}{2S}$ buffers to ensure that the bias of each iterate is small. Then for the last $\frac{T}{2S}$ buffers, we perform SGD in the same way, but we tail average over the last iterate produced using each buffer to give our final estimate $w$. We formally write Algorithm \ref{alg_main}. 

\begin{algorithm}
\caption{SGD with Experience Replay (SGD-ER)}
\label{alg_main}
\begin{algorithmic}
\REQUIRE $(X_1, Y_1), \ldots (X_T, Y_T) \in \R^d$ sampled using \eqref{eq:gaussar}, $\eta$: learning rate
\STATE $w \sim \N(0,1)$,  $B \leftarrow \frac{1}{\eps^7}$, $u \leftarrow  \max(\frac{2}{\epsilon^2} \log \frac{300000 \pi d B}{\epsilon}, \frac{2}{\epsilon^2} \log \frac{300000 \pi d^2\sigma^6}{\epsilon^2 \delta})$, $S\leftarrow B+u$
\FOR{each buffer $j \in [0, \frac{T}{S}-1]$}
\STATE $\text{Buffer}_j=[X_{(S\cdot j + 1)},\ \ldots,\ X_{(S\cdot j + S)}]$ 
\FOR{iterate in range[1, $B$]}
\STATE  $w=w-\eta (Y_{Sj + i}-\langle X_{Sj + i},w \rangle)X_{Sj + i}$ where $i\stackrel{unif}{\sim}[u,S]$
\ENDFOR
\STATE Store $w_j\leftarrow w$
\ENDFOR
\RETURN $\frac{2S}{T}\sum\limits_{j={\cdot T/{2S}+1}}^\frac{T}{S} w_j\ \ $ (i.e. average over last $T/{2S}$ buffers)
\end{algorithmic}
\end{algorithm}

\begin{theorem}[SGD with Experience Replay for Gaussian AR Chain]
	\label{thm:exp-replay}For any $\eps \leq 0.21$, if $B \geq \frac{1}{\eps^7}$ and $d = \Omega(B^4 \log (\frac{1}{\beta}))$, with probability at least $1-\beta$, Algorithm~\ref{alg_main} returns $w$ such that $\E[\loss(w)] \leq O\left(\exp\left(\frac{-T\log({d})}{\kappa\sqrt{\tmix}}\right) \norm{w_0 - w^*}^2\right) + \tilde{O}\left(\frac{\sigma^2d\sqrt{\tmix}}{ T}\right) + \loss(w^*)$. Recall that $\kappa=d$ for MC in \eqref{eq:gaussar}.
\end{theorem}

\begin{proof}
$\E[\loss(w)] = \E(\loss(w^{\bias})) + \E(\loss(w^{\var})) + \loss(w^*)$. We give proof sketches for $\E(\loss(w^{\bias})) \leq O\left(\exp\left(\frac{-T\log({d})}{\kappa\sqrt{\tmix}}\right) \norm{w_0 - w^*}^2\right)$ and $\E(\loss(w^{\var})) \leq \tilde{O}\left(\frac{\sigma^2d\sqrt{\tmix}}{ T}\right)$. Formal proofs are given in the appendix. 
\end{proof}

{\bf Proof Sketch for Bias Decay (Proof in Section \ref{sec:er_sgd_upper_bound}):}
Since the samples within the same buffer and across buffers are highly dependent, our algorithm drops the first $u$ samples from each batch. This ensures that across buffers the sampled points are approximately independent. This allows us to break down the problem into analyzing the progress made by one buffer, which we can then multiply $T/S$ times to get the overall bias bound. Lemma \ref{lem:approx-iid} formalizes this idea and upper bounds the expected contraction after every $B$ samples from a buffer $j$, by the expected contraction of a parallel process where the first vector in this buffer was sampled i.i.d. from $\N(0, \frac{1}{d}I)$.

The rest of the proof involves solving for expected rate of error decay when taking $B$ steps of SGD using samples generated from a single buffer in this parallel process. We write $H \defeq  \frac{1}{B} \sum\limits^S_{j=u+1} X_j X_j^T$, where $X_j$ are the vectors in the sampling pool. Lemma \ref{lem:iid_contraction} establishes that the error in the direction of an eigenvector $v$, with associated eigenvalue $\lambda(v)\geq \frac{1}{B}$ in $H$ contracts by a factor of $\frac{1}{2}$ after $B$ rounds of SGD, while smaller eigenvalues in $H$ in the worst case do not contract. By spherical symmetry of eigenvectors, as long as the fraction of eigenvalues $\geq \frac{1}{B}$ is at least $\frac{\eps B}{d}$, and since we draw $B$ samples from each buffer, we see that the loss decays at a rate of $\exp(\frac{-T\eps}{d}) = \exp(\frac{-T\log({d})}{\kappa \sqrt{\tmix}})$. 

So, the key technical argument is to establish that $\frac{\eps B}{40\pi}$ of the eigenvalues of $H$ are larger than or equal to $\frac{1}{B}$, the overall proof structure is as follows. The non-zero eigenvalues of $H$ correspond directly to the non-zero eigenvalues of the gram matrix $M = \frac{1}{B} X^TX$, where the columns of $X$ are each $X_j$. We show that the Gram matrix can be written as $C$ + $E$, where $C$ is a circulent matrix and $E$ is a small perturbation which can be effectively bounded when $d$ is large, i.e., $d=O(B^4)$. Using standard results about  eigenvalues of $C$ along with Weyl's inequality \cite{bhatia97} to handle perturbation $E$, we get a bound on the number of large-enough eigenvalues of $H$. The formal proof is in Section \ref{sec:er_sgd_upper_bound}.



{\bf Proof Sketch for Variance Decay (Proof in Section \ref{sec:er_sgd_upper_bound_variance})}

To analyze the variance, we start with $w^{\var}_0 = w^*$, and based on the SGD update dynamics, consider the expected covariance matrix $\E[(w^{\var}-w^*)(w^{\var}-w^*)^T]$, where $w^{\var}$ is the tail-averaged value of the last iterate from every buffer, for the last $\frac{T}{2S}$ buffers. We show that for each iterate in the average, the covariance matrix $\E[(w^{\var}_t-w^*)(w^{\var}_t-w^*)^T] \preceq 3\sigma^2$. Next, we analyze the cross terms $\E[(w^{\var}_i-w^*)(w^{\var}_j-w^*)^T]$, which is approximately equal to $(I- H)^{j-i} \E[(w^{\var}_i-w^*)(w^{\var}_i-w^*)^T] $, when $j> i$ are buffer indices. This approximation is based on perfectly iid buffers, which we later correct by explicitly quantifying the worst case difference in the expected contraction of our SGD process and a parallel SGD process that does use perfectly iid buffers, (see Lemma \ref{lem:approxim-iid-helper}). We use our earlier analysis of the eigenvalues of $H$ to arrive at our final rate. 

\begin{wrapfigure}{r}{0.33\columnwidth}
	\centering
	\includegraphics[width=0.33\columnwidth]{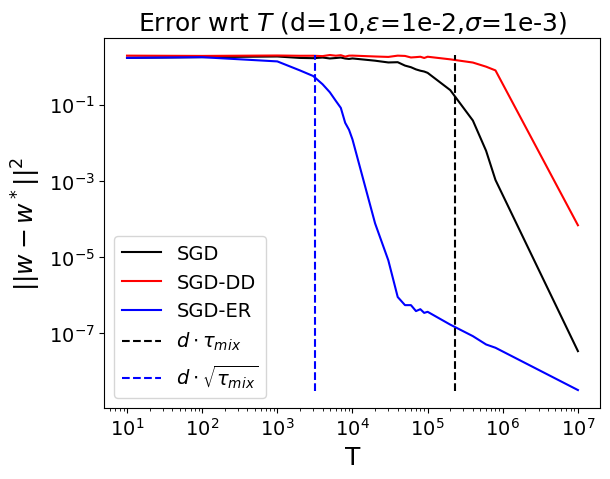}\vspace*{-10pt}
	\caption{Gaussian AR Chain: error incurred by various methods.}
	\label{fig:exprply}
	\vspace{-2mm}
\end{wrapfigure}

\paragraph{Simulations.} We also conducted experiments on  data generated using Gaussian AR MC \eqref{eq:gaussar} . We set $d=10$,  noise std. deviation $\sigma=1e-3$, $\epsilon=.01$ (i.e. $\tmix\approx 3e-4$), and buffer size $B=1/\eps^2$. We report results averaged over $100$ runs. Figure~\ref{fig:exprply} compare the estimation error achieved by SGD, \sgddd, and the proposed SGD-ER method. Note that, as expected by our theorems, the decay regime starts at $d\sqrt{\tmix}$ for SGD-ER and $d\tmix$ for SGD which is similar to rate of \sgddd . At about $50,000$ samples, SGD-ER's bias term becomes smaller than the variance term, hence we observe a straight line post that point. Also, according to Theorem~\ref{thm:exp-replay} the variance at final point should be about $2\sigma^2 d^2/(\epsilon T)\approx 2e-9$, which matches the empirically observed error. We present results for higher dimensions in the appendix.

\section{Conclusion}
In this paper, we obtain the fundamental limits of performance/minimax rates that are achievable in linear least squares regression problem with Markov chain data. Furthermore, we discuss algorithms that achieve these rates (\sgddd and Parallel SGD). In the general agnostic noise setting, we show that any algorithm suffers by a factor of $\tmix$ in both bias and variance, compared to the i.i.d. setting. In the independent noise setting, the minimax rate for variance can be improved to match that of the i.i.d. setting but standard SGD method with constant step size still suffers from a worse rate. 
Finally, we study a version of the popular  technique `experience replay' used widely for RL in the noiseless Gaussian AR setting and show that it achieves a significant improvement over the vanilla SGD with constant step size. Overall, our results suggest that instead of considering the general class of optimization problems with arbitrary Markov chain data (where things cannot be improved by much), it may be useful to identify and focus on important special cases of Markovian data, where novel algorithms with nontrivial improvements might be possible.

\section*{Broader Impact}
We build foundational theoretical groundwork for the fundamental problem of optimization with Markovian data. We think that our work sheds light on the possibilities and impossibilities in this space. For practitioners, our focus on the popular SGD algorithm provides them with a rigorously justified understanding of what SGD can achieve and for specially structured chains, experience replay with SGD can be provably helpful (though not in the general case). We also think that the proof techniques in this paper could impact future research in this space and beyond.
\bibliographystyle{unsrt}  

\bibliography{bibliography}

\begin{thebibliography}{10}

\bibitem{tsitsiklis1997analysis}
John~N Tsitsiklis and Benjamin Van~Roy.
\newblock Analysis of temporal-diffference learning with function
  approximation.
\newblock In {\em Advances in neural information processing systems}, pages
  1075--1081, 1997.

\bibitem{kushner2003stochastic}
H.~Kushner and G.G. Yin.
\newblock {\em Stochastic Approximation and Recursive Algorithms and
  Applications}.
\newblock Stochastic Modelling and Applied Probability. Springer New York,
  2003.

\bibitem{Mokkadem88}
Abdelkader Mokkadem.
\newblock Mixing properties of {ARMA} processes.
\newblock {\em Stochastic Processes and their Applications}, 29(2):309 -- 315,
  1988.

\bibitem{DuchiAJJ12}
John~C. Duchi, Alekh Agarwal, Mikael Johansson, and Michael~I. Jordan.
\newblock Ergodic mirror descent.
\newblock {\em {SIAM} Journal on Optimization}, 22(4):1549--1578, 2012.

\bibitem{bhandari2018finite}
Jalaj Bhandari, Daniel Russo, and Raghav Singal.
\newblock A finite time analysis of temporal difference learning with linear
  function approximation.
\newblock In {\em Conference On Learning Theory}, pages 1691--1692, 2018.

\bibitem{srikant2019finite}
R~Srikant and Lei Ying.
\newblock Finite-time error bounds for linear stochastic approximation and {TD}
  learning.
\newblock In {\em Conference on Learning Theory}, pages 2803--2830, 2019.

\bibitem{daskalakis2019regression}
Constantinos Daskalakis, Nishanth Dikkala, and Ioannis Panageas.
\newblock Regression from dependent observations.
\newblock In {\em Proceedings of the 51st Annual ACM SIGACT Symposium on Theory
  of Computing}, pages 881--889, 2019.

\bibitem{dagan2019learning}
Yuval Dagan, Constantinos Daskalakis, Nishanth Dikkala, and Siddhartha Jayanti.
\newblock Learning from weakly dependent data under dobrushin’s condition.
\newblock In {\em Conference on Learning Theory}, pages 914--928, 2019.

\bibitem{adhikari2013introductory}
Ratnadip Adhikari and Ramesh~K Agrawal.
\newblock An introductory study on time series modeling and forecasting.
\newblock {\em arXiv preprint arXiv:1302.6613}, 2013.

\bibitem{mnih2015human}
Volodymyr Mnih, Koray Kavukcuoglu, David Silver, Andrei~A Rusu, Joel Veness,
  Marc~G Bellemare, Alex Graves, Martin Riedmiller, Andreas~K Fidjeland, Georg
  Ostrovski, et~al.
\newblock Human-level control through deep reinforcement learning.
\newblock {\em Nature}, 518(7540):529--533, 2015.

\bibitem{schaul2015prioritized}
Tom Schaul, John Quan, Ioannis Antonoglou, and David Silver.
\newblock Prioritized experience replay.
\newblock {\em arXiv preprint arXiv:1511.05952}, 2015.

\bibitem{andrychowicz2017hindsight}
Marcin Andrychowicz, Filip Wolski, Alex Ray, Jonas Schneider, Rachel Fong,
  Peter Welinder, Bob McGrew, Josh Tobin, OpenAI~Pieter Abbeel, and Wojciech
  Zaremba.
\newblock Hindsight experience replay.
\newblock In {\em Advances in neural information processing systems}, pages
  5048--5058, 2017.

\bibitem{van2000asymptotic}
Aad~W Van~der Vaart.
\newblock {\em Asymptotic statistics}, volume~3.
\newblock Cambridge university press, 2000.

\bibitem{levin2017markov}
David~A Levin and Yuval Peres.
\newblock {\em Markov chains and mixing times}, volume 107.
\newblock American Mathematical Soc., 2017.

\bibitem{jain2018parallelizing}
Prateek Jain, Sham Kakade, Rahul Kidambi, Praneeth Netrapalli, and Aaron
  Sidford.
\newblock Parallelizing stochastic gradient descent for least squares
  regression: mini-batching, averaging, and model misspecification.
\newblock {\em Journal of machine learning research}, 18, 2018.

\bibitem{mourtada2019exact}
Jaouad Mourtada.
\newblock Exact minimax risk for linear least squares, and the lower tail of
  sample covariance matrices.
\newblock {\em arXiv preprint arXiv:1912.10754}, 2019.

\bibitem{sutton2018reinforcement}
Richard~S Sutton and Andrew~G Barto.
\newblock {\em Reinforcement learning: An introduction}.
\newblock MIT press, 2018.

\bibitem{bhatia97}
Rajendra Bhatia.
\newblock {\em Matrix Analysis}, volume 169.
\newblock Springer, 1997.

\bibitem{csiszar2006context}
Imre Csisz{\'a}r and Zsolt Talata.
\newblock Context tree estimation for not necessarily finite memory processes,
  via {BIC} and {MDL}.
\newblock {\em IEEE Transactions on Information theory}, 52(3):1007--1016,
  2006.

\bibitem{paulin2015concentration}
Daniel Paulin.
\newblock Concentration inequalities for {M}arkov chains by {M}arton couplings
  and spectral methods.
\newblock {\em Electronic Journal of Probability}, 20, 2015.

\end{thebibliography}

\clearpage
\appendix

\section{Sharp Upper Bounds via. SGD-type Algorithms}

\subsection{SGD with Data Drop for Agnostic Noise Setting}\label{subsec:agn_sgd_dd}
In this section, we modify SGD so that despite having constant step size, the algorithm converges to the optimal solution as $t\rightarrow \infty$ even if the noise in each observation $\noise_t(X)$ can depend on $X$. The modified algorithm is known as SGD with data drop (\sgddd, Algorithm~\ref{alg_SGDDD}): fix $K \in \mathbb{N}$ and run SGD on samples  $X_{Kr}$ for $r \in \mathbb{N}$, and ignore the other samples. Theorem~\ref{thm:agn_sgd_dd} below shows that if $K =\Omega(\tmix \log{T})$, then the error is $O(\frac{\tmix \log T}{T})$. Combined with the lower bounds in Theorems~\ref{thm:agn_minimax} and~\ref{thm:bias-lb}, this implies that \sgddd is optimal up to log factors -- in particular, the mixing time must appear in the rates. The analysis simply bounds the distance between the iterates of SGD with independent samples and the respective iterates of \sgddd with Markovian samples.

We now formally describe the algorithm and result. Given samples from an exponentially ergodic finite state Markov Chain, $\MC$ with stationary distribution $\pi$ and mixing time $\tmix$, for $T \in \mathbb{N}$ we obtain data $(X_t,Y_t)_{t=1}^{T}$ corresponding to the states of the Markov chain $X_1 \to \cdots \to X_T \sim \MC$. We pick $K = \tmix \lceil L\log_2{T}\rceil$ for some constant $L > 0 $ to be fixed later. For the sake of simplicity we assume that $T/K$ is an integer. 

\begin{algorithm}
\caption{\sgddd}
\label{alg_SGDDD}
\begin{algorithmic}
\REQUIRE  $T \in \mathbb{N}$ , $(X_1,Y_1),\dots, (X_T,Y_T) \in  \R^d\times \R$ , step size $\alpha >0$, initial point $w_1 \in \R^d$, drop number $K \leq T$

\FOR{t in range [1, $T/K$]}
\STATE Set \begin{equation}\label{e:sgd-dd}
w_{t+1} \leftarrow w_t - \alpha X_{tK}\left(\langle w_t,X_{tK}\rangle - Y_{tK}\right)\,.
\end{equation}
\ENDFOR

\RETURN $\hat{w} \leftarrow \frac{2K}{T}\sum_{s = T/2K+2}^{T/K + 1}w_{s}\,.$
\end{algorithmic}
\end{algorithm}

We now present our theorem bounding the bias and variance for \sgddd. 
\begin{theorem}[\sgddd]\label{thm:agn_sgd_dd}
Let $MC$ be any exponentially mixing ergodic finite state Markov Chain with stationary distribution $\pi$ and mixing time $\tmix$. For $T \in \mathbb{N}$ we obtain data $(X_t,Y_t)_{t=1}^{T}$ corresponding to the states of the Markov chain $X_1 \to \cdots \to X_T \sim \MC$. Let $\alpha$ be small enough as given in Theorem 1 of \cite{jain2018parallelizing}. Then
\begin{align*}
&\mathbb{E}[\loss(\hat{w})] - \loss(w^*) \\&\quad \leq \underbrace{ \exp\left( \frac{-\alpha T}{C \cdot L\cdot  \tmix \kappa\log_2{T}}\right) \norm{w_0 - w^*}^2 + \frac{C\cdot L\cdot \tmix \mathrm{Tr}\left(A^{-1} \Sigma\right)\log_2T}{T}}_{\text{Suboptimality for i.i.d. SGD with $T/K$ samples}} + \underbrace{\frac{16\|w_0\|^2}{T^{L-2}}+ \frac{16\alpha^2\upsilon}{T^{L-3}},}_{\text{error due to leftover correlations}}\end{align*}
where $\hat{w}$ is the output of \sgddd (Algorithm~\ref{alg_SGDDD}), $A \defeq \E_{x \sim \pi}\left[xx^{\intercal}\right]$ is the data covariance matrix and $\Sigma \defeq \E_{x \sim \pi}\left[n^2 xx^{\intercal}\right]$ is the noise covariance matrix.
\end{theorem}
\textbf{Remarks}:
\begin{itemize}[leftmargin=*]
	\itemsep 0pt
	\topsep 0pt
	\parskip 0pt
	\item The bound above has two groups of terms. The first group is the error achieved by SGD on i.i.d. samples and the second group is the error due to the fact that the samples we use are only \emph{approximately} independent.
	\item With $L=5$, the error is bounded by that of SGD on i.i.d. data plus a $O(1/T^2)$ term. 
\end{itemize}
\begin{proof}[Main ideas of the proof]
	By Lemma~\ref{lem:data_drop_independence_coupling} in Section~\ref{sec:coupling_lemmas} we can couple $(\tilde{X}_{K}, \tilde{X}_{2K},\dots,\tilde{X}_{T}) \sim \pi^{\otimes (T/K)}$ to $(X_K,X_{2K},\dots,X_{T})$ such that:
	$$\mathbb{P}\left(\tilde{X}_{K},\tilde{X}_{2K},\dots,\tilde{X}_{T}) \neq (X_{K},X_{2K},\dots,X_{T})\right) \leq \tfrac{T}{K} d(K) \leq \tfrac{T}{K}  e^{-K/\tmix}.$$
	We call the data $\left(\tilde{X}_{tK},Y_{tK}(\tilde{X}_{tK})\right)$ as $(\tilde{X}_{tK},\tilde{Y}_{tK})$ for $t = 1,\dots, \frac{T}{K}$.  We replace $(X_{tK},Y_{tK})$ in the definition of \sgddd with $(\tilde{X}_{tK},\tilde{Y}_{tK})$ (with the exogenous, contextual noise $n_{tK}(\tilde{X}_{tK})$). We call the resulting iterates $\tilde{w}_t$.
	We can first show that $\mathbb{E}\|w_t -\tilde{w}_t\|^2$ is small and hence that the guarantees for SGD with i.i.d data, run for $T/K$ steps as given in \cite{jain2018parallelizing} carry over to `SGD with Data Drop' (Algorithm~\ref{alg_SGDDD}).
We refer to Appendix~\ref{app:sgd_dd} for a detailed proof.
\end{proof}

\subsection{Parallel SGD for Independent Noise Setting}\label{subsec:parsgd}
We established in Section \ref{subsec:minimax_sgd_constant_step_independent_noise} that SGD with constant step size and averaging cannot achieve the minimax risk for least squares regression with Markovian data and independent noise \cite{van2000asymptotic}, so we propose Parallel SGD algorithm with parallelization number $K \in \mathbb{N}$ to bridge the gap. For the sake of simplicity, let $\frac{T}{2K}$ be an integer. 

In this algorithm, we run $K$ different SGD instances in parallel such that the $i$th instance of the algorithm observes $(X_{K(t-1) + i}, Y_{K(t-1)+i})$ for $t \geq 1$. Therefore, each parallel instance of SGD observes points which are $K$ time units apart and if $K \gg \tmix$, the observations used by each of the SGD instance appear to be almost independent.

\begin{algorithm}
\caption{Parallel SGD}
\label{alg_PSGD}
\begin{algorithmic}
\REQUIRE  $T \in \mathbb{N}$ , $(X_1,Y_1),\dots, (X_T,Y_T) \in  \R^d\times \R$ , step size $\alpha >0$,  parallelization number $K \leq T$, initial points $w^{(i)}_1 \in \R^d$ for $1\leq i\leq K$.

\FOR{$t$ in range [1, $T/K$]}
\FOR{$i$ in range [1, $K$]}
\STATE Set 
        $w_{t+1}^{(i)} \leftarrow w^{(i)}_t - \alpha X_{(t-1)K+i } \left(\langle X_{(t-1)K+i},w_t^{(i)}\rangle - Y_i\right)$
\ENDFOR
\ENDFOR
\RETURN 
       $ \hat{w} \leftarrow \frac{2}{T} \sum_{i=1}^{K} \sum_{t = T/2K+1}^{T/K} w_t^{(i)}$
\end{algorithmic}
\end{algorithm}

The following is the main result of this section.
\begin{theorem}[Parallel SGD]\label{thm:par-sgd}
	Consider the Parallel SGD algorithm in the independent noise setting. Let the step size $\alpha < \frac{1}{2}$ and the number of parallel instances $K \geq \tmix \lceil r\log_2(T) \rceil$ where $r>5$. Assume $T/K$ is an integer. If $\hat{w}$ is the output of the algorithm using $T$ data points, we have for a universal constant $C>0$ the bound: 
	{\small
	\begin{align*}&\mathbb{E}[\loss(\hat{w})] - \loss(w^{*}) \leq  2\left( 1-\frac{\alpha}{2\kappa}\right)^{\frac{T}{2K}}\frac{1}{\tmix  \cdot \log T}\left[\sum_{i=1}^{K}\|w_1^{(i)}-w^*\|^2\right] + \frac{Cd\sigma^2}{T}.
	\end{align*}}
\end{theorem}
Note that compared to the rate for SGD and \sgddd (Section~\ref{subsec:agn_sgd_dd}), the variance term has no dependence on $\tmix$.
The bias decay is slower by a factor of $\tmix$ compared to the i.i.d. data setting, but is optimal up to a logarithmic factor for the Markovian setting. A
complete proof can be found in Appendix~\ref{app:ind_psgd}.  
\section{Coupling Lemmas}
\label{sec:coupling_lemmas}
We give a well known characterization of total variation distance:
\begin{lemma}
	\label{lem:coupling_tv} Let $\mu$ and $\nu$ be any two probability measures over a finite set $\Omega$. Then, there exist coupled random variables $(X,Y)$, that is random variables on a common probability space, such that $X \sim \mu$, $Y \sim \nu$ and, 
	$$\mathbb{P}(X\neq Y) = \tv(\mu,\nu).$$
\end{lemma}
\begin{lemma}
	\label{lem:markov_coupling_tv}
	Let $X_0,\dots, X_t,\dots$ be a stationary finite state Markov chain $\MC$ with stationary distribution $\pi$. For arbitrary $r,s \in \mathbb{N}$, consider the following random variable: $$Y_{t,r,s}:= (X_{t+r},X_{t+r+1},\dots,X_{t+r+s}).$$ Then, we have:
	$$\tv(\law(X_t,Y_{t,r,s}),\pi \otimes \law(Y_{t,r,s}) )\leq \dmix(r),$$
	where $d(r)$ is the mixing metric as defined in Section~\ref{sec:notation}.
\end{lemma}
\begin{proof}
	Using the fact that $X_t \sim \pi$ and by definition of total variation distance, we have: $$\tv(\law(X_t,Y_{t,r,s}),\pi \otimes \law(Y_{t,r,s}) ) = \sum_{x\in \Omega} \pi(x) \tv(\law(Y_{t,r,s}|X_t = x), \law(Y_{t,r,s}))\,.$$
	 By the Markov property, $\tv(\law(Y_{t,r,s}|X_t = x), \law(Y_{t,r,s})) = \tv(\law(X_{t+r}|X_t = x), \law(X_{t+r})) = \tv(\law(X_{t+r}|X_t = x), \pi)$. Lemma now follows from the definition of $\dmix(r)$.
\end{proof}
\begin{lemma}\label{lem:data_drop_independence_coupling}
	Let $X_0,\dots, X_t,\dots$ be a stationary finite state Markov chain $\MC$ with stationary distribution $\pi$. Let $K, n \in \mathbb{N}$. Then,
	$$\tv\left(\law(X_0,X_{K},X_{2K},\dots, X_{nK}), \pi^{\otimes (n+1)}\right) \leq n\dmix(K)\,.$$
	Furthermore, we can couple $(X_0,X_{K},\dots,X_{nK})$ and $(\tilde{X}_0,\tilde{X}_{K},\dots,\tilde{X}_{nK}) \sim \pi^{\otimes (n+1)}$ such that:
	$$\mathbb{P}\left((X_0,X_{K},\dots,X_{nK}) \neq (\tilde{X}_0,\tilde{X}_{K},\dots,\tilde{X}_{nK})\right) \leq n\dmix(K)\,.$$
\end{lemma}
\begin{proof}
	We prove this inductively. By Lemma~\ref{lem:markov_coupling_tv}, we have:
	$$\tv(\law(X_{(n-1)K},X_{nK}),\pi^{\otimes 2}) \leq \dmix(K)\,.$$
	From this it is easy to show that 
	\begin{equation} \label{eq:first_induction_case}
	\tv(\pi\otimes\law(X_{(n-1)K},X_{nK}),\pi^{\otimes 3}) = \tv(\law(X_{(n-1)K},X_{nK}),\pi^{\otimes 2}) \leq \dmix(K)\,.
	\end{equation}
	Using the notation in Lemma~\ref{lem:markov_coupling_tv}, we have:
	$$\tv(\law(X_{(n-2)K}, Y_{(n-2)K,K,K}), \pi \otimes \law(Y_{(n-2)K,K,K})) \leq \dmix(K)\,.$$
	By elementary properties of $\tv$ distance, it is clear that the $\tv$ between the respective marginals is smaller than the $\tv$ between the given measures. Therefore,
	\begin{equation}\label{eq:second_induction_case}
	\tv(\law(X_{(n-2)K}, X_{(n-1)K},X_{nK}), \pi \otimes \law(X_{(n-1)K},X_{nK}) \leq \dmix(K)\,.
	\end{equation}
	Using triangle inequality for TV distance along with \eqref{eq:first_induction_case} and~\eqref{eq:second_induction_case}, we have:
	$$\tv(\law(X_{(n-2)K}, X_{(n-1)K},X_{nK}), \pi^{\otimes 3} )\leq 2\dmix(K) \,.$$
	First part of the Lemma follows by using similar argument for all $i$, $1\leq i\leq n$.  The coupling part of the lemma then follows by Lemma~\ref{lem:coupling_tv}.
\end{proof}

\section{Minimax Lower Bounds: Proofs}\label{app:minimax_lb}

We first note some well known and useful results about the square loss. 
\begin{lemma}\label{lem:basic_loss_results}
	\begin{enumerate}
		\item
		$\mathbb{E}_{x\sim \pi}\mathbb{E}[\noise_t(x)\cdot x] = 0$
		\item
		
		$\loss(w) - \loss(w^{*}) = (w - w^*)^{\intercal}A(w - w^{*})$
		
	\end{enumerate}
\end{lemma}

\begin{proof}
	\begin{enumerate}
		\item This follows from the fact that $w^*$ is the minimizer of the square loss $\loss(w)$ and hence $\nabla \loss(w^{*}) = 0$.
		\item Clearly, $\loss(w) = w^{\intercal}Aw + \mathbb{E}_{x\sim \pi} \mathbb{E}|Y_0(x)|^2 - 2\mathbb{E}_{x\sim \pi}\mathbb{E} Y_0(x) x^{\intercal}w$. The result follows after a simple algebraic manipulation involving item 1 above.
	\end{enumerate}
\end{proof}

\subsection{General Minimax Lower Bound for Bias Decay}
\label{app:minimax_bias}

{\bf Proof sketch of Theorem~\ref{thm:bias-lb}}: The proof of Theorem~\ref{thm:bias-lb} proceeds by considering a particular Markov chain and constructing a two point Bayesian lower bound. Let $\Omega = \{e_1,e_2\} \subset \mathbb{R}^2$ where $e_1,e_2 \in\mathbb{R}^2$ are the standard basis vectors. Let $\kappa \geq 2$ be given. Fix $\delta \leq (0,1/2]$ and  define $\epsilon = \frac{\delta}{\kappa -1}$. Consider the Markov Chain $\MC_3$ defined by its transition matrix:
\begin{equation}
P_3 = \begin{bmatrix}
P_3(e_1,e_1) & P_3(e_1,e_2) \\
P_3(e_2,e_1) & P_3(e_2,e_2)
\end{bmatrix} 
= \begin{bmatrix}
1-\epsilon & \epsilon\\
\delta & 1-\delta
\end{bmatrix}
\end{equation}
Below given proposition shows that the mixing time $\tmix^{(3)}$ of this Markov chain is bounded. 
\begin{proposition}\label{prop:general_mc_mixing}
	$\tmix^{(3)} \leq \frac{C}{\kappa\epsilon} \leq \frac{C}{\delta}$ for some universal constant $C$.
\end{proposition}
We use $\MC_3$ to generate a set of points. We note that if we start in $e_1$, with probability $\sim \left(1-\frac{C}{\tau_0\kappa_0}\right)^T$, we do not visit $e_2$ for the first $T$ time steps. In this event, the algorithm does not have any information about $\langle w^{*},e_2\rangle$, giving us the lower bound.
\begin{proof}[Proof of Proposition~\ref{prop:general_mc_mixing}]
	We consider the metric:
	$$\dmixbar(t) := \sup_{i,j \in 
		\Omega}\tv(P^t(i,\cdot),P^t(j,\cdot))\,.$$
	Clearly, $\dmixbar(1) = (1-\frac{\delta \kappa}{\kappa-1}) = (1-\epsilon \kappa)$.
	
	By Lemma 4.12 in \cite{levin2017markov}, $\dmixbar(t)$ is submultiplicative. Therefore, $\dmixbar(t) \leq (1-\epsilon\kappa)^t$. Now, by Lemma 4.11 in~\cite{levin2017markov}, we conclude that $\dmix(t) \leq \dmixbar(t) \leq (1-\epsilon\kappa)^t \leq e^{-t \epsilon \kappa}$. From this we conclude that $\tmix^{(3)}\leq \frac{C}{\kappa\epsilon}$ for some universal constant $C$.
\end{proof}
\begin{proof}[Proof of Theorem~\ref{thm:bias-lb}]
	Let the stationary distribution of $MC_3$ be $\pi_3$. We can easily show that $\pi_3(1) = \frac{\delta}{\delta + \epsilon} = 1 - \tfrac{1}{\kappa}$ and $\pi_3(2) =1/ \kappa$. Let $X_1,X_2,\dots,X_T \sim \MC_3$. Consider the event $\mathcal{E}_T = \cap_{t=1}^{T}\{X_t \neq 2\}$. The event $\mathcal{E}_T$ holds if and only if the Markov chain starts in state $1$ and remains in state $1$ for the next $T$ transitions. Therefore,
	\begin{equation}\label{eq:critical_event}
	\mathbb{P}(\mathcal{E}_T) = \pi_3(1)P(1,1)^{T-1} = \left(1-\tfrac{1}{\kappa}\right)(1-\epsilon)^{T-1}
	\end{equation}

	We will first consider the case $\tau_0 \geq 2C$ for the universal constant $C$ given in Proposition~\ref{prop:general_mc_mixing}. Now, we give a two point Bayesian lower bound for the minimax error rate using the Markov chain $\MC_3$ defined above over the set $\{e_1,e_2\}$.  Consider the following two observation models associated with the markov chain $\MC_3$ - which we denote with subscripts/ superscripts $1$ and $2$ respectively. Call these models $Q_1$ and $Q_2$. Let $w_1^{*}, w_2^{*} \in \mathbb{R}^2$ and set $w_1^{*} = e_2$, $w_2^{*} = -e_2$. For $k \in \{1,2\}$, and for a stationary sequence $ X_1 \to \dots \to X_T \sim \MC_3$, we obtain the data sequence $(X_t,Y^{k}_t) \in \mathbb{R}^2 \times R$. We let $Y^k_t = \langle X_t, w_k^{*}\rangle + \eta_t$ for any sequence of noise random variables considered in the class $\mathcal{Q}$. Now, \begin{equation}
	A_3 := \mathbb{E}X_t X_t^{\intercal} = \begin{bmatrix}
	1-\tfrac{1}{\kappa} & 0 \\
	0 & \tfrac{1}{\kappa}
	\end{bmatrix}
	\end{equation}
	$A_3 \geq \frac{I}{\kappa}$ and $\kappa$ is the `condition number'. We take $\kappa = \kappa_0$ and fix $\delta$ such that $\tmix^{(3)} \leq \frac{C}{\kappa_0 \epsilon} =  \tau_0$. Here $C$ is the universal constant given in Proposition~\ref{prop:general_mc_mixing}. We see that the choice of $\delta$ above can be made using Proposition~\ref{prop:general_mc_mixing}. Clearly, $Q_1,Q_2 \in \mathcal{Q} $.

	From Lemma~\ref{lem:basic_loss_results}, it follows that for any $w \in \mathbb{R}^2$ and $k\in \{1,2\}$,  we have
	$$\mathcal{L}_{Q_k}(w) -\mathcal{L}_{Q_k}(w_k^{*}) = \frac{1}{\kappa}\|w - w_k^{*}\|^2\,.$$
	
	The following lower bound holds for the LHS of Equation~\eqref{eq:minimax_def}:
	\begin{align}
	\mathcal{L}(\mathcal{Q}) \geq \inf_{\alg \in \mathcal{A}} \frac{1}{2\kappa}\mathbb{E}\|\alg(D_{Q_1}(T))-w^{*}_1\|^2 + \frac{1}{2\kappa}\mathbb{E}\|\alg(D_{Q_2}(T))-w^{*}_2\|^2 \label{eq:two_point_lower_bound}
	\end{align}
	
	Now, we can embed $D_{Q_1}(T)$ and $D_{Q_2}(T)$ into the same probability space such that the data is generated by the same sequence of states $X_0,\dots,X_T$ and they have the same noise sequence $\eta_0,\dots,\eta_T$ almost surely. It is easy to see that conditioned on the event $\mathcal{E}_T$ described  above, $D_{Q_1}(T) = D_{Q_2}(T)$ almost surely. Under this event, $\alg(D_{Q_1}(T)) =\alg(D_{Q_2}(T)) $. Using this in 
	Equation~\eqref{eq:two_point_lower_bound}, we concude:

	\begin{align}
	\mathcal{L}(\mathcal{Q}) &\geq \inf_{\alg \in \mathcal{A}} \frac{1}{2\kappa}\mathbb{E}\|\alg(D_{Q_1}(T))-w^{*}_1\|^2\mathbbm{1}(\mathcal{E}_T) + \frac{1}{2\kappa}\mathbb{E}\|\alg(D_{Q_2}(T))-w^{*}_2\|^2\mathbbm{1}(\mathcal{E}_T) \nonumber  \\
	&=\inf_{\alg \in \mathcal{A}} \frac{1}{2\kappa}\mathbb{E}\|\alg(D_{Q_1}(T))-w^{*}_1\|^2\mathbbm{1}(\mathcal{E}_T) + \frac{1}{2\kappa}\mathbb{E}\|\alg(D_{Q_1}(T))-w^{*}_2\|^2\mathbbm{1}(\mathcal{E}_T) \nonumber\\
	&\geq \inf_{\alg \in \mathcal{A}}  \frac{ \|w^{*}_1 - w^{*}_2\|^2}{4\kappa}\mathbb{P}(\mathcal{E}_T) =   \frac{ \|w^{*}_1 - w^{*}_2\|^2}{4\kappa}  \mathbb{P}(\mathcal{E}_T) = \frac{\kappa-1}{\kappa^2}\left(1 - \epsilon\right)^{T-1} \nonumber \\
	&\geq \frac{\kappa-1}{\kappa^2}\left(1-\frac{C}{\tau_0\kappa}\right)^{T-1} = \frac{\kappa_0-1}{\kappa_0^2}\left(1-\frac{C}{\tau_0\kappa_0}\right)^{T-1} \geq \frac{\kappa_0-1}{\kappa_0^2}\left(1-\frac{C}{\tau_0\kappa_0}\right)^{T}.
	\label{eq:bias_lower_bound_final}
	\end{align}
	In the third step above, we have used the fact that  due to convexity of the map $a \to \|a\|^2$, we have $\|a-b\|^2 + \|c-b\|^2 \geq \frac{1}{2}\| a-c\|^2$ for arbitrary $a,b,c \in \mathbb{R}^d$ and Equation~\eqref{eq:critical_event} in the sixth step and the choice of $\epsilon$ and $\delta$ in the seventh step and the choice of $\kappa = \kappa_0$ in the last step.

	For the case $1 \leq \tau_0 < 2C$,  we take $X_1 \to X_2 \dots \to X_T$ to be an i.i.d seqence with distribution $\pi_3(\cdot)$ and let $\kappa = \kappa_0$. In this case, $\mathbb{P}(\mathcal{E}_T) = (1-\frac{1}{\kappa})^T$ and its mixing time is $1$. The lower bounds for this case follows using similar reasoning as above. We conclude that even when $1 \leq \tau_0 < 2C$, Equation~\eqref{eq:bias_lower_bound_final} holds.
	
\end{proof}

\subsection{Minimax Lower Bound for Agnostic Setting}
\label{app:minimax_var}

\begin{proof}[Proof of Theorem~\ref{thm:agn_minimax}]$ $\par\nobreak\ignorespaces
In the setting considered below, $\sigma^2 =  c$ for some constant $c$. But, we note that we can achieve lower bounds for more general $\sigma^2$ by scaling $w^{*}$ and $Y_t$ below simultaenously by $\sigma$. (This would also require a scaling of the lower bound on $T$ below to ensure $\|w^{*}\| \leq 1$)

Let $\bI = (I_1,\dots, I_{d}) \in \{0,1\}^{d}$. Let $\epsilon,\delta \in (0,1)$ be such that $1/2 \geq \epsilon > \delta $. We consider a collection of irreducible Markov chains, indexed by $\{0,1\}^{d}$ with a common state space $\Omega$ such that $|\Omega| = 2d$ and $\Omega \subset \mathbb{R}^d$. For now, we denote $\Omega = \{a_1,\dots,a_{2d}\}$. We denote the Markov chain corresponding to $\bI$ by $\MC_{\bI}$, the corresponding transition matrix by $P_{\bI}$, the stationary distribution by $\pi_{\bI}$ and the mixing time by $\tmix^{\bI}$. Let 

\begin{equation}
    P_{\bI}(a_i,a_j) = \begin{cases} 1 - \epsilon  &\quad\text{ if } i = j, i \leq d \text{ and } I_i = 0\\
          1 - \epsilon -\delta &\quad\text{ if } i = j, i\leq d \text{ and } I_i = 1 \\
     \frac{\epsilon}{2d-1} &\quad\text{ if } i \neq j, i\leq d \text{ and } I_i = 0 \\
     \frac{\epsilon+\delta}{2d-1} &\quad\text{ if } i \neq j, i \leq d \text{ and } I_i = 1 \\
      1 - \epsilon  &\quad\text{ if } i = j \text{ and } i \geq d\\
      \frac{\epsilon}{2d-1}  &\quad\text{ if } i \neq j \text{ and } i \geq d\\
      \end{cases}
\end{equation}

We consider the data model corresponding to each $\MC_{\bI}$. For $i \in \{1,\dots, d\}$, we take $a_i := e_i$ and $a_{d+i} := - e_i$ where $e_i$ is the standard basis vector in $\mathbb{R}^d$. We let the output corresponding to $a_i$, $Y_t(a_i)  = 1$ almost surely for $i \in \{1,2,\dots,2d\}$. Let $w_{\bI}^{*} \in \mathbb{R}^d$ the optimum corresponding to regression problem described in Equation~\eqref{eqn:prob}.  A simple computation shows that:
$$w_{\bI}^{*} = \arg\inf_{w\in \mathbb{R}^d} \sum_{i=1}^{d} \pi_{\bI}(e_i)(\langle w,e_i\rangle - 1)^2 + \pi_{\bI}(-e_{i})(\langle w,e_i\rangle + 1)^2\,.$$
Optimizing the RHS by setting the gradient to $0$, we conclude that:
$$\langle w_{\bI}^{*},e_i\rangle = \frac{\pi_{\bI}(e_i) - \pi_{\bI}(-e_i)}{\pi_{\bI}(e_i)+\pi_{\bI}(-e_i)}\,.$$
It is clear from an application of Proposition~\ref{prop:clique_walk_general} that:
\begin{equation}\label{eq:optima_splitting}
    \langle w_{\bI}^{*},e_i\rangle = \begin{cases} 0 &\quad\text{ if } I_i = 0 \\
    -\frac{\delta}{2\epsilon + \delta} &\quad\text{ if } I_i = 1
    \end{cases}
\end{equation} 
Denote $A_{\bI} = \mathbb{E}_{x \sim \pi_{\bI}} xx^{\intercal}$. It is easy to show that $A_{\bI} \succeq \frac{\mathrm{I}_d}{2d}$ from the identity given for $\pi_{\bI}$ in Proposition~\ref{prop:clique_walk_general}. Let $Q_{\bI}$ be the regression problem corresponding to $\MC_{\bI}$. We define the data set $D_{Q_{\bI}}=\{X^{(\bI)}_1,\dots, X^{(\bI)}_T\}$.

	\noindent We now consider the minimax error rate. In the equations below, we will denote $\alg(D_{Q_{\bI}}(T))$ by just $\hat{w}_{\bI}$ for the sake of clarity. From Proposition~\ref{prop:clique_walk_general}, we conclude that if we take $1/\epsilon \sim \tau_0$ then $Q_{\bI} \in \mathcal{Q}$ for every $\bI \in \{0,1\}^d$
	\begin{align}
	\mathcal{L}(\mathcal{Q}) &= \inf_{\alg \in \mathcal{A}} \sup_{Q \in \mathcal{Q}} \mathbb{E}[\loss_Q(\alg(D_Q(T)))] - \loss_Q(w_Q^{*}) \nonumber\\	&\geq \inf_{\alg \in \mathcal{A}}\sup_{\bI \in \{0,1\}^d} \mathbb{E}[\loss_{Q_{\bI}}(\hat{w}_{\bI})] - \loss_{Q_{\bI}}(w_{\bI}^{*})\nonumber\\
	&= \inf_{\alg \in \mathcal{A}}\sup_{\bI \in \{0,1\}^d}\mathbb{E}(\hat{w}_{\bI}-w_{\bI}^{*})^{\intercal} A_{\bI}(\hat{w}_{\bI}-w_{\bI}^{*})\nonumber\\
	&\geq \inf_{\alg \in \mathcal{A}} \sup_{\bI \in \{0,1\}^d} \frac{1}{2d}\mathbb{E} \|\hat{w}_{\bI}-w_{\bI}^{*}\|^2 \nonumber \\
	&\geq \inf_{\alg \in \mathcal{A}} \mathbb{E}_{\bI \sim \unif\{0,1\}^d} \frac{1}{2d} \mathbb{E}\|\hat{w}_{\bI}-w_{\bI}^{*}\|^2 \nonumber \\
	&=  \frac{1}{2d}\inf_{\alg \in \mathcal{A}} \sum_{i=1}^{d}\mathbb{E}_{\bI \sim \unif\{0,1\}^d }\mathbb{E}|\langle\hat{w}_{\bI},e_i\rangle-\langle w^{*}_{\bI},e_i\rangle|^2 \label{eq:agnostic_loss_bound_1},
	\end{align}
	The third step follows by an application of Lemma~\ref{lem:basic_loss_results}. The fourth step follows from the fact that $A_{\bI} \succeq \frac{\mathrm{I}_d}{2d} $ as shown above. In the fourth and fifth steps, the inner expectation is with respect to the randomness in the data and the outer expectation is with respect to the randomness in $I \sim \unif\{0,1\}^d$. We refer to Lemma~\ref{lem:iterative_expectation_bound}, proved below, which essentially argues that whenever $\bI,\bJ \in \{0,1\}^d$ are such that they differ only in one position, the outputs of $\MC_{\bI}$ and $\MC_{\bJ}$ have similar distribution whenever $\delta$ is `small enough' (as given in the lemma). Therefore, with constant probability, any given algorithm fails to distinguish between the data from the two Markov chains. Applying lemma~\ref{lem:iterative_expectation_bound} to Equation~\eqref{eq:agnostic_loss_bound_1}, we conclude that for some absolute constant $C,C_1,C_2$, whenever $T \geq C\frac{d^2}{\epsilon}$ and $\delta \leq C_1\sqrt{\frac{d\epsilon}{T}}$, we have $\|w^{*}_{\bI}\| \leq 1$  and:
\begin{align}
    \sup_{\bI \in \{0,1\}^d}\mathbb{E}\loss_{\bI}(\hat{w}_{\bI}) - \loss_{\bI}(w_{\bI}^{*}) &\geq C_2 \frac{\delta^2}{\epsilon^2+\delta^2}
    \end{align}
 The lower bounds follow from the equation above after noting that $\tau_0 \sim 1/\epsilon$
 
 \end{proof}
The following proposition gives a uniform bound for the mixing times for the class of Markov chains and determines their stationary distributions considered in the proof of Theorem~\ref{thm:agn_minimax} above.
\begin{proposition}\label{prop:clique_walk_general}
$\tmix^{\bI} \leq \frac{C_0}{\epsilon}$ for some universal constant $C_0$. Let $|\bI| := \sum_{i=1}^{d}I_i$. Then,

\begin{equation}
    \pi_{\bI}(a_i) = \begin{cases} \frac{\epsilon}{2d\epsilon + (2d-|\bI|)\delta} &\quad\text{ if } i \leq d \text{ and } I_i = 1 \\
    \frac{\epsilon + \delta}{2d\epsilon + (2d-|\bI|)\delta} &\quad\text{ otherwise } 
    \end{cases}
\end{equation}
\end{proposition}
\begin{proof}
Consider the distance measure for mixing:
$\dmixbar(t) = \sup_{a,b \in \Omega}\tv(P^{t}_{\bI}(a,\cdot),P^{t}_{\bI}(b,\cdot))$. A simple calculation, using the fact that $1/2 \geq \epsilon > \delta$ shows that for any $\bI \in \{0,1\}^d$, we have:
$$\dmixbar(1) \leq 1-\frac{\epsilon}{2}\,.$$
Using Lemma 4.12 in \cite{levin2017markov}, we conclude that $\dmixbar$ is submultiplicative and therefore, $\dmixbar(t) \leq (1-\tfrac{\epsilon}{2})^t$. By Lemma 4.11 in \cite{levin2017markov}, $\dmix(t) \leq \dmixbar(t) \leq (1-\tfrac{\epsilon}{2})^t$. From this inequality, we conclude the result.

The identity for the stationary distribution follows from the definition.
\end{proof}

 Suppose $\bI, \bJ \in \{0,1\}^{d}$ and that they differ only in one co-ordinate. Let $X_1^{(\bI)} \to X_2^{(\bI)} \to \dots \to X_T^{(\bI)} \sim \MC_{\bI}$ and $X_1^{(\bJ)} \to X_2^{(\bJ)} \to \dots \to X_T^{(\bJ)} \sim \MC_{\bJ}$ be stationary sequences. We will denote them as $\mathbf{X}^{(k)}$ for $k \in \{\bI,\bJ\}$ respectively.

\begin{lemma} \label{lem:adjacent_MC_tv_bound}
There exist universal constants $C,C_1$ such that whenever $T \geq C \frac{d}{\epsilon}$ and $\delta \leq C_1 \sqrt{\frac{d\epsilon}{T}} $, we have
 $$\tv(\mathbf{X}^{(\bJ)},\mathbf{X}^{(\bI)}) \leq \frac{1}{2},$$
	where $\tv(\mathbf{X}^{(\bJ)},\mathbf{X}^{(\bI)})$ is the total variation distance between random variables $\mathbf{X}^{(\bJ)}$ and $\mathbf{X}^{(\bJ)}$.
\end{lemma}

\begin{proof}
	We will bound the total variation distance between  $\mathbf{X}^{(\bJ)}$ and $\mathbf{X}^{(\bJ)}$ below by first bounding $\kl$ divergence between the sequences and then using Pinsker's inequality. Without loss of generality, we assume that $I_1=0$, $J_1 = 1$.

	Let $(z_1,\dots,z_T) \in \Omega^{T}$. Henceforth, we will denote this tuple by $\mathbf{z}$. For $k\in \{\bI,\bJ\}$, we have:
	
	\begin{equation}\label{eq:markov_chain_likelihood}
	\mathbb{P}(\mathbf{X}^{(k)} = \mathbf{z}) = \pi_k(z_1)\prod_{t=2}^{T}P_k(z_{t-1},z_t)\,.
	\end{equation}
	
	Define the function $\eta_{ab} : \Omega^{T} \to \mathbb{N}$  for $a,b\in \Omega$ by: $\eta_{ab}(\mathbf{z}) = |\{ 2\leq t \leq T : z_{t-1} = a \text{ and } z_{t}  = b \}|\,.$ $\eta_{ab}(\mathbf{z})$ counts the number of transitions from state $a$ to state $b$ in $\mathbf{z}$.  Equation~\ref{eq:markov_chain_likelihood} can be rewritten using functions $\eta_{ab}$ as:
	
	$$\mathbb{P}(\mathbf{X}^{(k)} = \mathbf{z})  = \pi_k(z_1) \prod_{a,b \in \Omega}P_k(a,b)^{\eta_{ab}(\mathbf{z})}.$$
	Abusing notation to use  $\mathbf{X}^{(k)}$ and  $\law(\mathbf{X}^{(k)})$ interchangably, and by using definition of the KL divergence,  we have: 	
	\begin{align}
	&\kl(\mathbf{X}^{(\bJ)}|| \mathbf{X}^{(\bI)}) = \sum_{\mathbf{z} \in \Omega^{T}} \mathbb{P}(\mathbf{X}^{(\bJ)} = \mathbf{z}) \log \frac{\mathbb{P}(\mathbf{X}^{(\bJ)} = \mathbf{z})}{\mathbb{P}(\mathbf{X}^{(\bI)} = \mathbf{z})}  \nonumber\\
	&= \sum_{\mathbf{z} \in \Omega^{T}} \mathbb{P}(\mathbf{X}^{(\bJ)} = \mathbf{z}) \left[\log \left(\frac{\pi_{\bJ}(z_1)}{\pi_{\bI}(z_1)}\right) +  \sum_{a,b \in \Omega} \eta_{ab}(\mathbf{z}) \log \frac{P_{\bJ}(a,b)}{P_{\bI}(a,b)} \right]\nonumber \\
	&= \kl(\pi_{\bJ}||\pi_{\bI}) + \sum_{j=1}^{d}\mathbb{E} \eta_{a_1a_j}(\mathbf{X}^{(2)}) \log \frac{P_{\bJ}(a_1,a_j)}{P_{\bI}(a_1,a_j)} \nonumber \\
	&= \kl(\pi_{\bJ}||\pi_{\bI}) + (T-1)\pi_{\bJ}(a_1)\kl(P_{\bJ}(a_1,\cdot)||P_{\bI}(a,\cdot)) \label{eq:kl_identity}
	\end{align}
	
	In the third step we have used that fact that $P_{\bJ}(a,b) \neq P_{\bI}(a,b)$ only when $a = a_1$ since $\bJ$ and $\bI$ differ only in the first co-ordinate. In the fourth step we have used the fact that $\mathbb{E} \eta_{a_1a_j}(\mathbf{X}^{(2)}) = (T-1)\pi_{\bJ}(a_1)P_{\bJ}(a_1,a_j)$

	
	For any two probability measures $P$ and $Q$ on the same finite space, the following holds by the Pinsker's inequality: 
	\begin{equation}\label{eq:pinsker_inequality}
	\tv(P,Q) \leq \sqrt{2 \kl(P||Q)}
	\end{equation}
	We now state the 'reverse Pinkser's inequality' to bound the KL divergence. 
	\begin{lemma}\label{lem:reverse_pinsker}[Lemma 6.3 in \cite{csiszar2006context}]
		Let $P_1$ and $P_2$ be probability distributions over some finite space $E$. Then,
		$$\kl(P_2||P_1) \leq \sum_{a \in E} \frac{|P_2(a)-P_1(a)|^2}{P_1(a)}\,.$$
		In particular, when $E = \{0,1\}$, $P_i = \ber(p_i)$, we have:
		
		$$\kl(P_2||P_1) \leq \frac{|p_1-p_2|^2}{p_1(1-p_1)}$$
		
	\end{lemma}

	An easy computation using Lemma~\ref{lem:reverse_pinsker} shows that for some universal constant $C_3$:
$$\kl(\pi_{\bJ}||\pi_{\bI}) \leq  \frac{C_3\delta^2}{d\epsilon^2}\,.$$
By Proposition~\ref{prop:clique_walk_general}, we have $\pi_{\bJ}(a_1) \leq 1/2d$. By a similar application of Lemma~\ref{lem:reverse_pinsker} we have:
$$\kl\left(P_{\bJ}(a_1,\cdot)||P_{\bI}(a_1,\cdot)\right) \leq \frac{\delta^2}{\epsilon(1-\epsilon)}$$
Combining these bounds with Equation~\eqref{eq:kl_identity} and using the fact that $\epsilon < 1/2$, we have, for some universal constant $C_3$,

\begin{align}
    \kl(\mathbf{X}^{(\bJ)}||\mathbf{X}^{(\bI)}) &\leq C_3\left[\frac{\delta^2}{d\epsilon^2} + \frac{T\delta^2}{d\epsilon}\right] \label{eq:kl_bound_final}
\end{align}

	We let $T \geq Cd/\epsilon$ and take $\delta \leq C_1\sqrt{\frac{d\epsilon}{T}}$ for appropriate constants $C,C_1$. Applying this in Equation~\eqref{eq:kl_bound_final} and then using Equation~\eqref{eq:pinsker_inequality}, we obtain the desired result. 
	\end{proof}

\begin{lemma}\label{lem:iterative_expectation_bound}
$T\geq \frac{Cd}{\epsilon}$ and $\delta \leq C_1\sqrt{\frac{d\epsilon}{T}}$
For any output $\hat{w}_{\bI}$ (as decribed in the proof of Theorem~\ref{thm:agn_minimax}),
 $$\mathbb{E}_{\bI \sim \unif\{0,1\}^d }\mathbb{E}|\langle\hat{w}_{\bI},e_i\rangle-\langle w^{*}_{\bI},e_i\rangle|^2 \geq \frac{\delta^2}{8(2\epsilon+\delta)^2}$$
\end{lemma}
\begin{proof}
Let $\bI_{\sim i}$ denote all the co-ordinates of $\bI$ other than $i$, let $\bI_{~i}^{+} \in \{0,1\}^d$ be such that its $i$-th co-ordinate is $1$ and the rest of the co-ordinates are $\bI_{\sim i}$. Similarly $\bI_{~i}^{-}$ be such that its $i$-th co-ordinate is $0$ and the rest of the co-ordinates are $\bI_{\sim i}$. 

By Lemma~\ref{lem:coupling_tv} and Lemma~\ref{lem:adjacent_MC_tv_bound}, we conclude that whenever $T\geq \frac{Cd}{\epsilon}$ and $\delta \leq C_1\sqrt{\frac{d\epsilon}{T}}$ we can couple the sequences $\mathbf{X}^{(\bI_i^{+})}$ and $\mathbf{X}^{(\bI_i^{-})}$ such that:
$$\mathbb{P}(\{\mathbf{X}^{(\bI_i^{+})} = \mathbf{X}^{(\bI_i^{-})}\}) \geq \frac{1}{2},.$$	
	Define the event $\mathcal{E}^{\prime}_T := \{D_{Q_1}(T) = D_{Q_2}(T)\}$. 
\begin{align}
&\mathbb{E}_{\bI \sim \unif\{0,1\}^{d-1} }\mathbb{E}|\langle\hat{w}_{\bI},e_i\rangle-\langle w^{*}_{\bI},e_i\rangle|^2 = \mathbb{E}_{\bI_{\sim i} \sim \unif\{0,1\}^{d-1}} \mathbb{E}_{I_i \sim \unif\{0,1\}} \mathbb{E}|\langle\hat{w}_{\bI},e_i\rangle-\langle w^{*}_{\bI},e_i\rangle|^2 \nonumber \\
&= \mathbb{E}_{\bI_{\sim i} \sim \unif\{0,1\}^{d-1}}
\frac{1}{2}\left[\mathbb{E}|\langle\hat{w}_{\bI^{+}_i},e_i\rangle-\langle w^{*}_{\bI_i^{+}},e_i\rangle|^2  + \mathbb{E}|\langle\hat{w}_{\bI^{-}_i},e_i\rangle-\langle w^{*}_{\bI_i^{-}},e_i\rangle|^2\right] \nonumber \\
&\geq \frac{1}{2} \mathbb{E}_{\bI_{\sim i} \sim \unif\{0,1\}^{d-1}}
\mathbb{E}\left[|\langle\hat{w}_{\bI^{+}_i},e_i\rangle-\langle w^{*}_{\bI_i^{+}},e_i\rangle|^2  + |\langle\hat{w}_{\bI^{-}_i},e_i\rangle-\langle w^{*}_{\bI_i^{-}},e_i\rangle|^2\right]\mathbbm{1}(\mathcal{E}^{\prime}_T)\nonumber \\
&\geq \frac{1}{4}\mathbb{E}_{\bI_{\sim i} \sim \unif\{0,1\}^{d-1}} \left[|\langle w^{*}_{\bI^{-}_i},e_i\rangle-\langle w^{*}_{\bI_i^{+}},e_i\rangle|^2 \right]\mathbb{P}(\mathcal{E}^{\prime}_T) \nonumber\\
&= \frac{\delta^2}{8(2\epsilon+\delta)^2}
\end{align}
In the fourth step we have used the fact that in the event $\mathcal{E}_T^{\prime}$, that is when $\mathbf{X}^{(\bI^{+}_i)} = \mathbf{X}^{({\bI}_i^{-})}$, the corresponding outputs of the algorithm are the same. That is $\hat{w}_{\bI}^{+} = \hat{w}_{\bI}^{-}$. We have also used the convexity of the map $x \to \|x\|^2$ to show that $\|a-b\|^2 + \|b-c\|^2 \geq \frac{\|a-c\|^2}{2}$. In the last step, we have used  Equation~\eqref{eq:optima_splitting}.
\end{proof}

\section{SGD algorithms: Proofs}\label{app:sgd_algorithms}
\subsection{SGD with Constant Step Size suffers Asymptotic Bias in the Agnostic Setting}
\label{app:sgd_bias}

\begin{proof}[Proof of Theorem~\ref{thm:agn_sgd}]
Fix $\epsilon \in (0,1)$. We describe the Markov chain $\MC_1$ over the space $\Omega = \{a,b\} \subset \mathbb{R}$ and the corresponding data model that we consider.  Let the corresponding stationary distribution be $\pi_1$, mixing time be $\tmix^1$ and the transition matrix be $P_{1}$, given by:
$$P_1(a,a) = P_1(b,b) = 1-\epsilon , P_1(a,b) = P_1(b,a) = \epsilon \,.$$ 

It is clear from Proposition~\ref{prop:clique_walk_general} with $d=1$ and $\delta = 0$ that $\tmix^1 \leq C/\epsilon$ for some universal constant $C$ and that the stationary distribution is uniform over $\Omega$. We set $a = 1/2$ and $b = -1$. The output $Y_t(a) = Y_t(b) = 1/2$ almost surely. It is easy to show that the corresponding optimal parameter $w_1^{*} = -\frac{1}{5}$. Let the $\textrm{SGD}_{\alpha}$ be run on an instance using the data from $\MC_1$ as described above. We call the iterates $w_{t}$.

We will first bound the Wasserstein distance between $w_{t+1}$ and $w_t$. Let $ X_1 \to X_2 \to \dots \to X_T \dots \sim \MC_1$ be a stationary sequence. We consider another stationary sequence $\tilde{X}_1 \to \tilde{X}_2 \to \dots \sim \MC_1$ such that $\tilde{X}_{t} = X_{t+1}$ for every $t \geq 1$ almost surely. We can run the SGD with data from the chain $X_t$ or from the chain $\tilde{X}_t$. Let the data corresponding to $\tilde{X}_t$ be $(\tilde{X}_t,\tilde{Y}_t)$. We let the iterates be $w_t$ and $\tilde{w}_t$ respectively and start both from the same initial point $w_1$. Now, $w_t$ and $\tilde{w}_t$ are identically distributed. For $t\geq 1$, consider:
	$$w_{t+2}-\tilde{w}_{t+1} = w_{t+1} - \tilde{w}_{t} - \alpha (X_{t+1} X_{t+1}^{\intercal}w_{t+1} - X_{t+1} Y_{t+1}) + \alpha (\tilde{X}_t \tilde{X}_t^{\intercal}w_t - \tilde{X}_t \tilde{Y}_t)$$
	
	Clearly, $\tilde{X}_t = X_{t+1}$ and $\tilde{Y}_t = Y_{t+1}$ almost surely. Hence, 
	$$w_{t+2}-\tilde{w}_{t+1} =(1-\alpha \tilde{X}_t \tilde{X}_t^{\intercal})\left( w_{t+1} - \tilde{w}_{t}\right)\,.$$
	
	Now, $\tilde{X}_t \in \mathbb{R}$ and $|\tilde{X}_t|^2 \geq \frac{1}{4}$ almost surely. Therefore, when $\alpha \in (0,1)$ we have: 
	\begin{equation}\label{eq:almost_sure_contraction}
	|w_{t+2}-\tilde{w}_{t+1}|^2 \leq \left(1-\tfrac{\alpha}{4}\right)^2| w_{t+1} - \tilde{w}_{t}|^2\,.
	\end{equation}
	Applying the above inequality for $t$ iterations and by applying expectation on both sides: 
	Therefore we conclude that: 	
	$$\mathbb{E}\left[ |w_{t+1}-\tilde{w}_{t}|^2 \right]\leq e^{-t\alpha/2 }\mathbb{E}\left[|w_1-w_2|^2\right] = e^{-t\alpha/2 }|w_2|^2 \leq e^{-(t-1)\alpha/2 }.$$
	Applying Jensen's inequality to the LHS and using the fact that $\tilde{w}_t$
	has the same distribution as $w_{t}$, we get:
	$$|\mathbb{E}w_{t+1} - \mathbb{E}w_t| \leq \frac{\alpha}{2}e^{-(t-1)\alpha/4}.$$
	Similarly, since Equation~\eqref{eq:almost_sure_contraction} holds almost surely, we have for $k\in \{1,2\}$:
	\begin{equation}\label{eq:conditional_contraction}
	\mathbb{E}\left[\bigr|w_{t+2}-\tilde{w}_{t+1}\bigr|^2\bigr|X_{t+1} = k\right] \leq \frac{\alpha}{2}e^{-(t-1)\alpha/2 }\,.
	\end{equation}
	Using the fact that $X_{t+1} = \tilde{X}_{t}$ almost surely, we conclude that: $$\mathbb{E}\left[\tilde{w}_{t+1}|X_{t+1} = k\right] = \mathbb{E}\left[\tilde{w}_{t+1}|\tilde{X}_{t} = k\right] = \mathbb{E}\left[w_{t+1}|X_{t} = k\right].$$
	Using the equation above and applying Jensen's inequality to Equation~\eqref{eq:conditional_contraction}, we obtain:
	\begin{equation}\label{eq:conditional_expectation_cauchy}
	\bigr|\mathbb{E}\left[w_{t+1}|X_{t} = k\right] -\mathbb{E}\left[w_{t}| X_{t-1}=k \right]\bigr|\leq \frac{\alpha}{2}e^{-(t-1)\alpha/4 }\,.
	\end{equation}
	For the sake of simplicity, we will denote $\mathbb{E}\left[w_{t}|X_{t-1} = k\right]$ by $e_k$ and $\mathbb{E}\left[w_{t+1}|X_{t} = k\right]$ by $e_k + \lambda_k$, where $|\lambda_k| \leq \frac{\alpha}{2}e^{-\alpha (t-1)/4}$. We hide the dependence on $t$ for the sake of clarity.
	
	Taking conditional expectation with respect to the event $X_{t} = 1$ in the recursion in Algorithm~\ref{alg_SGD}, we have:
	\begin{equation}\label{eq:expectation_consistency}
	e_1 + \lambda_1 = \left(1-\tfrac{\alpha}{4}\right) \mathbb{E}\left[w_t|X_{t} = 1\right] + \frac{\alpha}{4} \,.
	\end{equation}

	Now, consider: 
	\begin{align*}
	\mathbb{E}\left[w_t|X_{t} = 1\right] &= 2 \mathbb{E}\left[w_t \mathbbm{1}(X_t = 1)\right]=  2 \mathbb{E}\left[w_t \mathbbm{1}(X_t = 1)\mathbbm{1}(X_{t-1}=1)+w_t \mathbbm{1}(X_t = 1)\mathbbm{1}(X_{t-1}=2)\right] \\
	&=  \mathbb{E}\left[w_t \mathbbm{1}(X_t = 1)\bigr|X_{t-1}=1\right] + \mathbb{E}\left[w_t \mathbbm{1}(X_t = 1)\bigr|X_{t-1}=2\right] = (1-\epsilon)e_1 + \epsilon e_2. 
	\end{align*}

	Using this in Equation~\eqref{eq:expectation_consistency}, we conclude: 
	$$e_1 + \lambda_1 = (1- \alpha/4)\left[(1-\epsilon)e_1 + \epsilon e_2\right] + \alpha/4$$
	Similarly, we have:
	$$e_2 + \lambda_2 = (1-\alpha)\left[\epsilon e_1 + (1-\epsilon)e_2\right] - \alpha/2$$
	Using the above two equations, we have 
	\begin{equation}
	\begin{bmatrix}
	\alpha/4 + \epsilon - \alpha \epsilon/4 & -\epsilon(1-\alpha/4) \\
	-(1-\alpha)\epsilon & \alpha + \epsilon - \alpha\epsilon 
	\end{bmatrix} \begin{bmatrix}
	e_1 \\ e_2
	\end{bmatrix}
	= \begin{bmatrix}
	\alpha/4 + \lambda_1\\
	-\alpha/2 + \lambda_2
	\end{bmatrix}
	\end{equation}
	Solving the equations above we get:
	\begin{equation}
	\begin{bmatrix}
	e_1 \\ e_2
	\end{bmatrix}
	= \begin{bmatrix}
	\frac{\alpha/2 - \alpha\epsilon/4 -\epsilon/2}{\alpha/2 + 5\epsilon/2 - \alpha\epsilon} \\
	\frac{-\alpha/4 - \alpha\epsilon/4 -\epsilon/2}{\alpha/2 + 5\epsilon/2 -\alpha\epsilon}
	\end{bmatrix} + O(C(\alpha,\epsilon)e^{-t\alpha/4} ).
	\end{equation}
	As $\mathbb{E}[w_t] = \frac{e_1+e_2}{2}$,   
	\begin{equation}\label{eq:asymptotic_expectation}
	\mathbb{E}[w_t] = \frac{1}{2}\left[\frac{\alpha/4 -\alpha\epsilon/2 -\epsilon}{\alpha/2 - \alpha\epsilon + 5\epsilon/2 }\right] + O(C(\alpha,\epsilon)e^{-t\alpha/4}) \,.
	\end{equation}
	It is easy to check that when $\epsilon = 1/2$, $X_0, X_1,\dots, $ is infact a sequence of i.i.d $\ber(1/2)$ random variables and we'd expect $w_t$ to be an unbiased estimator as $t\to \infty$. This can be verified by plugging in $\epsilon = 1/2$ in Equation~\eqref{eq:asymptotic_expectation}. When $\epsilon = 1/4$, the corresponding value becomes $\frac{1}{2}\frac{\alpha -2}{2\alpha +5} + o_t(1)$, which does not tend to $w_1^{*} = -1/5$ as $t\to\infty$.
\end{proof}

\subsection{A Lower Bound for SGD with Constant Step Size in the Independent Noise Setting}
\label{app:sgd_lb}

\begin{proof}[Proof of Theorem~\ref{thm:sgd-lb}]

Recall the class of Markov chains $\MC_{\bI}$ for $\bI \in \{0,1\}^d$ defined in the proof of Theorem~\ref{thm:agn_minimax} in Appendix~\ref{app:minimax_var}. We consider a similar Markov chain $\MC_0$ with state space $\Omega = \{e_1,\dots,e_d\}$. Let its transition matrix be $P_0$, the stationary distribution be $\pi_0$ and the mixing time be $\tmix^0$. Let $0 <\epsilon < 1/2$. We define:
\begin{equation}
P_0(e_i,e_j) = \begin{cases}
1-\epsilon &\quad \text{ if } i=j\\
\frac{\epsilon}{d-1} &\quad \text{ if } i\neq j
\end{cases}
\end{equation}

Through steps analogous to the proof of Proposition~\ref{prop:clique_walk_general}, we can show that $\tmix^0 \leq \frac{C}{\epsilon}$ for some universal constant $C$. It follows from definitions that $\pi_0$ is the uniform distribution over $\Omega$.
 
 Let $X_1 \to X_2 \dots \to X_T \sim \MC^0$ is a stationary seqence.  We let $w_0^{*} = 0$ and let the output be $Y_t = \langle X_t,w_0^{*}\rangle + \noise_t = \noise_t$ such that $\noise_t \sim \mathcal{N}(0,\sigma^2)$. Since this is the independent noise case, $\noise_t$ is taken to be i.i.d. and independent of $X_t$. The matrix $A_0 = \mathbb{E}X_tX_t^{\intercal}  = \frac{\id_d}{d}$. Consider the SGD algorithm with iterate averaged output which achieves the information theoretically optimal rates in the i.i.d data case. 
Suppose $(X_t,Y_t)_{t=1}^{T}$ is drawn from the model associated with $\MC_0$ described above. The evolution equations become:
\begin{equation} \label{eq:independent_SGD_evolution}
\langle w_{t+1},e_i\rangle = \begin{cases} 
\langle w_t,e_i\rangle &\quad\text{ if } X_t \neq e_i\\
(1-\alpha)\langle w_t,e_i\rangle + \alpha \noise_t &\quad\text{ if } X_t = e_i
\end{cases}
\end{equation}

Let the averaged output of SGD be $\hat{w} := \frac{2}{T}\sum_{t = T/2+1}^{T}w_t$. Now we will directly give a lower bound for the excess loss of the estimator $\hat{w}$ for $w_0^*$. We let $w_0 = 0$. For the problem under consideration, $w_0^* = 0$ and $A_0 = \frac{\id_d}{d}$. Therefore, using Lemma~\ref{lem:basic_loss_results}
\begin{align}
\mathbb{E}\loss(\hat{w}) - \loss(w_0^*) &= \frac{1}{d} \mathbb{E} \|\hat{w}\|^2\nonumber \\
&= \frac{4}{T^2d} \sum_{t,s= T/2 +1}^{T} \mathbb{E}\langle w_t ,w_s\rangle \nonumber\\
&= \frac{4}{T^2d} \sum_{t,s= T/2 +1}^{T}\sum_{i=1}^{d} \mathbb{E}\langle w_t ,e_i\rangle\langle w_s,e_i\rangle
 \label{eq:loss_correlation_equation}
\end{align}

Consider the case $s > t$. For $i \in \{1,\dots,d\}$, let $N_i(s-1,t) := |\{ t\leq l \leq s-1: X_{l} = e_i\}|$ and let $t \leq t_1 < \dots < t_{N_i(s-1,t)} \leq s-1$ be the sequence of times such that $X_{t_p}  = e_i$. We have $\langle w_s,e_i\rangle = (1-\alpha)^{N_i(s-1,t)}\langle w_t,e_i\rangle + \sum_{p = 1}^{N_i(t,s)} (1-\alpha)^{N_i(t,s) - p}\alpha\noise_{t_p}$. Therefore, multiplying by $\langle w_t,e_i\rangle$ on both sides and taking expectation, we conclude:

\begin{align}
\mathbb{E}\langle w_t,e_i\rangle\langle w_s,e_i \rangle &= \mathbb{E} (1-\alpha)^{N_i(s-1,t)} |\langle w_t,e_i\rangle|^2 \nonumber\\
&\geq  \sum_{j\neq i}\mathbb{E}\left[ |\langle w_t,e_i \rangle|^2\bigr|X_{t-1} = j, N_i(s-1,t) = 0\right]\mathbb{P}\left(X_{t-1} = j, N_i(s-1,t) = 0\right)\nonumber \\
&= \sum_{j\neq i}\mathbb{E}\left[ |\langle w_t,e_i\rangle|^2\bigr|X_{t-1} =j\right]\mathbb{P}\left(X_{t-1} =j, N_i(s-1,t) = 0\right)\nonumber\\
&= \sum_{j\neq i}\mathbb{E}\left[ |\langle w_t,e_i\rangle|^2\bigr|X_{t-1} =j\right]\mathbb{P}\left( N_i(s-1,t) = 0\bigr|X_{t-1} = j\right)\mathbb{P}(X_{t-1} = j)\nonumber\\
&= \sum_{j\neq i}\frac{(1-\frac{\epsilon}{d-1})^{s-t}}{d}\mathbb{E}\left[ |\langle w_t,e_i\rangle|^2\bigr|X_{t-1} =j\right] \label{eq:correlation_decay}
\end{align}
The first equality follows from fact that $\noise_l$ are i.i.d mean $0$ and independent of $X_l$. In the second step we have used the fact that conditioned on the event $N_2(s-1,t) = 0$, $(1-\alpha)^{N_2(s-1,t)} = 1$. In the third step we have used the fact that $w_{t}$ depends only on $X_1,\dots,X_{t-1}$, and $\noise_1,\dots, \noise_{t-1}$ and $N_2(s-1,t)$ depends only on $X_t,\dots, X_{s-1}$ and therefore are conditionally independent given $X_{t-1}$. The last step follows from the fact that $\mathbb{P}\left( N_i(s-1,t) = 0\bigr|X_{t-1} = j\right) = (1-\frac{\epsilon}{d-1})^{s-t}$.

Using Equation~\eqref{eq:correlation_decay} in Equation~\eqref{eq:loss_correlation_equation}, we have:

\begin{align}
\mathbb{E}\loss(\hat{w}) - \loss(w^*) 
&\geq \frac{2}{T^2d^2} \sum_{t= T/2 +1}^{T} \sum_{s=t}^{T}\sum_{i=1}^{d}\sum_{j\neq i}(1-\tfrac{\epsilon}{d-1})^{s-t}\mathbb{E}\left[ |\langle w_t,e_i\rangle|^2\bigr|X_{t-1} =j\right]\nonumber \\
&= \frac{2(d-1)}{T^2d^2\epsilon}\sum_{t= T/2 +1}^{T}\sum_{i=1}^{d}\sum_{j\neq i}\left(1- (1-\tfrac{\epsilon}{d-1})^{T-t+1}\right) \mathbb{E}\left[ |w_t|^2\bigr|X_{t-1} =0\right] \nonumber \\
&\geq \frac{2(1- (1-\tfrac{\epsilon}{d-1})^{T/4})(d-1)}{d^2T^2\epsilon}\sum_{t= T/2 +1}^{3T/4}\sum_{i=1}^{d}\sum_{j\neq i} \mathbb{E}\left[ |\langle w_t,e_i\rangle|^2\bigr|X_{t-1} =j\right] \nonumber \\
&\geq \frac{c\alpha\sigma^2(d-1)^2}{Td\epsilon(2-\alpha)} \left[1- O\left((1-\tfrac{\epsilon}{d-1})^{T/4} + d(1-\alpha)^{T/2d} + de^{-T/36d^2\tmix^0}\right)\right] \nonumber \\
&\geq \frac{c^{\prime}\alpha\tmix^0\sigma^2 d}{T(2-\alpha)} \left[1- O\left((1-\tfrac{\epsilon}{d-1})^{T/4} + d(1-\alpha)^{T/2d} + de^{-T/36d^2\tmix^0}\right)\right]
\end{align}
Where $c$ in the third step is some positive universal constant. In the third step we have used Lemma~\ref{lem:conditional_second_moment}. In the last step we have used the bounds on $\tmix^0$. This establishes the lower bound.
\end{proof}

\begin{lemma}\label{lem:conditional_second_moment}
	For $j \in \{1,\dots,d\}$,
	$$ \frac{\alpha\sigma^2}{2-\alpha}\left(1 - d(1-\alpha)^{\tfrac{t}{d}} - de^{-\tfrac{t}{72d^2\tmix^0}} \right)  \leq \mathbb{E}\left[|\langle e_i,w_{t+1}\rangle|^2 \bigr|X_{t} = e_j \right] \leq  \frac{\alpha\sigma^2}{2-\alpha}  \,.$$
\end{lemma}
\begin{proof}
	It is clear from Equation~\eqref{eq:independent_SGD_evolution} that $$\langle w_{t+1},e_i\rangle = \sum_{s = 1}^{N_i(t)} (1-\alpha)^{N_i(t)-s} \alpha \epsilon_{t_s}\,.$$
	Where $N_2(t) = |\{1 \leq l \leq t : X_l = e_i\}|$ and $1 \leq t_1 \leq t_2 \dots t_{N_2(t)} \leq t$ is the increasing and exhaustive sequence of times such that $X_{t_s} = e_i$. We understand an empty summation to be $0$.
	Therefore we have:
	\begin{align}
	&\mathbb{E}\left[|\langle w_{t+1},e_i\rangle|^2 \bigr|X_{t} = e_j \right] \nonumber\\&= \sum_{n=0}^{t}\mathbb{E}\left[ \sum_{s,p = 1}^{n}  (1-\alpha)^{2n-s-p} \alpha^2\epsilon_{t_s}\epsilon_{t_p}\biggr|X_{t } = e_j,N_i(t) = n\right]\mathbb{P}(N_i(t) = n|X_t = e_j) \nonumber\\
	&= \sum_{n=0}^{t}\mathbb{E}\left[ \sum_{s = 1}^{n}  (1-\alpha)^{2n-2s} \alpha^2\sigma^2\biggr|X_{t } = e_j,N_i(t) = n\right]\mathbb{P}(N_i(t) = n|X_t = e_j) \nonumber\\
	&=\alpha^2 \sigma^2\mathbb{E}\left[ \frac{1- (1-\alpha)^{2N_i(t)})}{1- (1-\alpha)^2}\biggr|X_t = e_j\right] \\
	&= \frac{\alpha\sigma^2}{2-\alpha} \left(1- \mathbb{E}\left[(1-\alpha)^{2N_i(t)}\bigr|X_t = e_j\right]\right)
	\end{align}
	In the second step we have used the fact that the sequence $(\epsilon_s)$ is i.i.d mean $0$ and independent of the sequence $(X_s)$. It is now sufficient to show that$$\mathbb{E}\left[(1-\alpha)^{2N_i(t)}\bigr|X_t = e_j\right]  \to 0$$ as $t \to \infty$.

	Clearly, $\mathbb{E}N_i(t) = t/d$. We will now bound $\mathbb{E}(1-\alpha)^{2N_i(t)}$.  By a direct application of Corollary 2.10 in  \cite{paulin2015concentration}, we conclude that for any $x\geq 0$
	
	$$\mathbb{P}(N_i(t)\leq \mathbb{E}N_i(t) - x ) \leq \exp\left(-\frac{2x^2}{9t\tmix^0}\right)\,.$$
	Taking $x = t/2d$, we conclude:
	\begin{equation}\label{eq:empirical_distribution_concentration}
	\mathbb{P}(N_i(t)\leq \tfrac{t}{2d} ) \leq \exp\left(-\frac{t}{18d^2\tmix^0}\right)
	\end{equation}
	Now consider:
	\begin{align}
	\mathbb{E}(1-\alpha)^{2N_i(t)} &\leq \mathbb{E}(1-\alpha)^{\tfrac{t}{d}}\mathbbm{1}(N_i(t) \geq t/2d) + \mathbb{E} \mathbbm{1}(N_i(t) \leq \tfrac{t}{2d}) \nonumber\\
	&\leq  (1-\alpha)^{\tfrac{t}{d}} + \mathbb{P} (N_2(t) \leq \tfrac{t}{2d}) \nonumber \\
	&\leq  (1-\alpha)^{\tfrac{t}{d}} + \exp\left(-\tfrac{t}{18d^2\tmix^0}\right)  \nonumber
	\end{align}
	In the last step we have used Equation~\eqref{eq:empirical_distribution_concentration}.  Now,
	\begin{align*}
	\mathbb{E}\left[(1-\alpha)^{2N_i(t)}\bigr|X_t = e_j\right] &= \frac{1}{\mathbb{P}(X_t = e_j)}\mathbb{E}\left[(1-\alpha)^{2N_i(t)} \mathbbm{1}(X_t = k)\right] \\
	&\leq  \frac{1}{\mathbb{P}(X_t = e_j)}\mathbb{E}\left[(1-\alpha)^{2N_i(t)} \right] \\
	&= d\mathbb{E}\left[(1-\alpha)^{2N_i(t)} \right] \\
	&\leq d(1-\alpha)^{\tfrac{t}{d}} + d\exp\left(-\tfrac{t}{18d^2\tmix^0}\right) 
	\end{align*}
	
	From this the result of the lemma follows.

\end{proof}

\subsection{SGD with Data Drop is Unbiased and Minimax Optimal in the Agnostic Setting}\label{app:sgd_dd}
\begin{proof}[Proof of Theorem~\ref{thm:agn_sgd_dd}]

\noindent Let $(\tilde{X}_{K}, \tilde{X}_{2K},\dots,\tilde{X}_{T}) \sim \pi^{\otimes (T/K)}$ and let $\tilde{w}_t$ be $t$-th iterate of standard SGD when applied to $(\tilde{X}_{K}, \tilde{X}_{2K},\dots,\tilde{X}_{T})$. 

Define $\Delta_t := w_t  - \tilde{w}_t$. We will bound $\mathbb{E}\|\Delta_t\|^2$ for every $t$.  Clearly, if $\{(\tilde{X}_{K},\tilde{X}_{2K},\dots,\tilde{X}_{T}) = (X_{K},X_{2K},\dots,X_{T})\}$, then $\Delta_t = 0$. We call this event $\mathcal{C}$.  In the event $\mathcal{C}^c$, we use the coarse bound given in Lemma~\ref{lem:iterate_deviation} to bound $\Delta_t$.

We have the following comparison theorem between i.i.d SGD and Markovian \sgddd. Recall that,q $$\hat{w} = \frac{2K}{T}\sum_{s = T/2K+2}^{T/K + 1}w_{s},\quad \hat{\tilde{w}} = \frac{2K}{T}\sum_{s = T/2K+2}^{T/K + 1}\tilde{w}_{s}\,.$$
Using Lemma~\ref{lem:basic_loss_results}, we have:
\begin{align}
\loss(\hat{w}) - \loss(w^*) &= (\hat{w}- w^*)^{\intercal}A(\hat{w} - w^{*})=  (\hat{w}- \hat{\tilde{w}} +  \hat{\tilde{w}}-w^*)^{\intercal}A(\hat{w}- \hat{\tilde{w}} +  \hat{\tilde{w}} - w^{*}) \nonumber\\
&\hspace*{-30pt}= 4(\tfrac{\hat{w}- \hat{\tilde{w}}}{2} +  \tfrac{\hat{\tilde{w}}-w^*}{2})^{\intercal}A(\tfrac{\hat{w}- \hat{\tilde{w}}}{2} +  \tfrac{\hat{\tilde{w}}-w^*}{2}) \leq 2 (\hat{\tilde{w}}- w^*)^{\intercal}A(\hat{\tilde{w}} - w^{*})  +  2(\hat{w}-\hat{\tilde{ w}})^{\intercal}A(\hat{w} - \hat{\tilde{w}}) \nonumber\\
&\hspace*{-30pt}\leq 2 (\hat{\tilde{w}}- w^*)^{\intercal}A(\hat{\tilde{w}} - w^{*})  +  2\|\hat{w}-\hat{\tilde{ w}}\|^2\label{eq:loss_comparison_bound}
\end{align}
In the fourth step we have used the fact that $A$ is a PSD matrix and hence $z \to z^{\intercal}Az$ is a convex function. In the fifth step we have used the fact that $\|A\|_{\mathrm{op} }\leq 1$. 

Now, to conclude the statement of the theorem from the equation above, we need to bound $\mathbb{E}\left[\|\hat{w}-\hat{\tilde{ w}}\|^2\right]$. By an application of Jensen's inequality, it is clear that: $$\mathbb{E}\left[\|\hat{w}-\hat{\tilde{ w}}\|^2\right] \leq \sup_{ \tfrac{T}{2K} + 2 \leq t \leq \frac{T}{K} +1}\mathbb{E}\left[ \|w_s - \tilde{w}_s\|^2\right]\,.$$

Now, under the event $\mathcal{C}$, $w_s - \tilde{w}_s = 0$ and under the event $\mathcal{C}^c$, we use the bounds on $\mathbb{E}\left[\|\Delta_t\|^2\bigr| \mathcal{C}^c\right]$ given in Lemma~\ref{lem:iterate_deviation} to conclude:
\begin{align}
\mathbb{E}\left[\|\hat{w}-\hat{\tilde{ w}}\|^2\right] &\leq \sup_{ \tfrac{T}{2K} + 2 \leq s \leq \frac{T}{K} +1} \mathbb{E}\left[ \|w_s - \tilde{w}_s\|^2\right] =  \sup_{ \tfrac{T}{2K} + 2 \leq s \leq \frac{T}{K} +1}\mathbb{P}(\mathcal{C}^c)\mathbb{E}\left[ \|w_s - \tilde{w}_s\|^2\bigr|\mathcal{C}^c\right] \nonumber \\
&\leq \left[4(T/K+1)T/K\|w_1\|^2 + 4(T/K)^2(T/K + 1) \alpha^2 \upsilon\right] e^{-K/\tmix}. \label{eq:square_deviation_bound_data_drop}
\end{align}
For $T,K \geq 3$, we have $T/K  + 1 \leq T$ and by definition of $K$, $e^{-K/\tmix} \leq \frac{1}{T^L}$. Combining this with Equations~\eqref{eq:loss_comparison_bound} and~\eqref{eq:square_deviation_bound_data_drop}, we have:

$$\mathbb{E}[\loss(\hat{w})] - \loss(w^*) \leq 2\left[\mathbb{E}[\loss(\hat{\tilde{w}})] - \loss(w^*)\right] + \frac{8\|w_0\|^2}{T^{L-2}}+ \frac{8\alpha^2\upsilon}{T^{L-3}}.$$
\end{proof}

\begin{lemma}\label{lem:iterate_deviation}
	Fix a sequence $\{x_{K}, x_{2K}, \dots, x_T\}$ in $\Omega$. Call this vector $\mathbf{x}$. Similarly, we let $\mathbf{X}$ and $\mathbf{\tilde{X}}$ respectively denote $(X_{tK})_{t=1}^{T/K}$ and $(\tilde{X}_{tK})_{r=1}^{T/K}$ where $X$ and $\tilde{X}$ are as defined in the proof of Theorem~\ref{thm:agn_sgd_dd}. Now, the following holds for any $\alpha \leq 1$:   
	\begin{enumerate}
		\item $ \mathbb{E}\left[\|w_t\|^2|\mathbf{X} = \mathbf{x}, \mathbf{\tilde{X}} = \mathbf{\tilde{x}}\right] \leq t\|w_1\|^2 + t(t-1)\alpha^2\upsilon \,.$
		\item $\mathbb{E}\left[\|\tilde{w_t}\|^2|\mathbf{X} = \mathbf{x}, \mathbf{\tilde{X}} = \mathbf{\tilde{x}}\right] \leq t\|w_1\|^2 + t(t-1)\alpha^2\upsilon \,.$
		
		\noindent We recall that $\upsilon$ is the uniform bound on $\mathbb{E}\| x Y_t(x)\|^2$
		as given in Section~\ref{sec:main_problem}. Therefore,
		$$\mathbb{E}\left[\|\Delta_t\|^2|\mathcal{C}^c\right] \leq 4t\|w_1\|^2 + 4t(t-1)\alpha^2 \upsilon \,.$$
	\end{enumerate}

\end{lemma}

\begin{proof}
	We will prove the inequality given in item 1. The inequality given in item 2 follows similarly. Define the matrices $B_s = \id - \alpha X_{sK}X_{sK}^{\intercal}$ and $E_s = X_{sK}y_{sK}$. Clearly, $w_{s+1} = B_sw_{s} + \alpha  E_s$.
	
	Clearly, $\|B_s\|_{\mathrm{op}} \leq 1$ almost surely. Therefore, almost surely:
	$\|w_{s+1}\| \leq \|w_s\| + \alpha \|E_s\|\,.$
	Summing the telescoping series from $1$ to $ t -1$, we have almost surely:$\|w_{t}\| \leq \|w_1\| + \sum_{s=1}^{t-1}\alpha \|E_s\|\,. $ 
	By Jensen's inequality,
	$$\|w_{t}\|^2 \leq t \|w_1\|^2 + t\sum_{s=1}^{t-1}\alpha^2 \|E_s\|^2\,.$$
	Lemma items 1 and 2 now follow by taking the necessary conditional expectation on both sides, using the uniform bound $\mathbb{E}[\| xy_t(x)\|^2] \leq \upsilon$ for all $x \in \Omega$ and using $t \in \mathbb{N}$ 
	as given in Section~\ref{sec:main_problem}.  The conditional expectation bound follows from the fact that $\|\Delta_t\|^2 = \|w_t - \tilde{w}_t\|^2 \leq 2\|w_t\|^2 + 2 \|\tilde{w}_t\|^2$ and using the bounds in items 1 and 2.
\end{proof}

\subsection{Parallel SGD accelerates Noise Decay in the Independent Noise Setting}\label{app:ind_psgd}
%

For ease of notation, for the rest of the section, we define:
$$X_{l,i} := X_{(l-1)K+i},\ \ \  \epsilon_{l,i} := \epsilon_{(l-1)K+i},\ \ \ \ \Gamma^{(i)}_{t,s} := \prod_{l = s}^{l=t-1}(I - \alpha \left(X_{l,i} X_{l,i}^{\intercal}\right)),$$ 
where $t\geq s+1$. For $t = s$, we use the convention that this product denotes $I$. 
We unroll the recursion in Algorithm~\ref{alg_PSGD} to show:
\begin{equation}
w_{t}^{(i)} = w^{*}+ \Gamma^{(i)}_{t,1}(w_1^{(i)} - w^{*}) + \alpha\sum_{l=1}^{t-1}\epsilon_{l,i}\Gamma^{(i)}_{t,l+1}X_{l,i}=w^{\bias}_{t,i}+w^{\var}_{t,i},
\label{eq:SGD_unwind}
\end{equation}
where $w^{\bias}_{t,i} := \Gamma^{(i)}_{t,1}(w_1^{(i)} - w^{*})$ and $w^{\var}_{t,i} = w^{*}+\alpha\sum_{l=1}^{t-1}\epsilon_{l,i}\Gamma^{(i)}_{t,l+1}X_{l,i}$. 


We first state elementary results to understand the bias and the variance term of $w_{t,i}^{\var}$.
\begin{lemma}\label{lem:basic_results}
	
	\begin{enumerate}
		\item $w_{t,i}^{\var}$ is the output of SGD when $w_1^{(i)} = w^{*}$ and $\mathbb{E}w_{t,i}^{\var} = w^{*}$
		\item Every entry of $w_{t,i}^{\var}$ is uncorrelated with every entry of $w_{s,j}^{\bias}$ for every $t,i,s,j$.
		\item Every entry of $w_{t,i}^{\var}$ is uncorrelated with every entry of $w_{s,j}^{\var}$ for every $t,s$ when $i\neq j$
	\end{enumerate}
	
\end{lemma}
\begin{proof}
	\begin{enumerate}
		\item This follows from Equation~\ref{eq:SGD_unwind} and the fact that $\epsilon_{l,i}$ are mean $0$ random variables independent of $a_{l,i}$.
		\item This follows from the fact that $\epsilon_t$ are i.i.d. mean 0 and independent of the Markov chain.
		\item The proof is similar to the proof of item 2.
	\end{enumerate}
\end{proof}
Define $$\hat{w}^{\bias} := \frac{2}{T} \sum_{i=1}^{K} \sum_{t = T/2K+1}^{T/K} w_{t,i}^{\bias},\qquad\hat{w}^{\var} := \frac{2}{T} \sum_{i=1}^{K} \sum_{t = T/2K+1}^{T/K} w_{t,i}^{\var}\,.$$ 
Now, $\hat{w} = \hat{w}^{\bias} + \hat{w}^{\var}$, where $\hat{w}$ is the output of the parallel SGD algorithm (Algorithm~\ref{alg_PSGD}). The following lemma follows from a simple application of item 2 of Lemma~\ref{lem:basic_loss_results} and Lemma~\ref{lem:basic_results}.
\begin{lemma}\label{lem:bias_variance_decomposistion}
	$$\mathbb{E}[\loss(\hat{w})] = \mathbb{E}\left(\hat{w}^{\var}-w^{*}\right)^{\intercal}A\left(\hat{w}^{\var}-w^{*}\right) + \mathbb{E}\left[ \left(\hat{w}^{\bias}\right)^{\intercal}A\left(\hat{w}^{\bias}\right)\right]+ \loss(w^*)$$
\end{lemma}

We will bound the two terms in the above lemma separately.
Bound for each of the terms is provided in Appendix~\ref{app:ind_psgd_bias} and Appendix~\ref{app:ind_psgd_var}, respectively.

\subsubsection{The Bias Term}\label{app:ind_psgd_bias}
In this section we will show that bias decays exponentially in $T$ when $K$ is large enough. Define sigma algebra $\mathcal{F}_{t,i} := \sigma(\epsilon_{s,i},X_{s,i}: 1\leq s\leq t)$

\begin{lemma}\label{lem:conditional_operator_contraction}
	Let $K >\tmix \lceil r\log_2{T}\rceil$ and $\Gamma^{(i)}_{t,s}$ and other notation be as defined in Section~\ref{subsec:parsgd}. Then, 
	\begin{enumerate}
		\item For every $t>s$, $\mathbb{E}\left[\Gamma^{(i)}_{t,s}\bigr|\mathcal{F}_{t-2,i}\right] = (I-\alpha A + E_t)\Gamma^{(i)}_{t-1,s} $ where $E_t$ is a random matrix such that $\|E_t\| \leq \frac{\alpha}{T^r}$ almost surely.
		\item For every random vector $X \in \mathcal{F}_{t-2,i}$ such that $\mathbb{E}\|X\|^2 < \infty$. Let $\alpha < 1$ and $T^r > 2\kappa$. we have:
		$$\mathbb{E} ||\Gamma^{(i)}_{t,t-1}X||^2 \leq \bigr(1-\tfrac{\alpha}{2\kappa}\bigr)\mathbb{E}\|X\|^2.$$
	\end{enumerate}
\end{lemma}
\begin{proof}

		 \noindent We first observe that $\Gamma^{(i)}_{t,s} = \left[I-\alpha X_{i,t-1}X_{i,t-1}^{\intercal}\right]\Gamma^{(i)}_{t-1,s}$ and $\Gamma^{(i)}_{t-1,s} \in \mathcal{F}_{t-2,i}$. Therefore, 
		\begin{equation}
		\mathbb{E}\left[\Gamma^{(i)}_{t,s}\bigr|\mathcal{F}_{t-2,i}\right] = \mathbb{E}\left[I-\alpha X_{i,t-1}X_{i,t-1}^{\intercal}\bigr|\mathcal{F}_{t-2,i}\right]\Gamma^{(i)}_{t-1,s} \,.
		\label{eq:cond_expectation_basic}
		\end{equation}
		Let $P^K$ denote the law of $X_{(t-1)K+i}$ conditioned on $\mathcal{F}_{t-2,i}$. From equation~\eqref{eq:binary_mixing}, $\tv(P^{K},\pi) \leq \frac{1}{T^r}$ almost surely. Now, 
		\begin{align}
		 &\mathbb{E}\left[I-\alpha X_{i,t-1}X_{i,t-1}^{\intercal}\bigr|\mathcal{F}_{t-2,i}\right] = I - \alpha \sum_{x \in \Omega} \left[x x^{\intercal}\right]P^{K}(x) \nonumber\\&= I -\alpha \sum_{x \in \Omega} \left[x x^{\intercal}\right]\pi(x) + \alpha \sum_{x \in \Omega} x x^{\intercal}(P^{K}(x)-\pi(x)) = I -\alpha A + \alpha \sum_{x \in \Omega} x x^{\intercal}(P^{K}(x)-\pi(x)) \label{eq:cond_expectation_term}
		\end{align}
		We take $E_t := \alpha \sum_{x \in \Omega} x x^{\intercal}(P^{K}(x)-\pi(x)) $. Define the event, $\mathcal{A} := \{x\in \Omega: P^K(x) \geq \pi(x)\}\,.$ For any arbitrary $\theta \in \mathbb{R}^d$, we have: 
		\begin{align}
		|\theta^{\intercal}E_t \theta| &= \alpha \bigr|\sum_x \langle x,\theta\rangle^2(P^{K}(x)-\pi(x))\bigr| \nonumber \\
		&\leq \alpha \max\left(\sum_{x\in \mathcal{A}}\langle x,\theta\rangle^2(P^{K}(x)-\pi(x)),\sum_{x\in \mathcal{A}^c}\langle x,\theta\rangle^2(\pi(x)-P^{K}(x)) \right) \nonumber \\
		&\leq \alpha \|\theta\|^2\tv(P^{K},\pi) \leq \frac{\alpha}{T^r}\|\theta\|^2 \text{ a.s.} \label{eq:as_operator_bound}
		\end{align}
		
		In the third step above, we have used the fact that $\|x\| \leq 1$. First part of the Lemma now follows using Equations~\eqref{eq:cond_expectation_term}, ~\eqref{eq:as_operator_bound} with Equation~\eqref{eq:cond_expectation_basic}.
		
		Next, we consider: 
		\begin{align}
		\mathbb{E}\left[\|\Gamma^{(i)}_{t,t-1}X\|^2\right]  &=\mathbb{E}\left[ X^{\intercal}\mathbb{E}\left[I-2\alpha X_{t-1,i}X_{t-1,i}^{\intercal} + \alpha^2\|X_{t-1,i}\|^2X_{t-1,i}X_{t-1,i}^{\intercal}\bigr|\mathcal{F}_{t-2,i}\right]X \right] \nonumber\\
		&\leq \mathbb{E}\left[ X^{\intercal}\mathbb{E}\left[I-(2\alpha-\alpha^2) X_{t-1,i}X_{t-1,i}^{\intercal} \bigr|\mathcal{F}_{t-2,i}\right]X \right] \label{eq:bias_contraction_basic}
		\end{align}
		
		In the second step we have used the fact that $\|X_{t-1,i}\| \leq 1$ almost surely. Substituting $s = t-1$ and replacing $\alpha$ by $2\alpha - \alpha^2$ in item 1 above, we conclude: $$\mathbb{E}\left[I-(2\alpha-\alpha^2) X_{t,i}X_{t,i}^{\intercal} \bigr|\mathcal{F}_{t-2,i}\right] = I - (2\alpha - \alpha^2)A + E_t,$$ where $\|E_t\| \leq \frac{2\alpha-\alpha^2}{T^r}$ a.s.
		Combining the above equation with \eqref{eq:bias_contraction_basic} and using $A \succeq \frac{1}{\kappa}I$, we obtain:
		\begin{align}
		\mathbb{E}\left[\|\Gamma^{(i)}_{t,t-1}X\|^2\right]  &\leq \mathbb{E}\left[ X^{\intercal}\mathbb{E}\left[I-(2\alpha-\alpha^2)A + E_t \bigr|\mathcal{F}_{t-2,i}\right]X \right] \nonumber \\
		&\leq \bigr(1-(2\alpha-\alpha^2)(\tfrac{1}{\kappa}-\tfrac{1}{T^r})\bigr)\mathbb{E}\left[\|X\|^2\right] \leq \bigr(1-\tfrac{\alpha}{2\kappa}\bigr)\mathbb{E}\left[\|X\|^2\right] \label{eq:bias_contraction}
		\end{align}
		Second part of the Lemma now follows by using the above equation. 
\end{proof}
%
%
\begin{lemma}\label{lem:bias_loss_bound}
	Let the data $X_1, \dots, X_T$ be generated from an exponentially mixing Markov Chain $\MC$.  
	Let $K > \tmix \lceil r\log_2{T}\rceil$, $\alpha < 1$ and $T^r > 2\kappa$ for some $r >0$. Then, we have:
	\begin{equation}
	\mathbb{E}[\|w^{\bias}_{t,i}  \|^2] \leq \bigr(1-\tfrac{\alpha}{2\kappa}\bigr)^{t-1}\|w_1^{(i)}\|^2
	\label{eq:bias_bound}
	\end{equation}
	
	Consequently, \begin{equation}\label{eq:bias_loss_bound}
	\mathbb{E}\left(\hat{w}^{\bias}\right)^{\intercal}A\left(\hat{w}^{\bias}\right) \leq \left( 1-\frac{\alpha}{2\kappa}\right)^{\frac{T}{2K}}\frac{1}{K}\left[\sum_{i=1}^{K}\|w_1^{(i)}-w^*\|^2\right]
	\end{equation}
	
\end{lemma}
\begin{proof}[Proof of Lemma~\ref{lem:bias_loss_bound}]
	Clearly, $w^{\bias}_{t,i} = \Gamma^{(i)}_{t,t-1}\left(w^{\bias}_{t-1,i}\right)$. It is clear that $w^{\bias}_{t-1,i}  \in \mathcal{F}_{t-2,i}$. Applying item 2 in Lemma~\ref{lem:conditional_operator_contraction}, we get: $\mathbb{E}\|w^{\bias}_{t,i}\|^2 \leq \bigr(1-\tfrac{\alpha}{2\kappa}\bigr)\mathbb{E}\|w^{\bias}_{t-1,i}\|^2\,.$
	By induction, we conclude Equation~\eqref{eq:bias_bound}.
	
	Now, $\left(\hat{w}^{\bias}\right)^{\intercal}A\left(\hat{w}^{\bias}\right) \leq \|\hat{w}^{\bias}\|^2$ since $\|A\|_{\mathrm{op}}\leq 1$. By Jensen's inequality,
	$$\|\hat{w}^{\bias}\|^2 \leq  \frac{2}{T} \sum_{i=1}^{K} \sum_{t = T/2K+1}^{T/K} \|w_{t,i}^{\bias}\|^2.$$
	Lemma now follows by taking expectation on both sides and using Equation~\eqref{eq:bias_bound}.
\end{proof}

We have therefore bound the bias term of Theorem~\ref{thm:par-sgd} in Lemma~\ref{lem:bias_loss_bound}.

\subsubsection{The Variance Term}\label{app:ind_psgd_var}

We will now bound the variance term. 
It is clear that,
$$\bigr(\hat{w}^{\var}-w^{*}\bigr)^{\intercal}A\bigr(\hat{w}^{\var}-w^{*}\bigr) = \frac{4}{T^2}\sum_{i,j=1}^{K}\sum_{t,s = T/2K+1}^{T/K}  \bigr(w_{t,i}^{\var}-w^{*}\bigr)^{\intercal}A\bigr(w_{s,j}^{\var}-w^{*}\bigr)$$

Using the item 3 in Lemma~\ref{lem:basic_results}, we get:
\begin{equation}
\mathbb{E}\left[\bigr(\hat{w}^{\var}-w^{*}\bigr)^{\intercal}A\bigr(\hat{w}^{\var}-w^{*}\bigr)\right] = \frac{4}{T^2}\sum_{i=1}^{K}\sum_{t,s = T/2K+1}^{T/K}  \mathbb{E}\left[\bigr(w_{t,i}^{\var}-w^{*}\bigr)^{\intercal}A\bigr(w_{s,i}^{\var}-w^{*}\bigr)\right]
\label{eq:intermediate_variance_identity}
\end{equation}

Consider the following term in the RHS of Equation~\eqref{eq:intermediate_variance_identity}

\begin{align}
\mathbb{E}\left[\bigr(w_{t,i}^{\var}-w^{*}\bigr)^{\intercal}A\bigr(w_{s,i}^{\var}-w^{*}\bigr) \right]&= \alpha^2 \sum_{l=1}^{t-1}\sum_{m=1}^{s-1} \mathbb{E}\left[\epsilon_{l,i}\epsilon_{m,i} X_{l,i}^{\intercal}\left(\Gamma^{(i)}_{t,l+1}\right)^{\intercal} A \Gamma^{(i)}_{s,m+1}X_{m,i}\right] \nonumber \\
&= \alpha^2 \sigma^2 \sum_{l=1}^{\min(t-1,s-1)} \mathbb{E} \left[X_{l,i}^{\intercal}\left(\Gamma^{(i)}_{t,l+1}\right)^{\intercal} A \Gamma^{(i)}_{s,l+1}X_{l,i}\right], \label{eq:variance_term_correlation}
\end{align}
where  the last step holds as $\epsilon_t$ are i.i.d.,  mean zero random variables with variance $\sigma^2$ and are independent of $(X_s)_{s\in \mathbb{N}}$.
For the sake of clarity, we will take $i$ to sum from $1$ to $K$ and $t,s$ to sum from $T/2K+1$ to $T/K$ in the equations below without stating this explicitly. Using equations~\eqref{eq:intermediate_variance_identity} and~\eqref{eq:variance_term_correlation}, we conclude:
\begin{equation}
\mathbb{E}\left[\bigr(\hat{w}^{\var}-w^{*}\bigr)^{\intercal}A\bigr(\hat{w}^{\var}-w^{*}\bigr)^{\intercal}\right] = \frac{4\alpha^2\sigma^2}{T^2} \sum_{i,t,s}\sum_{l=1}^{\min(t-1,s-1)}\mathbb{E} \left[X_{l,i}^{\intercal}\left(\Gamma^{(i)}_{t,l+1}\right)^{\intercal} A \Gamma^{(i)}_{s,l+1}X_{l,i}\right].  \label{eq:variance_identity}
\end{equation}
Now, we bound RHS above using the following lemma:
\begin{lemma}\label{lem:summation_1}
	Let $\alpha <1$, $K \geq \tmix\lceil r\log_2{T}\rceil$ and $l\leq s-1$. Then:
	\begin{equation}
	\sum_{t= s}^{T/K} \mathbb{E} X_{l,i}^{\intercal}\left(\Gamma^{(i)}_{t,l+1}\right)^{\intercal} A \Gamma^{(i)}_{s,l+1}X_{l,i} \leq \frac{\alpha}{4K^2T^{r-2}} + \frac{1}{\alpha}\mathbb{E}\|\Gamma^{(i)}_{s,l+1}X_{l,i}\|^2
	\end{equation}
\end{lemma}
\begin{proof}
	For $t=s$, $ X_{l,i}^{\intercal}\left(\Gamma^{(i)}_{t,l+1}\right)^{\intercal} A \Gamma^{(i)}_{s,l+1}X_{l,i} = X_{l,i}^{\intercal}\left(\Gamma^{(i)}_{s,l+1}\right)^{\intercal} A \Gamma^{(i)}_{s,l+1}X_{l,i}.$ Similarly, for $t>s$, we have: $\Gamma_{t,l+1}^{(i)} = \Gamma^{(i)}_{t,t-1}\Gamma^{(i)}_{t-1,l+1}$. Clearly, $X_{l,i}, \Gamma^{i}_{s,l+1}$ and $\Gamma^{i}_{t-1,l+1}$ are $\mathcal{F}_{t-2,i}$ measurable. Therefore, {\small
	\begin{align}
	&\mathbb{E} \left[X_{l,i}^{\intercal}\left(\Gamma^{(i)}_{t,l+1}\right)^{\intercal} A \Gamma^{(i)}_{s,l+1}X_{l,i}\right] = \mathbb{E}\left[X_{l,i}^{\intercal}\left(\Gamma^{(i)}_{t-1,l+1}\right)^{\intercal}\mathbb{E}\left[\Gamma^{(i)}_{t,t-1}\bigr|\mathcal{F}_{t-2,i}\right] A \Gamma^{(i)}_{s,l+1}X_{l,i}\right]\nonumber \\
	&= \mathbb{E}\left[X_{l,i}^{\intercal}\left(\Gamma^{(i)}_{t-1,l+1}\right)^{\intercal}\left[I-\alpha A+E_t\right] A \Gamma^{(i)}_{s,l+1}X_{l,i}\right]\leq \mathbb{E}\left[X_{l,i}^{\intercal}\left(\Gamma^{(i)}_{t-1,l+1}\right)^{\intercal}\left[I-\alpha A\right] A \Gamma^{(i)}_{s,l+1}X_{l,i}\right] + \frac{\alpha}{T^r}, \label{eq:geometric_recursion}
	\end{align}}
where the second step follows using Lemma~\ref{lem:conditional_operator_contraction}. The third step follows as $\|E_t\| \leq \frac{\alpha}{T^r}$ almost surely and the fact that $\|A\|_{\mathrm{op}},\|X_{l,i}\|,\|\Gamma^{(i)}_{a,b}\| \leq 1 $. Continuing in a similar way as Equation~\eqref{eq:geometric_recursion}, we have:
	\begin{equation}
	\mathbb{E} \left[X_{l,i}^{\intercal}\left(\Gamma^{(i)}_{t,l+1}\right)^{\intercal} A \Gamma^{(i)}_{s,l+1}X_{l,i} \right]\leq \mathbb{E}\left[X_{l,i}^{\intercal}\left(\Gamma^{(i)}_{s,l+1}\right)^{\intercal}\left[I-\alpha A\right]^{t-s} A \Gamma^{(i)}_{s,l+1}X_{l,i}\right] + \frac{\alpha(t-s)}{T^r}. \label{eq:geometric_recursion_unwind}
	\end{equation}
	Now, $(I -\alpha A)^{t-s}$ is a PSD matrix and it commutes with $A$. Therefore, from Equation~\eqref{eq:geometric_recursion_unwind}, it follows that:
	\begin{align}
	 \sum_{t= s}^{T/K-1} \mathbb{E} &\left[X_{l,i}^{\intercal}\left(\Gamma^{(i)}_{t,l+1}\right)^{\intercal} A \Gamma^{(i)}_{s,l+1}X_{l,i}\right] \nonumber \\ 
	&\leq  \sum_{t= s}^{T/K-1} \left[\mathbb{E}X_{l,i}^{\intercal}\left(\Gamma^{(i)}_{s,l+1}\right)^{\intercal}\left[I-\alpha A\right]^{t-s} A \Gamma^{(i)}_{s,l+1}X_{l,i} + \frac{\alpha(t-s)}{T^r}\right] \nonumber\\
	&\leq  \mathbb{E}\left[X_{l,i}^{\intercal}\left(\Gamma^{(i)}_{s,l+1}\right)^{\intercal}\left[\sum_{t= s}^{\infty}(I-\alpha A)^{t-s}\right] A \Gamma^{(i)}_{s,l+1}X_{l,i}\right] + \frac{\alpha T^2}{4K^2T^r} \nonumber \\
	&= \frac{1}{\alpha}\mathbb{E}X_{l,i}^{\intercal}\left(\Gamma^{(i)}_{s,l+1}\right)^{\intercal} \Gamma^{(i)}_{s,l+1}X_{l,i} + \frac{\alpha}{4K^2T^{r-2}} \label{eq:tail_sum_inequality} 
	\end{align}
	
	In the third step we have used the fact that $\sum_{i=0}^{\infty}(I-\alpha A)^i = \frac{A^{-1}}{\alpha}$. Equation~\eqref{eq:tail_sum_inequality} establishes the result of the Lemma.
\end{proof}

Consider the following operator on $\mathcal{S}(d)$ - the space of $d\times d$ symmetric matrices: 
$$\Lambda(M) = \mathbb{E}_{x\sim \pi}\mathbb{E}(I-\alpha xx^{\intercal})M(I-\alpha xx^{\intercal})\,.$$
Now, a (linear)PSD map  is a linear operator over $\mathcal{S}(d)$ which maps PSD matrices to PSD matrices. We list some important properties of $\Lambda$ below.
\begin{lemma}\label{lem:PSD_maps}
	\begin{enumerate}
		\item $\Lambda$ is a PSD map.
		\item If $A,B \in \mathcal{S}(d)$ such that $B \preceq A$ then $\Lambda(B)\preceq \Lambda(A)$.
		\item Let $M$ be a PSD operator. Then $ \|\Lambda(M)\|_2 \leq \|M\|_2$ and in particular $\Lambda(I) \preceq I$.
	\end{enumerate}
\end{lemma}

\begin{proof}
	\begin{enumerate}
		\item The proof follows from the definition of PSD matrices and PSD maps.
		\item $A-B \succeq 0$. By item 1, $\Lambda(A-B) \succeq  0$. Therefore, by linearity of $\Lambda$, $\Lambda(A) \succeq \Lambda(B)$.
		\item This follows easily from the definition of $\Lambda$ and submultiplicativity of operator norm and the fact that $\|a_x\| \leq 1$ almost surely.
	\end{enumerate}
\end{proof}

\begin{lemma}\label{lem:PSD_map_properties}
	Let $\alpha < 1$, $l \leq s-1$ and $K > \tmix\lceil r\log_2{T} \rceil$. Then:
	$$\sum_{s = l+1}^{T/K}\mathbb{E}\|\Gamma^{(i)}_{s,l+1}X_{l,i}\|^2 \leq \frac{d\left(4\alpha+2\alpha^2\right)}{K^2T^{r-2}}+ \sum_{s\geq l+1} \tr(\Lambda^{s-l-1}(A)).  $$
	Here $\Lambda^{0}$ is understood to be the identity operator.
\end{lemma}
\begin{proof}
	Consider \begin{equation}
	\|\Gamma^{(i)}_{s,l+1}X_{l,i}\|^2  = \tr\left[\Gamma^{(i)}_{s,l+1}X_{l,i}X_{l,i}^{\intercal}\left(\Gamma^{(i)}_{s,l+1}\right)^{\intercal} \right] \label{eq:trace_norm_equality}
	\end{equation}
	When $s=l+1$, it is clear that $\Gamma^{(i)}_{s,l+1} = \id$  and $\mathbb{E}\|\Gamma^{(i)}_{s,l+1}X_{l,i}\|^2 = \mathbb{E}\tr(X_{l,i}X_{l,i}^{\intercal}) = \tr(A)$. When $s > l+1$, $\Gamma^{(i)}_{s,l+1} = \Gamma^{(i)}_{s,s-1}\Gamma^{(i)}_{s-1,l+1} $. For the sake of clarity, we will denote $A_{s,l,i} := \Gamma^{(i)}_{s-1,l+1}X_{l,i}X_{l,i}^{\intercal}\left(\Gamma^{(i)}_{s-1,l+1}\right)^{\intercal} $. Since $\Gamma^{(i)}_{s-1,l+1}, X_{l,i} \in \mathcal{F}_{s-2,i}$, we have:
	\begin{align*}
	&\mathbb{E}\left[\Gamma^{(i)}_{s,s-1}A_{s,l,i}\left(\Gamma^{(i)}_{s,s-1}\right)^{\intercal} \bigr|\mathcal{F}_{s-2,i}\right] = \mathbb{E}\left[\left(I-\alpha X_{i,s-1}X_{i,s-1}^{\intercal}\right)A_{s,l,i}\left(I-\alpha X_{i,s-1}X_{i,s-1}^{\intercal}\right) \bigr|\mathcal{F}_{s-2,i}\right].
	\end{align*}
	 Let $P^{K}$ be the distribution of $X_{(s-1)K+i}$ given $\mathcal{F}_{s-2,i}$. From Equation~\eqref{eq:binary_mixing}, we have: $\tv(P^{K},\pi) \leq \frac{1}{T^r}$. 
	 Now, using similar arguments as in the proof of Lemma~\ref{lem:conditional_operator_contraction} and using the fact that $A_{s,l,i}$ is $\mathcal{F}_{s-2,i}$ measurable,  we show that:
	$$\mathbb{E}\left[\Gamma^{(i)}_{s,s-1}A_{s,l,i}\left(\Gamma^{(i)}_{s,s-1}\right)^{\intercal} \bigr|\mathcal{F}_{s-2,i}\right] \preceq \Lambda(A_{s,l,i}) + \frac{4\alpha +2\alpha^2}{T^r}\id.$$
	Taking expectation on both sides, we get:
	\begin{align}
	\mathbb{E}\left[\Gamma^{(i)}_{s,s-1}A_{s,l,i}\left(\Gamma^{(i)}_{s,s-1}\right)^{\intercal}\right] \preceq \mathbb{E}\Lambda(A_{s,l,i}) + \frac{4\alpha +2\alpha^2}{T^r}\id=  \Lambda(\mathbb{E}A_{s,l,i}) + \frac{4\alpha +2\alpha^2}{T^r}\id,\label{eq:induction_step}
	\end{align}
	where in the last step we have used the linearity of the operator $\Lambda$. 
	
	We now use induction and results in Lemma~\ref{lem:PSD_maps}, to prove:  $$\mathbb{E}\left[\tr\left(\Gamma^{(i)}_{s,l+1}X_{l,i}X_{l,i}^{\intercal}\left(\Gamma^{(i)}_{s,l+1}\right)^{\intercal}\right)\right] \preceq \Lambda^{s-l-1}(A) + (4\alpha+2\alpha^2)(s-l-1)\frac{\id}{T^r},$$ where $A = \mathbb{E}\left[X_l^{\intercal}X_l\right]$.
	
	The statement is clearly true when $s = l+1$.  If the result is true for $s = s_0-1$, by Equation~\eqref{eq:induction_step} and the definition of $A_{s,l,i}$, we have:
	
	\begin{align*}
	&\mathbb{E}\left[\tr\left(\Gamma^{(i)}_{s_0,l+1}X_{l,i}X_{l,i}^{\intercal}\left(\Gamma^{(i)}_{s_0,l+1}\right)^{\intercal}\right)\right] \preceq  \Lambda\left[\mathbb{E}\Gamma^{(i)}_{s_0-1,l+1}X_{l,i}X_{l,i}^{\intercal}\left(\Gamma^{(i)}_{s_0-1,l+1}\right)^{\intercal}\right] + (4\alpha+2\alpha^2)\frac{\id}{T^r}  \\
	&\qquad\preceq \Lambda\left[ \Lambda^{s_0-l-2}(A) + (s_0-l-2)(4\alpha+2\alpha^2)\frac{\id}{T^r}\right] +  (4\alpha+2\alpha^2)\frac{\id}{T^r}  \\
	&\qquad\preceq \Lambda^{s_0-l-1}(A) + (s_0-l-1)(4\alpha+2\alpha^2)\frac{\id}{T^r},
	\end{align*}
	where in the first step we have used Equation~\eqref{eq:induction_step}. In the second step we have use item 2 of Lemma~\ref{lem:PSD_maps} and the induction hypothesis for $s = s_0 -1$. In the third step we have used linearity of $\Lambda$ and item 3 of Lemma~\ref{lem:PSD_maps} . We use the equation above along with Equation~\eqref{eq:trace_norm_equality} to conclude the result.
\end{proof}

We will use some results proved in~\cite{jain2018parallelizing} - there we take the batch size $b = 1$ and consider the homoscedastic (independent) noise case. Consider Lemmas 13,14 and 15 in~\cite{jain2018parallelizing}. 
We have the following correspondences between terms in our work and ~\cite{jain2018parallelizing}
\begin{enumerate}
	\item The step size $\alpha$ in this work corresponds to $\gamma$.
	\item The operator $\frac{\mathcal{I} - \Lambda}{\alpha} : \mathcal{S}(d) \to \mathcal{S}(d)$ here corresponds to the operator $\mathcal{T}_b$.
	\item The matrix $\sigma^2 A$ here corresponds $\Sigma$.
	\item The matrix $A$ here corresponds to $H$.
\end{enumerate}
Under the step size condition becomes $\alpha < \frac{1}{2}$, 
\begin{align}
\sum_{s\geq l+1} \tr(\Lambda^{s-l-1}(A)) &= \tr\left((\mathcal{I} - \Lambda)^{-1}A\right)= \frac{1}{\alpha}\tr(\mathcal{T}_b^{-1}A)\leq \frac{2}{\alpha}\tr(\id) = \frac{2d}{\alpha}.\label{eq:imported_operator_theory}
\end{align} 
The first step follows from the proof of Lemma 13 in~\cite{jain2018parallelizing}, the second step follows from Lemma 15 item 4 in \cite{jain2018parallelizing}. 

We will combine the inequalities proved above to obtain bounds for the RHS of Equation~\eqref{eq:variance_identity}. When $\alpha < \frac{1}{2}$ and $K > \tmix\lceil r\log_2{T} \rceil$, we have:
\begin{align}
& \mathbb{E}\bigr(\hat{w}^{\var}-w^{*}\bigr)^{\intercal}A\bigr(\hat{w}^{\var}-w^{*}\bigr)^{\intercal} = \frac{4\alpha^2\sigma^2}{T^2} \sum_{i,t,s}\sum_{l=1}^{\min(t-1,s-1)}\mathbb{E} X_{l,i}^{\intercal}\left(\Gamma^{(i)}_{t,l+1}\right)^{\intercal} A \Gamma^{(i)}_{s,l+1}X_{l,i}  \nonumber \\
&\leq \frac{8\alpha^2\sigma^2}{T^2}\sum_{i,s}\sum_{t = s}^{T/K-1}\sum_{l=1}^{s-1} \mathbb{E} X_{l,i}^{\intercal}\left(\Gamma^{(i)}_{t,l+1}\right)^{\intercal} A \Gamma^{(i)}_{s,l+1}X_{l,i}  \nonumber \\
&\leq \frac{8\alpha^2\sigma^2}{T^2}\sum_{i,s}\sum_{l=1}^{s-1}\left[\frac{\alpha}{4K^2T^{r-2}} + \frac{1}{\alpha}\mathbb{E}\|\Gamma^{(i)}_{s,l+1}X_{l,i}\|^2\right] \nonumber \\
&\leq \frac{2\alpha^3 \sigma^2}{K^3T^{r-2}} + \frac{8\alpha\sigma^2}{T^2}\sum_{i,s}\sum_{l=1}^{s-1}\left[ \mathbb{E}\|\Gamma^{(i)}_{s,l+1}X_{l,i}\|^2\right] \nonumber \\
&= \frac{2\alpha^3 \sigma^2}{K^3T^{r-2}} + \frac{8\alpha\sigma^2}{T^2}\sum_{i} \sum_{l=1}^{T/K-1}\sum_{s=l+1}^{T/K}\left[ \mathbb{E}\|\Gamma^{(i)}_{s,l+1}a_{l,i}\|^2\right] \nonumber \\
&\leq  \frac{2\alpha^3 \sigma^2}{K^3T^{r-2}} + \frac{d(32\alpha^2+16\alpha^3)\sigma^2}{K^2T^{r-1}}+ \frac{8\alpha\sigma^2}{T^2}\sum_{i} \sum_{l=1}^{T/K-1} \sum_{s= l+1}^{\infty} \tr(\Lambda^{s-l-1}(A))  
\nonumber \\
&\leq \frac{2\alpha^3 \sigma^2}{K^3T^{r-2}} + \frac{d(32\alpha^2+16\alpha^3)\sigma^2}{K^2T^{r-1}}+ \frac{16d\sigma^2}{T}. \label{eq:final_variance_bound}
\end{align}
In the third step we have used Lemma~\ref{lem:summation_1}. In the sixth step we have used Lemma~\ref{lem:PSD_map_properties}. In the seventh step we have used Equation~\eqref{eq:imported_operator_theory}.

The above equation bounds the variance term. Theorem~\ref{thm:par-sgd} now follows by combining the bias and variance bounds given above. 

\section{Experience Replay Accelerates Bias Decay for Gaussian Autoregressive (AR) Dynamics}
\label{app:gaussian_ar}

\subsection{Problem Setting}
\label{sec:upper_bound_formulation}
Suppose our sample vectors $X \in \R^d$ are generated from a Markov chain with the following dynamics: 
$$X_1 = G_1$$
$$X_2 = \sqrt{1-\eps^2}X_1 + \eps G_2$$
$$\cdots$$
$$X_{t+1} = \sqrt{1-\eps^2}X_t + \eps G_{t + 1}$$

where $\eps$ is fixed and known, and each $G_j$ is independently sampled from $ \frac{1}{\sqrt{d}} \N(0,I)$. 

Each observation $Y_i = X^T_i w^* + \xi_i$, where $\xi_i$ are independently drawn random variables with mean 0 and variance $\sigma^2$. 

We use SGD to find $w$ that minimizes the loss 
\[
L(w) = \E [(X^Tw - Y)^2] - \E [ (X^T w^* - Y)^2]
\]
for some $X \sim \N(0,\frac{1}{\sqrt{d}}I_d)$. 

We first establish standard properties of the vectors $X_i$. 

\begin{lemma}
	\label{lem:concentration_norm}
	With probability $1-\beta$, $1-\frac{c}{\sqrt{d}}\log(\frac{1}{\beta}) \leq \|X_j\|_2^2 \leq 1 + \frac{c}{\sqrt{d}}\log(\frac{1}{\beta})$, for some constant $c$. 
\end{lemma}

\begin{proof}
	Note that for each $X_j$, each component $k \in [1,d]$, is independently normally distributed with mean 0 and variance $\frac{1}{d}$. Then writing $\|X_j\|^2 = \sum\limits_{k=1}^d X^2_{jk}$ and using Bernstein's inequality for sub-exponential random variables , we will get the desired result. 
\end{proof}

\begin{lemma}
	\label{lem:concentration_inner_product}
	With probability $1-\beta$, $-\frac{c}{\sqrt{d}}\log(\frac{1}{\beta}) \leq X^T_i G_j \leq \frac{c}{\sqrt{d}}\log(\frac{1}{\beta})$, for some constant $c$, where $X$, and $G_j$ are defined as before, ie $G_j \sim \frac{1}{\sqrt{d}} \N(0,I)$ and $j > i$ (so that $G_j$ is independent of $X_i$). 
\end{lemma}

\begin{proof}
	Note that $X_i \sim \N(0, \sigma^2I)$, where $\sigma^2 = \frac{1}{d}$. Then, for any fixed $X_1$, it follows that $X_i^TG_j = \sum\limits_{\ell=1}^d X_{i\ell}G_{j\ell}$, where each random variable in the summation is independent. Since $G_j \sim \N(0, \frac{1}{d} I)$, the result follows by Bernstein's inequality for sub-exponential random variables. 
\end{proof}

\begin{lemma}
\label{lem:brownian_mixing}
The mixing time of the Gaussian AR chain is $\Theta \left(\frac{1}{\epsilon^2}\log({d})\right)$. 
\end{lemma}

\begin{proof}
The stationary distribution of the Gaussian AR chain is $\pi = \N(0, \frac{1}{d}I)$.
Given $X_0$, we compare $KL(P^t(X_0), \pi)$. The standard formula for the KL divergence of two multivariate Gaussians $\N(\mu_1, \Sigma_1)$ and $\N(\mu_2, \Sigma_2)$ is: 
\[
KL(G_1, G_2) = \frac{1}{2} \left[\log \frac{|\Sigma_2|}{|\Sigma_1|} - d + Tr\left(\Sigma^{-1}_2\Sigma_1\right) + (\mu_2 - \mu_1)^T\Sigma^{-1}_2(\mu_2 - \mu_1)\right]
\]

Note that $X_t = (1-\eps^2)^{\frac{t}{2}}X_0 + \eps \sum\limits_{j=0}^{t-1} (1-\eps^2)^{\frac{j}{2}}G_j$. Therefore, $P^t(X_0) \sim \N((1-\eps^2)^{\frac{t}{2}}X_0, \frac{1-(1-\eps^2)^t}{d} I)$
For the KL divergence to be $\leq \delta$ where $\delta$ is a fixed constant, we need: 
\[
\frac{1}{2} \left[\log C - d + d^2\frac{1-(1-\eps^2)^t}{d} + d\cdot (1-\eps^2)^t \left(1 + \sqrt{\frac{\log(\frac{1}{\beta})}{d}}\right)\right] \leq \delta
\]
where $C$ is an appropriate constant. Eventually, we will get that we need $(1-\eps^2)^t \leq \frac{c}{\sqrt{d}}$, for some constant $c$. A direct application of Pinsker's inequality shows that $\tmix = \Theta \left(\frac{1}{\epsilon^2}\log({d})\right)$. 
\end{proof}

Suppose we have a continuous stream of samples $X_1, X_2, \ldots X_T$ from the Markov Chain. We split the $T$ samples into $\frac{T}{S}$ separate buffers of size $S$ in a sequential manner, ie $X_1, \ldots X_S$ belong to the first buffer. Let $S = B + u$, where $B$ is orders of magnitude larger than $u$. From within each buffer, we drop the first $u$ samples. Then starting from the first buffer, we perform $B$ steps of SGD, where for each iteration, we sample uniformly at random from within the $[u, B+u]$samples in the first buffer. Then perform the next $B$ steps of SGD by uniformly drawing samples from within the $[u, B+u]$ samples in the second buffer. We will choose $u$ so that the buffers are are approximately i.i.d.. 

We run SGD this way for the first $\frac{T}{2S}$ buffers to ensure that the bias of each iterate is small. Then for the last $\frac{T}{2S}$ buffers, we perform SGD in the same way, but we tail average over the last iterate produced using each buffer to give our final estimate $w$. We formally write Algorithm \ref{alg_main}. 


\begin{theorem}[SGD with Experience Replay for Gaussian AR Chain]
	
For any $\eps \leq 0.21$, if $B \geq \frac{1}{\eps^7}$ and $d = \Omega(B^4 \log (\frac{1}{\beta}))$, with probability at least $1-\beta$, Algorithm~\ref{alg_main} returns $w$ such that $\E[\loss(w)] \leq O\left(\exp\left(\frac{-T\log({d})}{\kappa\sqrt{\tmix}}\right) \norm{w_0 - w^*}^2\right) + \tilde{O}\left(\frac{\sigma^2d\sqrt{\tmix}}{ T}\right) + \loss(w^*)$. Recall that $\kappa=d$
\end{theorem}

\begin{proof}
$\E[\loss(w)] = \E(\loss(w^{\bias})) + \E(\loss(w^{\var})) + \loss(w^*)$. We analyze $\E(\loss(w^{\bias}))$ in Theorem \ref{thm:exp-replay_bias} and analyze $\E(\loss(w^{\bias}))$ in Theorem \ref{thm:exp-replay_var}. 
\end{proof}

\subsection{Bias Decay with Experience Replay}
Standard analysis of SGD says that for bias decay, $w^{\bias}_{t+1} - w^* = (I - \eta \hat{X}_t\hat{X}_t^T) (w^{\bias}_t - w^*)$, where $(\hat{X}_t, \hat{y}_t)$ is the sample used in the $t$-th iteration of SGD. 
\label{sec:er_sgd_upper_bound}

\begin{theorem}[Bias Decay for SGD with Experience Replay for Gaussian AR Chain]
	\label{thm:exp-replay_bias}For any $\eps \leq 0.21$, if $B \geq \frac{1}{\eps^7}$ and $d = \Omega(B^4 \log (\frac{1}{\beta}))$, with probability at least $1-\beta$, Algorithm~\ref{alg_main} produces $w$ such that $\E[\loss(w^{\bias})] \leq O\left(\exp\left(\frac{-T\log({d})}{\kappa\sqrt{\tmix}}\right) \norm{w_0 - w^*}^2\right) $. 
\end{theorem}

\begin{proof}[Proof of Theorem~\ref{thm:exp-replay_bias}]
	We first write $Tr\left(\E\left[(w_{\frac{BT}{S}}- w^*) (w_{\frac{BT}{S}} - w^*)^T \right]\right)$ as:
	\begin{center}
		\scalebox{0.85}{
			$Tr\left(\E\left[\left(I - \eta \hat{X}_{\frac{BT}{S}}\hat{X}_{\frac{BT}{S}}^T\right)\ldots\left(I - \eta \hat{X}_1\hat{X}_1^T\right) (w_0 - w^*) (w_0 - w^*)^T\left(I - \eta \hat{X}_1\hat{X}_1^T\right) \ldots \left(I - \eta \hat{X}_{\frac{BT}{S}}\hat{X}_{\frac{BT}{S}}^T\right)\right]\right)$
		}
	\end{center}
	
	Since we assume $(w_0 - w^*)^T$ to be independent standard Gaussian, we are mostly interested in:
	\begin{center}
		\scalebox{0.85}{
			$Tr\left(\E\left[\left(I - \eta \hat{X}_1\hat{X}_1^T\right) \ldots \left(I - \eta \hat{X}_{\frac{BT}{S}}\hat{X}_{\frac{BT}{S}}^T\right)\left(I - \eta \hat{X}_{\frac{BT}{S}}\hat{X}_{\frac{BT}{S}}^T\right) \ldots \left(I - \eta \hat{X}_{1}\hat{X}_{1}^T\right)\right]\right)$
		}
	\end{center}
	
	This can be written as: 
	\begin{center}
		\scalebox{0.85}{
			$Tr\left(\E\left[\E\left[\left(I - \eta \hat{X}_1\hat{X}_1^T\right) \ldots \left(I - \eta \hat{X}_{\frac{BT}{S}}\hat{X}_{\frac{BT}{S}}^T\right)\left(I - \eta \hat{X}_{\frac{BT}{S}}\hat{X}_{\frac{BT}{S}}^T\right) \ldots \left(I - \eta \hat{X}_{1}\hat{X}_{1}^T\right)|  X_1, \ldots X_{\frac{(B-1)T}{S}}\right]\right]\right)$
		}
	\end{center}
	
	The quantity of interest is therefore, 
	\begin{center}
		\scalebox{0.75}{
			$\E\left[\left(I - \eta \hat{X}_{\frac{(B-1)T}{S} + 1}\hat{X}_{\frac{(B-1)T}{S}+1}^T\right) \ldots \left(I - \eta \hat{X}_{\frac{BT}{S}}\hat{X}_{\frac{BT}{S}}^T\right)\left(I - \eta \hat{X}_{\frac{BT}{S}}\hat{X}_{\frac{BT}{S}}^T\right) \ldots \left(I - \eta \hat{X}_{\frac{(B-1)T}{S}}\hat{X}_{\frac{(B-1)T}{S}}^T\right)|  X_1, \ldots X_{\frac{(B-1)T}{S}}\right]~,$
		}
	\end{center}
	
	which Lemma \ref{lem:approx-iid} says is $\preceq \left(1 - \frac{1}{8} \cdot \frac{\eps B}{40d\pi}\right) I$. Therefore, if $N$ is the number of buffers, it follows that the loss decays at a rate of $\exp\left(\frac{-NB\epsilon}{\kappa}\right) \norm{w_0 - w^*}^2$, and since $T = NB$, we conclude that the rate is $\exp\left(\frac{-T \log({d})}{\kappa\sqrt{\tmix}}\right) \norm{w_0 - w^*}^2$. 
\end{proof}

We first solve the issue that the buffers are approximately iid by establishing a relationship between the contraction rate of the sampled vectors with the contraction rate of vectors sampled from buffers that are iid. The proof compares two parallel processes, one where the samples of the buffer follow Gaussian AR  dynamics from an initial $X_0$, where this initial $X_0$ is the $Sj$-th sample from the Markov Chain for some buffer index $j$, and another process which follows Gaussian AR dynamics from an initial $\tilde{X}_0$ generated independently from $\frac{1}{\sqrt{d}}\mathcal{N}(0,I)$. We show that the expected contraction of the first processes can be bounded by the expected contraction of the second process plus a constant. 

\begin{lemma}
\label{lem:approxim-iid-helper}
Suppose that $\hat{X}_1, \ldots \hat{X}_j$ are vectors sampled from arbitrary buffers (ie, they are of the form $X_{u+s} = \left(1-\epsilon^2\right)^{(u+s)/2} X_0 + \epsilon \sum_{t = 1}^{u+s} \left(1-\epsilon^2\right)^{(t-(u+s))/2} G_t$, for $X_0$ which is the first vector in the buffer). Sample a new random $\tilde{X}_0$ independently from $\frac{1}{\sqrt{d}}\mathcal{N}(0,I)$ and denote $\tilde{X}_{u+s} \defeq \left(1-\epsilon^2\right)^{(u+s)/2} \tilde{X}_0 + \epsilon \sum_{t = 1}^{u+s} \left(1-\epsilon^2\right)^{(t-(u+s))/2} G_t $. Then we have: 
\[
\E\left[\left(I - \eta \hat{X}_{1}\hat{X}_{1}^T\right) \ldots\left(I - \eta \hat{X}_{j}\hat{X}_{j}^T\right)\right] \preceq \E\left[\left(I - \eta \hat{\tilde{X}}_{1}\hat{\tilde{X}}_{1}^T\right) \ldots\left(I - \eta \hat{\tilde{X}}_{j}\hat{\tilde{X}}_{j}^T\right) \right] + c\cdot j^2 (1-\eps^2)^{\frac{u}{2}} ~,
\]

where $c$ is a constant, and $\hat{\tilde{X}}_{1}$ denotes the same sample index as that of $\hat{X}_1$, except that the initial vector was sampled independently from $\frac{1}{\sqrt{d}}\mathcal{N}(0,I)$.
\end{lemma}

\begin{proof}
We see that in general $X_{u+s} - \tilde{X}_{u+s} = \left(1-\epsilon^2\right)^{(u+s)/2} \left(X_0 - \tilde{X}_0\right)$. With probability more than $1-\beta$, we have $\norm{X_0 - \tilde{X}_0} \leq \sqrt{2} + \sqrt{\frac{\log \frac{1}{\delta}}{d}}$ and so $\norm{X_{u+s} - \tilde{X}_{u+s}} \leq \left(1-\epsilon^2\right)^{(u+s)/2} \cdot \left(\sqrt{2} + \sqrt{\frac{\log \frac{1}{\delta}}{d}}\right) \leq \left(1-\epsilon^2\right)^{u/2} \cdot \left(\sqrt{2} + \sqrt{\frac{\log \frac{1}{\delta}}{d}}\right)$

Therefore, 
\begin{align*}
\E &\left[\left(I - \eta \hat{X}_{1}\hat{X}_{1}^T\right) \ldots\left(I - \eta \hat{X}_{j}\hat{X}_{j}^T\right)\right]\\
& = \E\left[\left(I - \eta \hat{X}_{1}\hat{X}_{1}^T + \eta \hat{\tilde{X}}_{1}\hat{\tilde{X}}_{1}^T - \eta \hat{\tilde{X}}_{1}\hat{\tilde{X}}_{1}^T \right) \ldots\left(I - \eta \hat{X}_{j}\hat{X}_{j}^T + \eta \hat{\tilde{X}}_{j}\hat{\tilde{X}}_{j}^T - \eta \hat{\tilde{X}}_{j}\hat{\tilde{X}}_{j}^T  \right)\right] \\
& \preceq \E\left[\left(I - \eta \hat{\tilde{X}}_{1}\hat{\tilde{X}}_{1}^T\right) \ldots\left(I - \eta \hat{\tilde{X}}_{j}\hat{\tilde{X}}_{j}^T\right)\right] +  \sum_{s=1}^{j} {j \choose s} \left(4\left(1-\epsilon^2\right)^{u/2} \cdot \left(\sqrt{2} + \sqrt{\frac{\log \frac{1}{\delta}}{d}}\right)\right)^s \cdot I \\
& \preceq \E\left[\left(I - \eta \hat{\tilde{X}}_{1}\hat{\tilde{X}}_{1}^T\right) \ldots\left(I - \eta \hat{\tilde{X}}_{j}\hat{\tilde{X}}_{j}^T\right)\right] + j \cdot \max_{s=1}^{j} (40j \cdot \left(1-\epsilon^2\right)^{u/2})^s  \cdot I 
\end{align*}
Therefore, for sufficiently large $u$, the lemma is proved. 
\end{proof}

Lemma \ref{lem:approx-iid} establishes the per buffer contraction rate, using \ref{lem:approxim-iid-helper}. The rest of the proofs in this section is devoted to establishing the contraction of the process where the vectors are sampled from iid buffers. 
\begin{lemma}\label{lem:approx-iid}
	Let $X_0$ be the first vector in the buffer. If $u > \frac{2}{\epsilon^2} \log \frac{300000 \pi d B}{\epsilon}$, and $\norm{X_0} \leq 1 + \sqrt{\frac{\log \frac{1}{\delta}}{d}}$, then
	\begin{align*}
	& \E\left[\left(I - \eta \hat{X}_{u+B}\hat{X}_{u+B}^T\right) \ldots\left(I - \eta \hat{X}_{u+1}\hat{X}_{u+1}^T\right) \left(I - \eta \hat{X}_{u+1}\hat{X}_{u+1}^T\right) \ldots \left(I - \eta \hat{X}_{u+B}\hat{X}_{u+B}^T\right) | X_0 \right] \\ 
	& \qquad \preceq \left(1 - \frac{1}{8} \cdot \frac{\eps B}{40d\pi}\right) I.
	\end{align*}
\end{lemma}
\begin{proof}
	\begin{align*}
	& \E\left[\left(I - \eta \hat{X}_{u+B}\hat{X}_{u+B}^T\right) \ldots\left(I - \eta \hat{X}_{u+1}\hat{X}_{u+1}^T\right) \left(I - \eta \hat{X}_{u+1}\hat{X}_{u+1}^T\right) \ldots \left(I - \eta \hat{X}_{u+B}\hat{X}_{u+B}^T\right) | X_0 \right] \\
	&\preceq \E\left[\left(I - \eta \hat{\tilde{X}}_{u+B}\hat{\tilde{X}}_{u+B}^T\right) \ldots\left(I - \eta \hat{\tilde{X}}_{u+1}\hat{\tilde{X}}_{u+1}^T\right) \left(I - \eta \hat{\tilde{X}}_{u+1}\hat{\tilde{X}}_{u+1}^T\right) \ldots \left(I - \eta \hat{\tilde{X}}_{u+B}\hat{\tilde{X}}_{u+B}^T\right) | X_0 \right] \\
	& + \sum_{s=1}^{2B} {2B \choose s} \left(4\left(1-\epsilon^2\right)^{u/2} \cdot \left(\sqrt{2} + \sqrt{\frac{\log \frac{1}{\delta}}{d}}\right)\right)^s \cdot I \\
	&= \E\left[\left(I - \eta \hat{\tilde{X}}_{u+B}\hat{\tilde{X}}_{u+B}^T\right) \ldots\left(I - \eta \hat{\tilde{X}}_{u+1}\hat{\tilde{X}}_{u+1}^T\right) \left(I - \eta \hat{\tilde{X}}_{u+1}\hat{\tilde{X}}_{u+1}^T\right) \ldots \left(I - \eta \hat{\tilde{X}}_{u+B}\hat{\tilde{X}}_{u+B}^T\right) \right] \\
	&+ \left(\sum_{s=1}^{2B} {2B \choose s} \left(4\left(1-\epsilon^2\right)^{u/2} \cdot \left(\sqrt{2} + \sqrt{\frac{\log \frac{1}{\delta}}{d}}\right)\right)^s\right) \cdot I \left(\mbox{ since } \hat{\tilde{X}}_{u+s} \mbox{ are all independent of } X_0\right)\\
	&\preceq \left(1 - \frac{1}{4} \cdot \frac{\eps B}{40d\pi}\right) I + 2B \cdot \max_{s=1}^{2B} (40B \cdot \left(1-\epsilon^2\right)^{u/2})^s  \cdot I \left(\mbox{ from Lemma~\ref{lem:iid_contraction} } \right)\\
	&\preceq \left(1 - \frac{1}{4} \cdot \frac{\eps B}{40d\pi}\right) I + 80B^2 \cdot \left(1-\epsilon^2\right)^{u/2}  \cdot I \\
	&\preceq \left(1 - \frac{1}{4} \cdot \frac{\eps B}{40d\pi}\right) I + \frac{1}{8} \cdot \frac{\eps B}{40d\pi} \cdot I \left(\mbox{ from hypothesis on } u\right)\\
	&= \left(1 - \frac{1}{8} \cdot \frac{\eps B}{40d\pi}\right) I.
	\end{align*}
	This finishes the proof.
\end{proof}

We now define $H$ as $\frac{1}{B} \sum\limits^S_{j=u} \tilde{X}_j \tilde{X}_j^T$, where $\tilde{X}_u, \tilde{X}_{u+1}, \ldots \tilde{X}_S$ are the vectors that we are sampling from in the parallel process (where $\tilde{X}_0$ is sampled i.i.d. from $\N(0, \frac{1}{d})$). For the sake of convenience, for the rest of this section we also say $\hat{X}_i$ is the sampled vector at the $i$-th iteration, where the samples are taken from the set of $\tilde{X}$s coming from the parallel process. 
\begin{lemma}\label{lem:iid_contraction}
	Suppose $\frac{\eps B}{40\pi}$ of the eigenvalues of $H$ are larger than or equal to $\frac{1}{B}$. Then 
	$$\E\left[(w_B - w^*) (w_B - w^*)^T \right] \preceq \left[\frac{3}{4} \cdot \frac{\eps B}{40d\pi} + \left(1 - \frac{\eps B}{40d\pi}\right) \right] I $$
\end{lemma}

\begin{proof}
	
	After iterating through 1 buffer, we have: 
	$$\E\left[(w_B - w^*) (w_B - w^*)^T \right]$$ $$\E\left[\E\left[\left(I - \eta \hat{X}_B\hat{X}_B^T\right)\ldots\left(I - \eta \hat{X}_1\hat{X}_1^T\right) \left(I - \eta \hat{X}_1\hat{X}_1^T\right) \ldots \left(I - \eta \hat{X}_B\hat{X}_B^T\right) | \tilde{X}_1, \ldots \tilde{X}_B\right]\right]$$
	
	Suppose that by Lemma \ref{lem:concentration_norm}, we have that each $\|X_j\|^2_2 \leq 2$ with high probability. Then when $\eta = \frac{1}{2}$, $\E\left[\left(I - \eta \hat{X}_1\hat{X}_1^T\right)\left(I - \eta \hat{X}_1\hat{X}_1^T\right)\right] \preceq I - \eta H$ in the PSD sense. 
	
	So now we can write: 
	
	\begin{align*}
	\E&\left[\left(I - \eta \hat{X}_B\hat{X}_B^T\right) \ldots\left(I - \eta \hat{X}_1\hat{X}_1^T\right) \left(I - \eta \hat{X}_1\hat{X}_1^T\right) \ldots \left(I - \eta \hat{X}_B\hat{X}_B^T\right) | \tilde{X}_1, \ldots \tilde{X}_B\right] \\
	&\preceq \E\left[\left(I - \eta \hat{X}_B\hat{X}_B^T\right)\ldots\left(I - \eta \hat{X}_2\hat{X}_2^T\right) \left(I - \eta H \right) \left(I - \eta \hat{X}_2\hat{X}_2^T\right) \ldots \left(I - \eta \hat{X}_B\hat{X}_B^T\right) | \tilde{X}_1, \ldots \tilde{X}_B\right]\\
	& = \E\left[\left[\prod^B_{i=2}\left(I - \eta \hat{X}_i\hat{X}_i^T\right)\right] \left[\prod^B_{i=2}\left(I - \eta \hat{X}_i\hat{X}_i^T\right)\right]^T| \tilde{X}_1, \ldots \tilde{X}_B\right] \\
	& - \eta \E\left[\left[\prod^B_{i=2}\left(I - \eta \hat{X}_i\hat{X}_i^T\right)\right] H\left[\prod^B_{i=2}\left(I - \eta \hat{X}_i\hat{X}_i^T\right)\right]^T| \tilde{X}_1, \ldots \tilde{X}_B\right]\\
	& \preceq \E\left[\left[\prod^B_{i=2}\left(I - \eta \hat{X}_i\hat{X}_i^T\right)\right] \left[\prod^B_{i=2}\left(I - \eta \hat{X}_i\hat{X}_i^T\right)\right]^T| \tilde{X}_1, \ldots \tilde{X}_B\right] \\
	& - \eta \E\left[\left[\prod^B_{i=3}\left(I - \eta \hat{X}_i\hat{X}_i^T\right)\right] \left(I-\eta H\right) H \left(I - \eta H\right)\left[\prod^B_{i=3}\left(I - \eta \hat{X}_i\hat{X}_i^T\right)\right]^T| \tilde{X}_1, \ldots \tilde{X}_B\right]\\
	\end{align*}
	
	where the last inequality comes from the fact that \\
	$\E\left[\left(I - \eta\hat{X}_i\hat{X}_i^T \right) S\left(I - \eta\hat{X}_i\hat{X}_i^T \right) | \tilde{X}_1, \ldots \tilde{X}_B\right] \succeq \left(I - \eta H \right) S\left(I - \eta H \right)$ for $S \succeq 0$. 
	
	Recursing on this inequality gives us that 
	\begin{align*}
	\E&\left[\left(I - \eta \hat{X}_B\hat{X}_B^T\right) \ldots\left(I - \eta \hat{X}_1\hat{X}_1^T\right) \left(I - \eta \hat{X}_1\hat{X}_1^T\right) \ldots \left(I - \eta \hat{X}_B\hat{X}_B^T\right) | \tilde{X}_1, \ldots \tilde{X}_B\right] \\
	&\preceq \E\left[I - \eta \sum\limits_{k=0}^{n-1} \left( I -\eta H \right)^k H \left( I -\eta H \right)^k | \tilde{X}_1, \ldots \tilde{X}_B\right]
	\end{align*}

	Suppose that $\lambda$ is an eigenvalue of $H$. Then using the formulas for geometric series, it follows that $\frac{1-\eta \lambda}{2-\eta \lambda} + \frac{(1-\eta \sigma)^{2B}}{2-\eta \lambda}$ is an eigenvalue of $I - \eta \sum\limits_{k=0}^{n-1} \left( I -\eta H \right)^k H \left( I -\eta H \right)^k$. 
	
	Suppose that $\frac{\eps B}{40\pi}$ of the eigenvalues of $H$ are larger than or equal to $\frac{1}{B}$. For those eigenvalues $\frac{1-\eta \lambda}{2-\eta \lambda} + \frac{(1-\eta \sigma)^{2B}}{2-\eta \lambda} \leq \frac{1}{2} + \left(1 - \eta \lambda \right)^{2B} \leq \frac{1}{2} + (1-\frac{1}{2B})^{2B} \leq \frac{3}{4}$, where we use $\eta = \frac{1}{2}$ without loss of generality, and the fact that $\left(1 - \frac{1}{2B} \right)^{2B} \leq \frac{1}{4}$. 
	
	Therefore, 
	\begin{align*}
	\E&\left[\E\left[\left(I - \eta \hat{X}_B\hat{X}_B^T\right) \ldots\left(I - \eta \hat{X}_1\hat{X}_1^T\right) \left(I - \eta \hat{X}_1\hat{X}_1^T\right) \ldots \left(I - \eta \hat{X}_B\hat{X}_B^T\right) |\tilde{X}_1, \ldots \tilde{X}_B\right]\right] \\
	&\preceq \left[\frac{3}{4} \cdot \frac{\eps B}{40d\pi} + \left(1 - \frac{\eps B}{40d\pi}\right) \right] I
	\end{align*}
	
\end{proof}

\begin{lemma}
\label{lem:H_eigs}
	$\frac{\eps B}{40\pi}$ of the eigenvalues of $H$ are larger than or equal to $\frac{1}{B}$ when $d \geq B^4 C \log(\frac{1}{\beta})$ for some constant $C$, $\eps < 0.21$ and $B = \frac{1}{\eps^7}$. 
\end{lemma}

\begin{proof}
	$H = \frac{1}{B} \sum\limits_{j=1}^B X_j X_j^T = \frac{1}{B} XX^T$, where the $j$-th column of $X$ is $X_j$. The non-zero eigenvalues of $H$ are equivalent to the non-zero eigenvalues of the gram matrix $M = \frac{1}{B} X^TX$. We can characterize each entry of $M$. For $j \geq i$, 
	\begin{align*}
	X_j^TX_i &= \left((1-\eps^2)^{\frac{j-i}{2}} X_i + \eps \sum\limits_{k = i+1}^j (1-\eps^2)^{\frac{j-k}{2}}G_k \right)^T X_i \\
	& =  (1-\eps^2)^{\frac{j-i}{2}} \|X_i\|^2 + \eps \sum\limits_{k = i+1}^j (1-\eps^2)^{\frac{j-k}{2}}G_k^TX_i
	\end{align*}
	
	Define Toeplitz matrix $Z$ with the following Toeplitz structure, $Z_{ij} = \frac{1}{B}(1-\eps^2)^{\frac{|i-j|}{2}}$ for $1 \leq i,j\leq B$. Then we can write $M = Z + E$. Lemma \ref{lem:T_eigs} establishes that $\frac{\eps B}{40\pi}$ of the eigenvalues of $Z$ are larger than or equal to $\frac{2}{B}$. By Weyl's inequality, the corresponding eigenvalues in $M$ can be perturbed by at most $\|E\|_F$, which we bound below to be within $\frac{1}{B}$. Therefore $\frac{\eps B}{40\pi}$ of the eigenvalues of $H$ are larger than or equal to $\frac{1}{B}$. 
	
	We conclude the proof with the analysis of the Frobenius norm of $E$. 
	
	Note that $\eps \sum\limits_{k = i+1}^j (1-\eps^2)^{\frac{j-k}{2}}G_k \sim \N(0,\sigma^2)$, where $\sigma^2 \leq \frac{1}{d}$. 
	
	By Lemmas \ref{lem:concentration_norm} and \ref{lem:concentration_inner_product}, we have that for $j \geq i$, 
	\begin{align*}
	\left|E[i][j]\right| &\leq \frac{c}{\sqrt{d}} \log\left(\frac{1}{\beta}\right) \left((1-\eps^2)^{\frac{j-i}{2}} +  1\right) \\
	& = \frac{2c}{\sqrt{d}} \log\left(\frac{1}{\beta}\right)
	\end{align*}
	So then the Frobenius norm of $E$, $\|E\|_F \leq \sqrt{B^2 \cdot \frac{c}{d} \log\left(\frac{1}{\beta}\right)} \leq B \cdot \frac{c}{\sqrt{d}} \log\left(\frac{1}{\beta}\right)$ for some constant $c$. Therefore, $d \geq B^4 C$ suffices for $\|E\|_F \leq \frac{1}{B}$, where $C$ is some constant. 
\end{proof}

\begin{lemma}
	\label{lem:T_eigs}
	$\frac{\eps B}{40\pi}$ of the eigenvalues of $Z$ are larger than or equal to $\frac{2}{B}$ when $\eps < 0.21$ and $B = \frac{1}{\eps^7}$. 
\end{lemma}

\begin{proof}
	To study the eigenvalues of $Z$, we first study the eigenvalues of the circulant matrix $C$, where the first row of $C$ has the following entries: 
	
	If $B$ is even: $C[1][j] = Z[1][j]$ if $1\leq j \leq \frac{B}{2}$, $C[1][\frac{B}{2} + 1] = 0$, and $C[1][\frac{B}{2} + j] = Z[1][\frac{B}{2} - j + 2]$ for $2 \leq j \leq \frac{B}{2}$
	
	If $B$ is odd: $C[1][j] = Z[1][j]$ if $1\leq j \leq \frac{B+1}{2}$, $C[1][\frac{B+1}{2} + j] = Z[1][\frac{B+1}{2} - j + 1]$ for $1 \leq j \leq \frac{B-1}{2}$. 
	
	The circulent matrices have the following eigenstructure. For simplicity, let $c_j = C[1][j]$. Then $\lambda_j = c_1 + c_2 w_j + c_3 w_j^2 + \ldots c_B w_j^{B-1}$, where $w_j$ is the $j$-th root of unity. 
	
	We first claim that the eigenvalues of the circulent matrix closely approximate the eigenvalues of the Topelitz matrix for sufficiently high $B$. 
	
	We first write $Z = C + P$, where $P$ is a perturbation matrix. Let $\lambda_1(C) \geq \cdots \geq \lambda_B(C)$ be the eigenvalues of $C$ in descending order. We establish in Lemma \ref{lem:circulent_eigs} that $\lambda_{\frac{\eps B}{20\pi}}(C) \geq \frac{9}{B}$. Moreover, in Lemma \ref{lem:perturbation_eigs} we analyze $P$ and establish that $\lambda_{B + 1- \frac{\eps B}{40\pi}}(P) \geq \frac{-7}{B}$. Therefore, by the generalized Weyl's theorem, then it follows that $\lambda_1(M) \ldots \lambda_{\frac{\eps B}{40\pi}}(M) \geq \frac{2}{B} $. 
	
\end{proof}

\begin{lemma}
	\label{lem:circulent_eigs}
	For $\eps < 0.21$, at least $\frac{\eps B}{20\pi}$ of the eigenvalues of $C$ are greater than or equal to $\frac{9}{B}$. 
\end{lemma}

\begin{proof}
	We first characterize all the eigenvalues $\lambda_j$ of $C$, and then we show that for odd $j$, $j \leq \frac{\eps B}{10\pi}$, $\lambda_j \geq \frac{9}{B}$ when $\eps < 0.21$. 
	We now characterize $\lambda_j$. 
	Using the formula for the eigenvalues of a circulent matrix, it follows that $\lambda_j = c_1 + c_2 w_j + c_3 w_j^2 + \ldots c_B w_j^{B-1}$, where $w_j$ is the $j$-th root of unity, ie $w_j = \cos\left( \frac{2\pi j }{B}\right) + i \sin \left( \frac{2 \pi j}{B}\right)$. For simplicity and without loss of generality, suppose $B$ is odd, so that $\frac{B-1}{2}$ is an integer. Using the symmetry of our circulent matrix as well as the symmetry of powers of the roots of unity, we have that: 
	\[
	\lambda_j = \frac{1}{B}\left(1+ \sum\limits_{k=1}^{\frac{B-1}{2}} (1-\eps^2)^{\frac{k}{2}}w_j^k + \sum\limits_{\ell = \frac{B-1}{2} + 1}^{B-1} (1-\eps^2)^{\frac{B-\ell}{2}}w_j^\ell\right)
	\]
	Note that the $(1-\eps^2)$ coefficients in the two different summations are equal when $\ell = B-k$. 
	Then we can rewrite as: 
	\[
	\lambda_j = \frac{1}{B} \left(1+ \sum\limits_{k=1}^{\frac{B-1}{2}} (1-\eps^2)^{\frac{k}{2}}(w_j^k + w_j^{B-k}) \right)
	\]
	We can further write
	{\small
	\begin{align*}
	w_j^k + w_j^{B-k} & = \cos\left( \frac{2\pi j \cdot k}{B}\right) + i \sin \left( \frac{2 \pi j \cdot k}{B}\right) + \cos\left( \frac{2\pi j (B-k)}{B}\right) + i \sin \left( \frac{2 \pi j(B-k)}{B}\right) \\
	& = \cos\left( \frac{2\pi j \cdot k}{B}\right) + i \sin \left( \frac{2 \pi j \cdot k}{B}\right) + \cos\left( \frac{2\pi j (B)}{B} - \frac{2\pi j k}{B}\right) + i \sin \left( \frac{2 \pi j(B)}{B}- \frac{2\pi j k}{B}\right) \\
	& = 2\cos\left( \frac{2\pi j \cdot k}{B}\right)
	\end{align*}
	}
	
	Therefore, we have: 
	{\tiny
	\begin{align*}
	\lambda_j &= \frac{2}{B}\left(\sum\limits^{\frac{B-1}{2}}_{k=0} (1-\eps^2)^{\frac{k}{2}} \cos\left(\frac{2\pi k \cdot j}{B}\right)\right) - \frac{1}{B} \\
	& = \frac{2}{B} Re \left(\sum\limits^{\frac{B-1}{2}}_{k=0} (1-\eps^2)^{\frac{k}{2}} w^k_j\right) - \frac{1}{B} \\
	& = \frac{2}{B} Re \left( \frac{1-\left(\sqrt{1-\eps^2}\right)^\frac{B+1}{2}w^\frac{B+1}{2}_j}{1-\sqrt{1-\eps^2}w_j}\right) - \frac{1}{B} \\
	& = \frac{2}{B} Re \left( \frac{\left(1-\sqrt{1-\eps^2}^\frac{B+1}{2}(\cos(\pi j + \frac{\pi j}{B}) + i\sin(\pi j + \frac{\pi j}{B}))\right)\left(1-\sqrt{1-\eps^2}\cos(\frac{2\pi j}{B}) + i \sqrt{1-\eps^2} \sin (\frac{2\pi j }{B})\right)}{2-\eps^2 - 2\sqrt{1-\eps^2}\cos\left(\frac{2\pi j}{B}\right)}\right) - \frac{1}{B} \\
	& =  \frac{2}{B} \left( \frac{\left(1-\sqrt{1-\eps^2}^\frac{B+1}{2}\cos(\pi j + \frac{\pi j}{B})\right)\left(1-\sqrt{1-\eps^2}\cos(\frac{2\pi j}{B})\right) + \sqrt{1-\eps^2}^{\frac{B+1}{2} + 1}\sin(\frac{2\pi j}{B})\sin(\pi j + \frac{\pi j}{B})}{2-\eps^2 - 2\sqrt{1-\eps^2}\cos\left(\frac{2\pi j}{B}\right)}\right) - \frac{1}{B}
	\end{align*}
	}
	When $j$ is odd, then for sufficiently small $j$ and sufficiently large $B$, we can say that: 
	{\small
	$$\left( \frac{\left(1-\sqrt{1-\eps^2}^\frac{B+1}{2}\cos(\pi j + \frac{\pi j}{B})\right)\left(1-\sqrt{1-\eps^2}\cos(\frac{2\pi j}{B})\right) + \sqrt{1-\eps^2}^{\frac{B+1}{2} + 1}\sin(\frac{2\pi j}{B})\sin(\pi j + \frac{\pi j}{B})}{2-\eps^2 - 2\sqrt{1-\eps^2}\cos\left(\frac{2\pi j}{B}\right)}\right)$$
	}
	{\small
	$$ \geq \frac{1}{2} \frac{\left(1-\sqrt{1-\eps^2}\cos\left(\frac{2\pi j}{B}\right)\right)}{2-\eps^2 - 2\sqrt{1-\eps^2}\cos\left(\frac{2\pi j}{B}\right)} \geq \frac{1}{2} \frac{\left(1-\sqrt{1-\eps^2}\right)}{2-\eps^2 - 2\sqrt{1-\eps^2}\cos\left(\frac{2\pi j}{B}\right)} \geq \frac{1}{2} \frac{\left(1-\sqrt{1-\eps^2}\right)}{2-\eps^2 - 2\sqrt{1-\eps^2}\cos\left(\frac{\eps}{5}\right)}$$
	}
	
	Standard computation shows that the last term is $\geq 5$ when $\eps < 0.21$ so that $\lambda_j \geq \frac{9}{B}$. 
	
\end{proof}

\begin{lemma}
	\label{lem:perturbation_eigs}
	Let $\lambda_1(P) \geq \ldots \geq \lambda_B(P)$ be the eigenvalues of $P$ in descending order. Suppose that $\eps < 0.21$ and $B = \frac{1}{\eps^7}$. Then $\lambda_{B + 1 - \frac{\eps B}{40 \pi}} \geq -\frac{7}{B}$.
\end{lemma}

\begin{proof}
	$P$ can be shown to have the following block form: 
	
	For even $B$,
	\begin{equation*}
	P =  
	\begin{pmatrix}
	0 & A  \\
	A^T & 0 
	\end{pmatrix}
	\end{equation*}
	where $A$ is an $\frac{B}{2}$ square upper triangular matrix with $\frac{1}{B}\left(1-\eps^2\right)^{\frac{B}{4}}$ along the diagonal. 
	
	For odd $B$, 
	\begin{equation*}
	P =  
	\begin{pmatrix}
	0 & A  \\
	0 & 0 \\
	A^T & 0 
	\end{pmatrix}
	\end{equation*}
	where $A$ is an $\frac{B-1}{2}$ square upper triangular matrix with $\frac{1}{B}\left[(1-\eps^2)^{\frac{B+1}{4}} - (1-\eps^2)^{\frac{B-1}{4}}\right]$ along the diagonal. 
	The eigenvalues of $P$ come in positive-negative pairs, so that $\lambda^2_{B + 1 - \frac{\eps B}{40 \pi}} = \lambda^2_{\frac{\eps B}{40 \pi}}$. Notice that $\sum\limits_{i=1}^B \lambda_i^2 = \|P\|_F^2$, where $\|\cdot\|_F^2$ denotes the Frobenius norm. We will bound $\lambda^2_{\frac{\eps B}{40 \pi}}$ using the Frobenius norm. 
	Suppose for loss of generality that $B$ is odd. Then it follows that the Frobenius norm, $\|P\|^2_F = 2 \cdot \|A\|_F^2$. So we can focus on $\|A\|_F^2$. Note that in this case, the circulent matrix $C$ has entries $C[1][\frac{B+1}{2} + j] = Z[1][\frac{B+1}{2} - j + 1] = \frac{1}{B} (1-\eps^2)^{\frac{\frac{B+1}{2} - j}{2}}$ for $1 \leq j \leq \frac{B-1}{2}$, whereas the original Toeplitz matrix has entries $Z[1][\frac{B+1}{2} + j] = \frac{1}{B} (1-\eps^2)^{\frac{\frac{B+1}{2} + j - 1}{2}}$. 
	
	\begin{align*}
	\|A\|_F^2 &= \sum\limits_{\ell = 1}^{\frac{B-1}{2}} \sum\limits_{j=1}^{\ell} \frac{1}{B^2} \left(  (1-\eps^2)^{\frac{\frac{B+1}{2} + j - 1}{2}} - (1-\eps^2)^{\frac{\frac{B+1}{2} - j}{2}}\right)^2 \\
	& = \frac{1}{B^2} (1-\eps^2)^{\frac{B+1}{2}} \sum\limits_{\ell = 1}^{\frac{B-1}{2}}\sum\limits_{j=1}^{\ell} \left(  (1-\eps^2)^{\frac{j - 1}{2}} - (1-\eps^2)^{\frac{ - j}{2}}\right)^2 \\
	& = \frac{1}{B^2} (1-\eps^2)^{\frac{B+1}{2}} \sum\limits_{\ell = 1}^{\frac{B-1}{2}}\sum\limits_{j=1}^{\ell} \left(  (1-\eps^2)^{j - 1} - 2(1-\eps^2)^{\frac{-1}{2}} + (1-\eps^2)^{-j}\right) \\
	& = \frac{1}{B^2} (1-\eps^2)^{\frac{B+1}{2}} \sum\limits_{\ell = 1}^{\frac{B-1}{2}}\sum\limits_{j=1}^{\ell} - 2(1-\eps^2)^{\frac{-1}{2}} + \frac{1}{B^2} (1-\eps^2)^{\frac{B+1}{2}} \sum\limits_{\ell = 1}^{\frac{B-1}{2}}\sum\limits_{j=1}^{\ell} \left(  (1-\eps^2)^{j - 1}  + (1-\eps^2)^{-j}\right) \\
	& \leq \frac{1}{B^2} (1-\eps^2)^{\frac{B+1}{2}} \sum\limits_{\ell = 1}^{\frac{B-1}{2}}\sum\limits_{j=1}^{\ell} \left(  (1-\eps^2)^{j - 1}  + (1-\eps^2)^{-j}\right)\stepcounter{equation} \tag{\theequation}\label{frobenius_proof_split}
	\end{align*}
	Notice that 
	\[
	\sum\limits_{j=1}^{\ell} (1-\eps^2)^{j-1} = \frac{1-(1-\eps^2)^\ell}{1-(1-\eps^2)} = \frac{1-(1-\eps^2)^\ell}{\eps^2} ~,
	\]
	\begin{align*}
	\sum\limits_{j=1}^\ell (1-\eps^2)^{-j} &= \frac{1}{1-\eps^2} + \left( \frac{1}{1-\eps^2} \right)^2 + \ldots \left( \frac{1}{1-\eps^2} \right)^\ell \\
	& = \frac{1}{1-\eps^2} \left(1 + \ldots  \left( \frac{1}{1-\eps^2} \right)^{\ell-1} \right) \\
	&= \frac{1}{1-\eps^2} \left( \frac{1- \left( \frac{1}{1-\eps^2}\right)^{\ell}}{1-\frac{1}{1-\eps^2}} \right) \\
	& = \frac{1}{\eps^2} \left(\left( \frac{1}{1-\eps^2}\right)^{\ell} -1 \right)
	\end{align*}
	
	Therefore, following line \ref{frobenius_proof_split}, we have: 
	\begin{equation}
	\label{frobenius_eq}
	\|A\|_F^2 \leq \frac{1}{B^2} (1-\eps^2)^{\frac{B+1}{2}} \cdot \frac{1}{\eps^2} \sum\limits_{\ell = 1}^{\frac{B-1}{2}} \left( \left(\frac{1}{1-\eps^2}\right)^{\ell} - (1-\eps^2)^\ell \right)
	\end{equation}
	
	Notice that
	\[
	\sum\limits_{\ell=1}^\frac{B-1}{2} \left( \frac{1}{1-\eps^2}\right)^\ell = \frac{1}{1-\eps^2} \left( \frac{1-\left(\frac{1}{1-\eps^2}\right)^{\frac{B-1}{2}}}{1- \left( \frac{1}{1-\eps^2}\right)}\right) = \frac{\left(\frac{1}{1-\eps^2}\right)^{\frac{B-1}{2}} - 1}{\eps^2} ~,
	\]
	\[
	\sum\limits_{\ell=1}^\frac{B-1}{2} \left(1-\eps^2\right)^{\ell} = (1-\eps^2) \left( \frac{1-(1-\eps^2)^{\frac{B-1}{2}}}{1-(1-\eps^2)} \right) = \frac{(1-\eps^2) - (1-\eps^2)^{\frac{B-1}{2} + 1}}{\eps^2}~,
	\]
	Therefore, 
	\begin{align*}
	\|A\|_F^2 &\leq \frac{1}{B^2} (1-\eps^2)^{\frac{B+1}{2}} \cdot \frac{1}{\eps^2} \sum\limits_{\ell = 1}^{\frac{B-1}{2}} \left( \left(\frac{1}{1-\eps^2}\right)^{\ell} - (1-\eps^2)^\ell \right) \\
	& = \frac{1}{B^2 \eps^2} (1-\eps^2)^{\frac{B+1}{2}} \left( \frac{(1-\eps^2)^{-\left(\frac{B-1}{2}\right)} - 2 + (1-\eps^2)^{\frac{B-1}{2} + 1}}{\eps^2} + 1 \right) \\
	& = \frac{1}{B^2 \eps^4} (1-\eps^2) - \frac{2 (1-\eps^2)^{\frac{B+1}{2}}}{B^2 \eps^4} + \frac{(1-\eps^2)^{B+1}}{B^2 \eps^4} + \frac{1}{B^2 \eps^2} (1-\eps^2)^{\frac{B+1}{2}}\\
	& \leq \frac{1}{B^2 \eps^4} (1-\eps^2)
	\end{align*}
	
	So we conclude that the Frobenius norm of $P$, satisfies: 
	\[
	\|P\|_F^2 \leq \frac{2}{B^2 \eps^4} (1-\eps^2)
	\]
	Therefore.  $\lambda^2_{\frac{\eps B}{40 \pi}} \leq \frac{\frac{2}{B^2 \eps^4} (1-\eps^2)}{\frac{\eps B}{40\pi}} = \frac{1}{B^2} \frac{80\pi}{B\eps^5} (1-\eps^2) = \frac{1}{B^2} \left(\eps^2 80\pi (1-\eps^2) \right)$. For our choice of $\eps, B$, we know that $\lambda_{\frac{\eps B}{40 \pi}} \leq \frac{7}{B}$, Therefore, $\lambda_{B+1 - \frac{\eps B}{40 \pi}} \geq \frac{-7}{B}$. 
\end{proof}

\subsection{Variance Decay with Experience Replay}
\label{sec:er_sgd_upper_bound_variance}
To analyze the variance, we start with $w^{\var}_0 = w^*$. The dynamics of SGD say that: 
$$w^{\var}_{0} - w^* = 0$$
$$w^{\var}_{1} - w^* = \eta \hat{\xi}^{(1)}_1 \hat{X}^{(1)}_1$$
\[
w^{\var}_{t+1} - w^* = \left(I - \eta \hat{X}_{t+1} \hat{X}_{t+1}^T \right) (w^{\var}_t - w^*) + \eta \hat{\xi}_{t+1} \hat{X}_{t+1}
\]

We let the superscript $(i)$ denote the $i$-th buffer index. Let $H^{(i)} = \frac{1}{B}\sum\limits_{j=1}^B X^{(i)}_j X^{(i)T}_j$, where $X^{(i)}_j$ are the vectors that comprise the sampling pool from buffer $i$. 

We produce our final $w$ by tail averaging over the last iterate of SGD from within each buffer $i$. Let $w_{iB}$ denote the last SGD iterate using buffer $i$. Then 
\begin{equation}
w = \frac{1}{N} \sum\limits_{i=1}^N w_{iB}~,
\end{equation} where $N$ is the number of buffers. 

Clearly, 
$$\E[(w^{\var}-w^*)(w^{\var}-w^*)^T] = \frac{1}{N^2} \E\left[\sum\limits_{i=1}^N w^{\var}_{iB} \sum\limits_{i=1}^N w^{\var T}_{iB}\right]$$

\begin{theorem}[Variance Decay for SGD with Experience Replay for Gaussian AR Chain]
	\label{thm:exp-replay_var}For any $\eps \leq 0.21$, if $B \geq \frac{1}{\eps^7}$ and $d = \Omega(B^4 \log (\frac{1}{\beta}))$, with probability at least $1-\beta$, Algorithm~\ref{alg_main} returns $w$ such that $\loss(w^{\var}) \leq \tilde{O}\left(\frac{\sigma^2d\sqrt{\tmix}}{ T}\right)$.
\end{theorem}

\begin{proof}
\begin{align*}
\E[(w^{\var}-w^*)(w^{\var}-w^*)^T] &= \E \left[ \frac{1}{N} \sum\limits_{k=1}^N (w^{\var}_{kB} - w^*)\frac{1}{N} \sum\limits_{k=1}^N (w^{\var}_{kB} - w^*)^T\right] \\
& = \E \left[ \frac{1}{N^2} \sum\limits_{k=1}^N \sum\limits_{\ell=1}^N (w^{\var}_{kB} - w^*)(w^{\var}_{\ell B} - w^*)^T\right] \\
& =\frac{1}{N^2} \E \left[ \sum\limits_{k > \ell} (w^{\var}_{kB} - w^*)(w^{\var}_{\ell B} - w^*)^T\right] + \frac{1}{N^2} \E \left[ \sum\limits_{k < \ell} (w^{\var}_{kB} - w^*)(w^{\var}_{\ell B} - w^*)^T\right]\\
& + \frac{1}{N^2} \E \left[ \sum\limits^N_{k =1} (w^{\var}_{kB} - w^*)(w^{\var}_{kB} - w^*)^T\right]\\
& \preceq \frac{3\sigma^2}{N^2}  \E \left[\sum\limits_{\ell = 1}^N \sum\limits_{k=\ell+1}^N \prod\limits_{j = \ell + 1}^k \prod\limits_{i=1}^B \left( I - \hat{X}_i^{(j)}\hat{X}_i^{(j)T}\right)\right] \\
&+ \frac{3\sigma^2}{N^2}  \E \left[ \left(\prod\limits_{j = \ell + 1}^k \prod\limits_{i=1}^B \left( I - \hat{X}_i^{(j)}\hat{X}_i^{(j)T}\right)\right)^T\right]+ \frac{3\sigma^2}{N^2} \left[N \cdot I\right] ~,
\end{align*}

where the last line follows as a consequence of Lemma \ref{lem:c_infty_er}. 

Now we can focus on $\E\left[\prod\limits_{j = \ell + 1}^k \prod\limits_{i=1}^B \left( I - \hat{X}_i^{(j)}\hat{X}_i^{(j)T}\right)\right]$, which by spherical symmetry is equal to $c \cdot I$ for constant $c$. 

Following Lemma \ref{lem:approxim-iid-helper}, let $\hat{\tilde{X}}^{(j)}_i$ be a sample from $\tilde{X}^{(j)}_{u+s} \defeq \left(1-\epsilon^2\right)^{(u+s)/2} \tilde{X}^{(j)}_0 + \epsilon \sum\limits_{t = 1}^{u+s} \left(1-\epsilon^2\right)^{(t-(u+s))/2} G_t$, where $\tilde{X}^{(j)}_0$ is sampled independently from $\N(0, \frac{1}{\sqrt{d}} I_d)$. 

\begin{align*}
\E\left[\prod\limits_{j = \ell + 1}^k \prod\limits_{i=1}^B \left( I - \hat{X}_i^{(j)}\hat{X}_i^{(j)T}\right)\right] &= \E\left[\prod\limits_{j = \ell + 1}^k \prod\limits_{i=1}^B \left( I - \hat{\tilde{X}}_i^{(j)}\hat{\tilde{X}}_i^{(j)T} + \hat{\tilde{X}}_i^{(j)}\hat{\tilde{X}}_i^{(j)T} -  \hat{X}_i^{(j)}\hat{X}_i^{(j)T}\right)\right] \\
& \preceq \E\left[\prod\limits_{j = \ell + 1}^k \prod\limits_{i=1}^B \left( I - \hat{\tilde{X}}_i^{(j)}\hat{\tilde{X}}_i^{(j)T}\right)\right] + c\cdot (NB)^2 (1-\eps^2)^{u/2} \cdot I
\end{align*}

where $c$ is an appropriate constant. 

By Lemma \ref{lem:H_eigs}, $\E\left[\sum\limits_{\ell = 1}^N \sum\limits_{k=\ell + 1}^N\prod\limits_{j = \ell + 1}^k \prod\limits_{i=1}^B \left( I - \hat{\tilde{X}}_i^{(j)}\hat{\tilde{X}}_i^{(j)T}\right)\right] \preceq \sum\limits_{\ell = 1}^N \sum\limits_{k=\ell + 1}^N \left[ 1-  \frac{\eps B}{160\pi d} \right]^{k-\ell} \cdot I~.$

Therefore, $\E[(w^{\var}-w^*)(w^{\var}-w^*)^T]  \preceq \frac{6\sigma^2}{N^2} \sum\limits_{\ell = 1}^N \sum\limits_{k=\ell}^N \left[ 1-  \frac{\eps B}{160\pi d} \right]^{k-\ell} + 3\sigma^2(NB)^2 (1-\eps^2)^{u/2} \cdot I$. 

\begin{align*}
\sum\limits_{\ell = 1}^N \sum\limits_{k=\ell}^N \left[ 1-  \frac{\eps B}{160\pi d} \right]^{k-\ell} & = \sum\limits_{\ell = 1}^N \sum\limits_{i=0}^{N-\ell} \left[ 1-  \frac{\eps B}{160\pi d} \right]^{i} \\
& =  \sum\limits_{\ell = 1}^N \frac{160 \pi d}{\eps B} -  \frac{160 \pi d}{\eps B} \sum\limits_{\ell = 1}^N  \left[ 1-  \frac{\eps B}{160\pi d} \right]^{N- \ell + 1} \\
& \leq N \cdot \frac{160 \pi d}{\eps B}
\end{align*}

Putting everything together, we have that
$\E[(w^{\var}-w^*)(w^{\var}-w^*)^T] \preceq \frac{6\sigma^2}{N} \cdot \frac{160 \pi d}{\eps B} \cdot I + 3\sigma^2(NB)^2 (1-\eps^2)^{u/2} \cdot I $.

Therefore, when $u = \frac{2}{\epsilon^2} \log \frac{300000 \pi d^2\sigma^6}{\epsilon^2 \delta}$, it follows that $\loss(w^{\var}) \leq O(\frac{\sigma^2d}{\eps T})$.
\end{proof}

\begin{lemma}
\label{lem:c_infty_er}
$\E[(w^{\var}_{t} - w^*)(w^{\var}_{t} - w^*)^T | X^{(1)}_1, \ldots X^{(T/S)}_S] \preceq 3\sigma^2 I$ for all $t$. 
\end{lemma}
\begin{proof}
Proof by induction. 

In the first iterate of the first buffer, we have:
$$(w^{\var}_1 - w^*)(w^{\var}_1 - w^*)^T =\eta^2( \hat{\xi}^{(1)}_1)^2 \hat{X}^{(1)}_1 \hat{X}^{(1) T}_1$$

Since each $\|X_j\|^2_2 \leq 2$ with high probability, and when $\eta = \frac{1}{2}$, we have: 
$$\E \left[(w^{\var}_1 - w^*)(w^{\var}_1 - w^*)^T | X_1, \ldots X_B \right] =\eta^2 \sigma^2 H^{(1)} \preceq \sigma^2 I$$

For the second iterate, we have: 
$$w^{\var}_2 - w^* = \left(I - \eta\hat{X}^{(1)}_2 \hat{X}^{(1) T}_2 \right) \hat{\xi}^{(1)}_1 \hat{X}^{(1)}_1 + \eta \hat{\xi}^{(1)}_2 \hat{X}^{(1)}_2$$

\begin{align*}
\E \left[(w^{\var}_2 - w^*)(w^{\var}_2 - w^*)^T | X_1, \ldots X_B \right] & = \E [ \left(I - \eta\hat{X}^{(1)}_2 \hat{X}^{(1) T}_2 \right) \eta^2( \hat{\xi}^{(1)}_1)^2 \hat{X}^{(1)}_1 \hat{X}^{(1) T}_1 \left(I - \eta\hat{X}^{(1)}_2 \hat{X}^{(1) T}_2 \right) \\
& +  \left(I - \eta\hat{X}^{(1)}_2 \hat{X}^{(1) T}_2 \right) \eta \hat{\xi}^{(1)}_1 \hat{\xi}^{(1)}_2 \hat{X}^{(1)}_1 \hat{X}^{(1) T}_2 \\
& + \eta \hat{\xi}^{(1)}_1 \hat{\xi}^{(1)}_2 \hat{X}^{(1)}_1 \hat{X}^{(1) T}_2  \left(I - \eta\hat{X}^{(1)}_2 \hat{X}^{(1) T}_2 \right) \\
& + \eta^2( \hat{\xi}^{(1)}_2)^2 \hat{X}^{(1)}_2 \hat{X}^{(1) T}_2 | X_1 \ldots X_B] \\
& \preceq \sigma^2 (I - \eta H^{(1)} + \eta H^{(1)} + \eta H^{(1)}+ \eta^2 H^{(1)}) \\
& \preceq 3 \sigma^2 I
\end{align*}
\end{proof}

Suppose that in the first buffer, for $k \leq B$, 
$$\E \left[(w^{\var}_{k-1} - w^*)(w^{\var}_{k-1} - w^*)^T | X_1, \ldots X_B \right] \preceq 3\sigma^2 I $$

Then we look at 
\begin{align*}
\E \left[(w^{\var}_k - w^*)(w^{\var}_k - w^*)^T | X_1, \ldots X_B \right] & = \E [ \left(I - \eta\hat{X}^{(1)}_k \hat{X}^{(1) T}_k \right)(w^{\var}_{k-1} - w^*) (w^{\var}_{k-1} - w^*)^T \left(I - \eta\hat{X}^{(1)}_k \hat{X}^{(1) T}_k \right) \\
& + \left(I - \eta\hat{X}^{(1)}_k \hat{X}^{(1) T}_k \right) (w^{\var}_{k-1} - w^*)\eta \hat{\xi}^{(1)}_k \hat{X}^{(1) T}_k \\
& + \eta \hat{\xi}^{(1)}_k \hat{X}^{(1)}_k  (w^{\var}_{k-1} - w^*)^T \left(I - \eta\hat{X}^{(1)}_k \hat{X}^{(1) T}_k \right) \\
& + \eta^2( \hat{\xi}^{(1)}_k)^2 \hat{X}^{(1)}_k \hat{X}^{(1) T}_k | X_1 \ldots X_B] \\
\end{align*}

Notice that $w^{\var}_{k-1} - w^* = \sum\limits_{j=1}^{k-1} \prod\limits_{i=j+1}^{k-1} \left(I - \eta\hat{X}^{(1)}_i \hat{X}^{(1) T}_i \right) \eta \hat{\xi}^{(1)}_j \hat{X}^{(1)}_j $

We first focus on the cross term: 
\begin{align*}
\E & \left[\left(I - \eta\hat{X}^{(1)}_k \hat{X}^{(1) T}_k \right) (w^{\var}_{k-1} - w^*)\eta \hat{\xi}^{(1)}_k \hat{X}^{(1) T}_k | X_1, \ldots X_B\right] \\
& = \E \left[ \left(I - \eta\hat{X}^{(1)}_k \hat{X}^{(1) T}_k \right) \sum\limits_{j=1}^{k-1} \prod\limits_{i=j+1}^{k-1} \left(I - \eta\hat{X}^{(1)}_i \hat{X}^{(1) T}_i \right) \eta \hat{\xi}^{(1)}_j \hat{X}^{(1)}_j  \left(\eta \hat{\xi}^{(1)}_k \hat{X}^{(1) T}_k\right) | X_1, \ldots X_B\right] 
\end{align*}

Notice that by the independence of the noise, only those terms where $\hat{X}_j = \hat{X}_k$ will be non-zero (and this event happens with probability $\frac{1}{B}$ for each $j$). Moreover, note that 
\begin{align*}
\E & \left[ \eta \hat{X}^{(1)}_k \hat{X}^{(1) T}_k \sum\limits_{j=1}^{k-1} \prod\limits_{i=j+1}^{k-1} \left(I - \eta\hat{X}^{(1)}_i \hat{X}^{(1) T}_i \right) \eta \hat{\xi}^{(1)}_k \hat{X}^{(1)}_k \left(\eta \hat{\xi}^{(1)}_k \hat{X}^{(1) T}_k \right) | X_1, \ldots X_B \right] \\
& = \E\left[ \eta \hat{X}^{(1)}_k \hat{X}^{(1) T}_k \sum\limits_{j=1}^{k-1} \left(I - \eta H^{(1)} \right)^{k-j-2} \eta \hat{\xi}^{(1)}_k \hat{X}^{(1)}_k \left(\eta \hat{\xi}^{(1)}_k \hat{X}^{(1) T}_k \right) | X_1, \ldots X_B \right] \\
& \succeq 0
\end{align*}

Therefore, we have: 
\begin{align*}
\E & \left[\left(I - \eta\hat{X}^{(1)}_k \hat{X}^{(1) T}_k \right) (w^{\var}_{k-1} - w^*)\eta \hat{\xi}^{(1)}_k \hat{X}^{(1) T}_k | X_1, \ldots X_B\right] \\
& \preceq \frac{\sigma^2}{B} \left[ \sum\limits_{j=1}^{k-1} \prod\limits_{i=j+1}^{k-1} \left(I - \eta H^{(1)} \right) \eta H^{(1)} \right] \\
& \preceq \eta \sigma^2 \frac{k-1}{B} H^{(1)} \preceq \eta \sigma^2 H^{(1)}
\end{align*}

Therefore, 
\begin{align*}
\E \left[(w^{\var}_k - w^*)(w^{\var}_k - w^*)^T | X_1, \ldots X_B \right] & \preceq 3\sigma^2 (I - \eta H^{(1)}) + 2 \eta \sigma^2 H^{(1)} +  \eta^2 \sigma^2 H^{(1)} \\
& \preceq 3\sigma^2 I 
\end{align*}

Therefore, in the first buffer, $\E \left[(w^{\var}_k - w^*)(w^{\var}_k - w^*)^T | X_1, \ldots X_B \right] \preceq 3\sigma^2 I$ for all $k \leq B$. 

For the first iterate using the second buffer, because the cross terms are 0 by independent noise, it is easy to show that
\begin{align*}
\E \left[(w^{\var}_{B+1} - w^*)(w^{\var}_{B+1} - w^*)^T | X^{(2)}_1, \ldots X^{(2)}_B \right] & = \E [ \left(I - \eta\hat{X}^{(2)}_1 \hat{X}^{(2) T}_1 \right)(w^{\var}_{B} - w^*) (w^{\var}_{B} - w^*)^T \left(I - \eta\hat{X}^{(2)}_1 \hat{X}^{(2) T}_1 \right) \\
& + \eta^2( \hat{\xi}^{(2)}_1)^2 \hat{X}^{(2)}_1 \hat{X}^{(2) T}_1 | X^{(2)}_1, \ldots X^{(2)}_B] \\
& \preceq 3 \sigma^2 I 
\end{align*}

For subsequent iterates in the second buffer, 

We write out $w^{\var}_{B+k-1} - w^* = \prod\limits_{j=1}^k \left(I - \eta\hat{X}^{(2)}_j \hat{X}^{(2) T}_j \right) (w^{\var}_B - w^*) + \sum\limits_{j=1}^{k-1} \prod\limits_{i=j+1}^{k-1} \left(I - \eta\hat{X}^{(2)}_i \hat{X}^{(2) T}_i \right) \eta \hat{\xi}^{(2)}_j \hat{X}^{(2)}_j $. 

Therefore, the cross term 
\begin{align*}
\E & \left[\left(I - \eta\hat{X}^{(2)}_{B+k} \hat{X}^{(2) T}_{B+k} \right)(w^{\var}_{B+k-1} - w^*)\eta \hat{\xi}^{(2)}_{B+k} \hat{X}^{(2)}_{B+k} | X^{(2)}_1, \ldots X^{(2)}_B \right]\\
& = \E \left[\left(I - \eta\hat{X}^{(2)}_{B+k} \hat{X}^{(2) T}_{B+k} \right)  \sum\limits_{j=1}^{k-1} \prod\limits_{i=j+1}^{k-1} \left(I - \eta\hat{X}^{(2)}_i \hat{X}^{(2) T}_i \right) \eta \hat{\xi}^{(2)}_j \hat{X}^{(2)}_j  \eta \hat{\xi}^{(2)}_{B+k} \hat{X}^{(2)}_{B+k} | X^{(2)}_1, \ldots X^{(2)}_B \right]\\
& \preceq \eta \sigma^2 H^{(2)} 
\end{align*}

Therefore, by induction, we have that $\E[(w^{\var}_{t} - w^*)(w^{\var}_{t} - w^*)^T | X^{(1)}_1, \ldots X^{(T/S)}_S] \preceq 3\sigma^2 I$ for all $t$.

\subsection{Lower Bound for SGD with Constant Step Size}
\label{app:er_lb}

\begin{proof}[Proof of Theorem~\ref{thm:sequential_lower_bound}]
	
	\noindent We know that $w_{t+1} - w^* = (I - \eta X_tX_t^T) (w_t - w^*)$. We define: 
	\[
	\alpha_t = X_t^T(w_t - w^*)
	\]
	\[
	\gamma_t = \|w_t - w^*\|
	\]
	Then we have: 
	\begin{align*}
	\alpha_{t+1} &= X_{t+1}^T(w_{t+1} - w^*) \\
	& = \left(\sqrt{1-\eps^2} X_t + \eps G_{t+1} \right)^T \left(I - \eta X_tX_t^T \right) (w_t - w^*) \\
	& = \sqrt{1-\eps^2}\alpha_t + \eps G_{t+1}^T(w_t - w^*) - \eta \sqrt{1-\eps^2} \alpha_t \|X_t\|^2 - \eta \eps G_{t+1}^TX_tX_t^T(w_t - w^*)
	\end{align*}
	Suppose that $ \frac{c\log(\frac{1}{\beta})}{\sqrt{d}} \leq 8$, then by Lemmas \ref{lem:concentration_norm} and \ref{lem:concentration_inner_product}, we have:
	\begin{align*}
	\E\left[\alpha_{t+1}^2\right] &\leq (1-\eps^2) \E[\alpha_t^2] + \frac{\eps^2}{d} \E[\gamma_t^2] + 7\eta^2 (1-\eps^2) \E[\alpha_t^2] + 7\frac{\eta^2 \eps^2}{d}  \E[\alpha_t^2] \\
	& + 14\eta (1-\eps^2) \E[\alpha_t^2] - \frac{2\eta\eps^2}{d} \E[\alpha_t^2] \\
	&\leq \left((1-\eps^2)(1+7\eta^2 + 14\eta) + \eps^2\left(\frac{7\eta^2}{d} - \frac{2\eta}{d}\right) \right)\E[\alpha_t^2] + \frac{\eps^2}{d} \E[\gamma_t^2] \\
	&=\left(1- \left[\eps^2 -  (1-\eps^2)(7\eta^2 + 14\eta) - \frac{\eps^2}{d} \left(7\eta^2 - 2\eta\right)\right]\right) \E[\alpha_t^2] + \frac{\eps^2}{d} \E[\gamma_t^2] \\
	\end{align*}
	Now we turn to $\gamma_{t+1}$. We have: 
	\begin{align*}
	\gamma^2_{t+1} &= \|w_{t+1} - w^*\|^2 = (w_t - w^*)^T(I-\eta X_tX_t^T)(I-\eta X_tX_t^T)(w_t - w^*) \\
	& = (w_t - w^*)^T (I - (2\eta-\eta^2 \|X_t\|^2)X_tX_t^T)(w_t - w^*) \\
	& = \gamma^2_{t} - (2\eta - \eta^2 \|X_t\|^2)\alpha_t^2
	\end{align*}
	Therefore, we can say that $\E[\gamma^2_{t+1}] \geq \E[\gamma^2_{t}] - (2\eta - \eta^2 (-7))\E[\alpha_t^2]$.

	When $\eps^2 > 0.5$ and $\eta < 0.05$, it follows that 
	\begin{itemize}
		\item $\E[\gamma^2_{t+1}] \leq \E[\gamma^2_{t}]$
		\item $\E\left[\alpha_{t+1}^2\right] \leq (1-\zeta) \E\left[\alpha_{t}^2\right] + \frac{\eps^2}{d}  \E[\gamma^2_{t}]$, where $0<\zeta<1$ 
		\item Moreover, $\zeta > 7\eta^2 + 2\eta$, so $\E[\gamma^2_{t+1}] \geq \E[\gamma^2_{t}] - \zeta\E[\alpha_t^2]$. 
	\end{itemize}
	
	Unwrapping, the recursion, we can say that 
	\begin{align*}
	\E[\alpha^2_{t+1}] &\leq (1-\zeta) \E[\alpha^2_{t}] + \frac{\eps^2}{d}  \E[\gamma^2_{1}] \\
	& \leq (1-\zeta) \left((1-\zeta) \E[\alpha^2_{t-1}] + \frac{\eps^2}{d}  \E[\gamma^2_{1}]\right) + \frac{\eps^2}{d}  \E[\gamma^2_{1}] \\
	& \leq (1-\zeta)^t\E[\alpha^2_{1}] + \frac{\eps^2}{d} \left(\sum\limits_{j=0}^{t-1} (1-\zeta)^j \right) \E[\gamma^2_{1}] \\
	& \leq (1-\zeta)^t\E[\alpha^2_{1}]  +  \frac{\eps^2}{d\zeta}  \E[\gamma^2_{1}]
	\end{align*}
	Note that $\E[\alpha^2_{1}] = \E\left[(X_1^T(w_1 - w^*))^2\right] = \frac{\E[\gamma^2_{1}]}{d}$
	Therefore we can say that $\E[\alpha^2_{t+1}] \leq (1-\zeta)^t\frac{\E[\gamma^2_{1}]}{d} +  \frac{\eps^2}{d\zeta}  \E[\gamma^2_{1}]$
	
	Now we unwrap the recursion for $\E[\gamma^2_{t+1}] \geq \E[\gamma^2_{t}] - \zeta\E[\alpha_t^2]$. We have: 
	\begin{align*}
	\E[\gamma^2_{t+1}] & \geq \E[\gamma^2_{t}] - \zeta\E[\alpha_t^2] \\
	& \geq \E[\gamma^2_{t-1}] - \zeta\E[\alpha_{t-1}^2] - \zeta\E[\alpha_t^2] \\
	& \geq \E[\gamma^2_{1}] - \zeta \sum\limits_{j=1}^t \E[\alpha_{j}^2] \\
	& \geq \E[\gamma^2_{1}] - \zeta \sum\limits_{j=1}^t \left( (1-\zeta)^{j-1}\E[\alpha^2_{1}]  +  \frac{\eps^2}{d\zeta}  \E[\gamma^2_{1}]\right) \\
	& = \E[\gamma^2_{1}] - \zeta \left(\frac{t\eps^2}{d\zeta} \right) \E[\gamma^2_{1}] - \zeta \cdot \frac{1-(1-\zeta)^t}{\zeta} \cdot \E[\alpha^2_{1}] \\
	& \geq \E[\gamma^2_{1}] - \zeta \left(\frac{t\eps^2}{d\zeta} \right) \E[\gamma^2_{1}] - \frac{1}{d} \E[\gamma^2_{1}]
	\end{align*}
	
	In order for $\frac{t\eps^2}{d} > \frac{1}{2}$, we need $t \geq \frac{d}{2\eps^2}$. Therefore the number of samples required is $T = \Omega \left(\frac{d}{\eps^2}\right)$. 
\end{proof}

\subsection{Additional Simulations}

We also conducted experiments on data generated using Gaussian AR MC \eqref{eq:gaussar} . We set $d=100$,  noise std. deviation $\sigma=1e-3$, $\epsilon=.01$, and buffer size $B=1/\eps^2$. We report results averaged over $20$ runs. Figure~\ref{fig:exprply_100} compares the estimation error achieved by SGD, \sgddd, and the proposed SGD-ER method.
\begin{figure}[H]
\centering
	\includegraphics[scale=0.5]{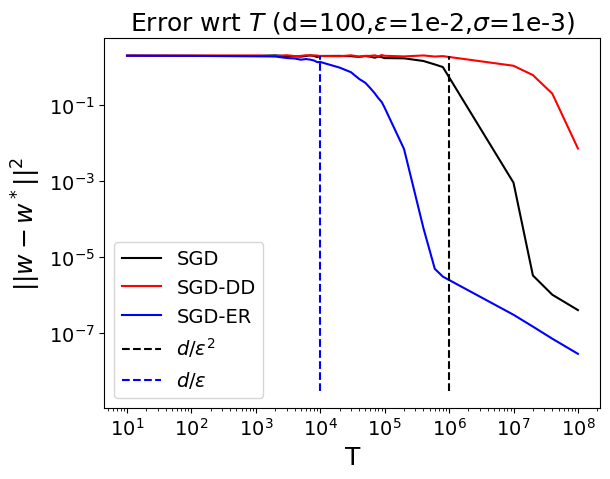}
	\caption{Gaussian AR Chain: error incurred by various methods.}
	\label{fig:exprply_100}
\end{figure}

\end{document}